%% file: _main.tex
\documentclass[lettersize,journal]{IEEEtran}
\usepackage{amsfonts}
\usepackage{algorithmic}
\usepackage{algorithm}
\usepackage{array}
\usepackage[caption=false,font=normalsize,labelfont=sf,textfont=sf]{subfig}
\usepackage{textcomp}
\usepackage{stfloats}
\usepackage{url}
\usepackage{verbatim}
\usepackage{graphicx}
\usepackage{subcaption}
\usepackage{tabularx}
\usepackage{multicol}
\usepackage{caption}
\usepackage{float}
\usepackage{cite}
\usepackage{soul, xcolor}
\usepackage{booktabs}
\usepackage{multirow}
\usepackage{enumitem}
\usepackage{amsmath,amssymb,amsthm}
\theoremstyle{definition}
\newtheorem{defn}{Definition}
\newtheorem{lemma}{Lemma}

\newtheorem{props}{Proposition}
\usepackage{threeparttable}

\begin{document}
\title{Meta-Reinforcement Learning via Evolution for Multi-Objective Combinatorial Supply Chain Optimisation}

\author{
  Rifny Rachman\thanks{Corresponding author. Email: \texttt{rifny.rachman@manchester.ac.uk}}\\
  The University of Manchester, United Kingdom\\
  \and
  Bahrul Ilmi Nasution\\
  The University of Manchester, United Kingdom\\
  \texttt{bahrul.nasution@manchester.ac.uk}\\
  \and
  Josh Tingey\\
  Peak AI, Ltd, United Kingdom\\
  \texttt{josh.tingey@peak.ai}\\
  \and
  Richard Allmendinger\\
  The University of Manchester, United Kingdom\\  \texttt{richard.allmendinger@manchester.ac.uk}\\
  \and
  Pradyumn Shukla\\
  The University of Manchester, United Kingdom\\
  \texttt{pradyumn.shukla@manchester.ac.uk}\\
  \and
  Wei Pan\\
  The University of Manchester, United Kingdom\\
  \texttt{wei.pan@manchester.ac.uk}
}



\maketitle

\begin{abstract}
Meta-reinforcement learning is a promising approach to multi-objective optimisation because it enables rapid policy adaptation across changing environments and preference settings. However, conventional few-shot methods usually fine-tune from a single shared meta-policy, which can reduce solution diversity and limit exploration of the Pareto front, especially in high-dimensional combinatorial problems such as supply chain optimisation. We propose a population-based Meta-reinforcement learning framework that combines decomposition with evolutionary search in scalarisation weight space. The framework maintains a population of weight vectors, each associated with a distinct meta-policy trained through gradient-based meta-learning, and iteratively refines this population through elitist selection, crossover, and mutation guided by hypervolume and entropy contributions. We evaluate the method in a multi-objective supply chain setting with conflicting economic, environmental, and social goals, and further test its generality on standard reinforcement learning problems. The results show that the proposed approach yields more diverse, better distributed Pareto front approximations, improves cross-task adaptation, increases hypervolume by up to 32\% over Meta-multi-objective reinforcement learning in the complex case, and attains the lowest average Hausdorff distance among all compared methods.
\end{abstract}


\section{Introduction}
\IEEEPARstart{M}{ulti-objective} combinatorial supply chain (SC) optimisation is inherently complex, involving many interconnected facilities, routes, and competing objectives. This complexity is amplified by uncertainties, operational fluctuations, and disruptions, which require rapid, adaptive decisions. Traditional methods, from classical optimisation and metaheuristics to hybrid models~\cite{jayarathnaMultiObjectiveOptimizationSustainable2021}, struggle with scalability, adaptability, and responsiveness. More recently, multi-objective reinforcement learning (MORL), which uses sequential decision-making~\cite{suttonReinforcementLearningIntroduction2015}, has shown promise by leveraging agent–environment interactions for dynamic, feedback-driven learning and convergence~\cite{sharMultiObjectiveReinforcementLearning2023, rachmanReinforcementLearningMultiobjective2026}. Yet, MORL methods remain computationally intensive and often overfit to specific SC configurations, making them brittle when SC structures change (e.g., fluctuating costs, demand, and network disruptions).

Meta-learning offers a compelling solution by enabling models to ``learn how to learn'' across a distribution of SC tasks, where each task corresponds to a distinct SC configuration, such as varying demand patterns or network topologies. It facilitates rapid adaptation and strong generalisation~\cite{finnModelAgnosticMetaLearningFast2017}, allowing policy reuse in previously unseen environments with minimal retraining. Unlike multi-task settings, where a single model is trained across heterogeneous problem instances sharing the same objective, multi-objective settings require simultaneously optimising conflicting criteria such as profit, emissions, and service-level (SL) inequality, and remain underexplored in SC contexts. Combining meta-learning with MORL (i.e., Meta-MORL) yields a flexible and scalable framework that efficiently adapts to dynamic environments while reducing computational overhead.

Recent work by Rachman et al.~\cite{rachmanMIRACLDiverseMetaReinforcement2026} showed that Meta-MORL can match conventional MORL performance in SC settings while adapting substantially faster. However, most Meta-MORL methods use a single shared meta-policy for adaptation, which reduces solution diversity and limits exploration of the Pareto front (PF), especially in complex, high-dimensional problems requiring diverse trade-offs.

We propose MERLION (Meta-Reinforcement Learning via Evolution), a population-based Meta-MORL framework inspired by multi-objective evolutionary algorithms (MOEAs) such as MOEA based on decomposition (MOEA/D)~\cite{qingfuzhangMOEAMultiobjectiveEvolutionary2007} and Non-dominated Sorting Genetic Algorithm (NSGA-II)~\cite{debFastElitistMultiobjective2002}. MERLION initialises a population of weight vectors, each defining a scalarisation of the multi-objective space. A meta-policy is meta-trained via gradient adaptation for each weight vector. The population is then refined using fitness scores based on marginal hypervolume and entropy to capture diversity, optimality, and density. These scores drive evolutionary operations (i.e., mutation and crossover of weight vectors) that guide learning across task-conditioned objective regions.

By decentralising learning across multiple objective neighbourhoods, MERLION trains diverse meta-policies specialised for different regions of the PF approximation. Fine-tuning then starts from multiple well-initialised policies that map SC states (e.g., inventory levels, outstanding orders) to actions (e.g., production and shipment quantities), improving both exploration and exploitation. This leads to more robust policy adaptation, greater diversity, and broader Pareto coverage without significant additional online adaptation cost. To our knowledge, MERLION is the first method to combine population-based evolutionary search over scalarisation weights with gradient-based meta-learning for multi-objective combinatorial SC optimisation. Our key contributions are:

\begin{enumerate}
    \item \textit{Methodology.} We propose \textit{MERLION}, a population-based Meta-MORL framework that incorporates a MOEA-inspired strategy into meta-training. Unlike conventional meta-learning, which yields a single shared meta-policy, MERLION maintains and updates multiple meta-policies to enhance solution diversity, balance exploration and exploitation, and enable efficient fine-tuning. During meta-training, an evolutionary procedure operates on scalarisation weight vectors, followed by local perturbation during fine-tuning (Section~\ref{sec:proposed_algorithm}).

    \item \textit{Problem adaptation.} We extend the existing multi-objective Markov decision process (MOMDP) formulation of the SC optimisation, specifically the SC network design and inventory management problem, to the proposed population-based Meta-MORL setting. In particular, we adapt the formulation to support multiple meta-policies, each conditioned on a scalarisation weight vector sampled from the preference simplex, and the associated meta-training and fine-tuning procedures (Section~\ref{sec:problem_definition}).

    \item \textit{Theoretical analysis.} We provide an analytical motivation for the use of population-based meta-policies in MERLION. In particular, we show how maintaining multiple meta-policies can increase the probability of hypervolume improvement, and we explain how preference specialisation can reduce cross-preference gradient interference during meta-training (Section~\ref{sec:analytical_motivation}).

    \item \textit{Empirical evaluation and case implementation.} We evaluate MERLION on three SC problems of increasing complexity, characterised by growing action and observation spaces and tighter operational constraints, mimicking our industrial partner's topologies and parameter settings. These instances are representative rather than exhaustive, and we benchmark MERLION against five methods: Meta-MORL, a population-based MORL approach, and three metaheuristic baselines. We analyse the operational impact and perform an ablation study to assess robustness to parameter changes (Sections~\ref{sec:experiment_results} and~\ref{sec:ablation_study}).

    \item \textit{Cross-domain evaluation.} To assess the generalisability of MERLION, we benchmark it beyond SC problems on five commonly used artificial RL problems varying in reward, action, and observation spaces, and compare it against Meta-MORL and five MORL baselines (Section~\ref{sec:domain_agnosticism}).
    
    \item \textit{Open-source benchmark framework.} We extend several existing RL problems from mo-gymnasium-based environments~\cite{felten_toolkit_2023} and provide reusable environments for SC, enabling randomised task generation and multiple-environment sampling for meta-training and adaptation in Meta-MORL settings. In addition, we release the MERLION implementation to support reproducible evaluation of population-based Meta-MORL methods across heterogeneous tasks.\footnote{The implementation is available upon acceptance.}
\end{enumerate}

Section~\ref{sec:related_work} reviews relevant literature. Section~\ref{sec:problem_definition} formulates the SC optimisation problem using MOMDP within a population-based Meta-MORL framework. Section~\ref{sec:methods} presents our proposed method. Section~\ref{sec:experiment_results} details our experiments and their results. Finally, Section~\ref{sec:conclusions} summarises the study and outlines future research directions.

\section{Related Works} \label{sec:related_work}
This section reviews previous work related to (i)~Population-based methods for SC optimisation, (ii)~evolutionary frameworks in RL, and (iii)~meta-learning for MORL.

\subsection{Population-based Optimisation for SC Problems}
MOEAs are population-based multi-objective optimisation methods that approximate the PF by evolving a set of non-dominated solutions through selection, mutation, crossover, and elitist replacement~\cite{smolinskiMultiobjectiveEvolutionaryAlgorithms2014,debMultiobjectiveOptimizationUsing2001}. Their ability to maintain diverse trade-off solutions makes them popular for multi-objective SC optimisation, which is often NP-hard, high-dimensional, and characterised by conflicting economic, environmental, and social objectives~\cite{faramarzi-oghaniMetaheuristicsSustainableSupply2023,detwalMetaheuristicsCircularSupply2023}.

Multi-objective genetic algorithm~\cite{fonseca1993genetic} uses non-dominance ranks and sharing functions, and has been applied to production planning and partner selection~\cite{deGreenLogisticsImperfect2018,yehUsingMultiobjectiveGenetic2011}, although its non-elitist design can reduce quality and diversity. NSGA-II~\cite{debFastElitistMultiobjective2002} improves this through elitist non-dominated sorting and crowding distance, and remains widely used in SC network design~\cite{ehtesham2021multi,kadzinskiEvaluationMultiobjectiveOptimization2017,wangEfficiencySortingMultiobjective2020}. Strength Pareto evolutionary algorithm II (SPEA-II)~\cite{zitzlerSPEA2ImprovingStrength2001} further emphasises elitism using dominance strength and an external archive with $k$-nearest-neighbour density estimation, showing competitive performance in green SC design~\cite{kadzinskiEvaluationMultiobjectiveOptimization2017}. MOEA/D~\cite{qingfuzhangMOEAMultiobjectiveEvolutionary2007} shifts from dominance-based selection to decomposition by solving scalar subproblems in parallel, and has performed well on large-dimensional SC network instances~\cite{azadehEvolutionaryMultiobjectiveOptimization2017}. 

Recent SC studies also adopt related Pareto-based swarm and evolutionary variants. MOPSO~\cite{coelloHandlingMultipleObjectives2004} extends particle swarm optimisation using Pareto dominance and an external repository, and has been applied to vaccine supply--production--distribution planning~\cite{jahedSustainableVaccineSupplyproductiondistribution2025}. Multi-objective grey wolf optimizer~\cite{mirjaliliMultiobjectiveGreyWolf2016} similarly maintains an external archive and has been used in healthcare and vaccine SC design~\cite{alaFuzzyMultiobjectiveOptimization2024,jahedSustainableVaccineSupplyproductiondistribution2025}. NSGA-III~\cite{6600851} extends NSGA-II with reference-point-based niching for many-objective settings, and has been combined with sparrow search for green SC coordination~\cite{wangGreenSupplyChain2025}. Nevertheless, NSGA-II remains an appropriate main evolutionary benchmark in this study: beyond its wide adoption and interpretability, empirical evidence suggests it consistently delivers competitive PF approximations in large-dimensional SC problems~\cite{halderMultiObjectiveClosedLoop2025,wangEfficiencySortingMultiobjective2020}, making it a robust and directly comparable baseline for our three-objective SC settings.

\subsection{Evolutionary Framework in RL}

Evolutionary reinforcement learning (ERL) combines gradient-based policy optimisation with population-based evolutionary search, leveraging evolutionary operators to improve exploration, diversity, and robustness while retaining the learning efficiency of policy gradients~\cite{linEvolutionaryReinforcementLearning2025}.

Recent work has extended ERL to multi-objective settings. For example, multi-objective cooperative evolutionary deep reinforcement learning~\cite{shianifarMultiObjectiveCooperativeEvolutionary2026} integrates deep RL, evolutionary search, Pareto archives, and hypervolume-based optimisation to improve convergence and diversity. Other studies combine evolutionary search with policy-gradient RL for multi-objective load balancing~\cite{yangReducingIdlenessFinancial2024}, integrate Q-learning into multi-objective differential evolution for feature selection~\cite{yuReinforcementLearningbasedMultiobjective2024}, and employ population-based training to generate behaviour-diverse game agents that outperform conventional RL baselines~\cite{shen2020generating}. Related work has also incorporated evolutionary search into PPO to improve safety and policy stability~\cite{marchesiniExploringSaferBehaviors2022}.

Despite these advances, ERL remains relatively underexplored for SC optimisation. Existing examples include the combination of RL with MOEA for inventory management~\cite{qiuLeveragingReinforcementLearning2024} and the integration of a genetic algorithm with Q-learning to support interpretable SC decision-making~\cite{garciaModelingEvolutionCarbon2025}. Moreover, most ERL studies focus on optimising policy populations within a fixed problem setting and do not address rapid adaptation across varying network structures, operational conditions, and stakeholder preferences that characterise practical SC environments.

\subsection{Meta-Learning for RL}
MORL methods can be effective but often require extensive interaction data and retraining when deployed in new or shifting environments. In SC contexts, where disruptions and parameter changes are common, this limits practicality. Meta-learning addresses this by learning a transferable initialisation across a task distribution, enabling rapid adaptation with few samples. The early foundations were discussed by \cite{schmidthuberEvolutionaryPrinciplesSelfreferential1987} and later formalised in learning-to-learn frameworks~\cite{thrunLearningLearnIntroduction1998}. Unlike conventional deep learning, meta-learning targets generalisation from limited task samples, supporting settings with constrained data and compute~\cite{altae2017low, ignatovAIBenchmarkAll2019}.

Meta-RL combines meta-learning with RL by optimising meta-parameters during an outer loop (meta-training) while adapting a base policy within the inner loop. Fast RL via slow RL (RL$^2$)~\cite{duanRL$^2$FastReinforcement2016} performs inner-loop adaptation using an RNN that conditions on past states, actions, and rewards, enabling within-episode adaptation under partial observability. In contrast, model-agnostic meta-learning (MAML)~\cite{finnModelAgnosticMetaLearningFast2017} learns an initialisation that can be adapted to a task with a small number of gradient steps. MAML typically provides stronger cross-task generalisation (including to unseen tasks), whereas RL$^2$ can be more sample-efficient within the inner loop due to its recurrent memory.

\subsection{Meta-MORL Framework}
To extend Meta-RL to multi-objective optimisation, \cite{chenMetaLearningMultiobjectiveReinforcement2019} proposed Meta-MORL, which samples preference weight vectors to scalarise objectives during inner-loop adaptation and updates a meta-policy across tasks in the outer loop. At deployment, fine-tuning under multiple weights generates a PF approximation set. However, PF quality depends heavily on task and weight sampling, and diversity is not explicitly encouraged.

To improve weight selection, PG-MORL~\cite{liuPredictionGuidedMetaLearning2021} uses a prediction model to select weights expected to improve the PF, although the resulting policy updates are not directly incorporated into the meta-policy. Beyond weight selection, \cite{luMetaLearningApproachMultiObjective2024} combines Reptile and generalised policy improvement (GPI) to support few-shot adaptation, where a policy adapts to a new task using only a small number of interactions or gradient updates. Similarly, \cite{10439641} applies MAML-based Meta-MORL to vehicular networks by decomposing a task into subproblems and learning an initialisation that can be rapidly fine-tuned.


Overall, Meta-MORL can efficiently generate multiple trade-off solutions through fine-tuning, but most approaches rely on a single meta-policy and prioritise rapid adaptation over diversity preservation. As a result, PF coverage may be limited, particularly in complex problems. Moreover, existing Meta-MORL studies focus primarily on continuous-control and communication-network domains, with limited attention to combinatorial optimisation problems such as SC network design. To the best of our knowledge, no prior work combines population-based evolutionary search with meta-learning for multi-objective SC optimisation, despite the need for both diverse Pareto solutions and rapid adaptation to changing operating conditions and preferences.

\section{Problem Definition} \label{sec:problem_definition}

We model the combinatorial SC optimisation problem as a finite-horizon MOMDP~\cite{hayesPracticalGuideMultiobjective2022}. In standard RL, an agent learns a policy $\pi_{\theta}(a_t|s_t)$ that maximises the expected cumulative reward obtained through interactions with an environment. The underlying Markov decision process is defined by the tuple $\langle \mathcal{S}, \mathcal{A}, ST, \gamma, \mu, R \rangle$, where $\mathcal{S}$ denotes the state space, $\mathcal{A}$ the action space, $ST:\mathcal{S}\times\mathcal{A}\rightarrow\mathcal{S}$ the state transition function, $\gamma\in[0,1]$ the discount factor, $\mu$ the initial-state distribution, and $R$ a scalar reward function.

In MORL, the scalar reward is replaced by a vector-valued reward function $\mathbf{R}:\mathcal{S}\times\mathcal{A}\rightarrow\mathbb{R}^{B}$ representing $B$ potentially conflicting objectives. The resulting MOMDP is defined by the tuple $\langle \mathcal{S}, \mathcal{A}, ST, \gamma, \mu, \mathbf{R} \rangle$. In our SC setting, the objectives correspond to profit, emissions, and SL inequality. A scalarisation weight vector $\mathbf{w}\in\Delta^{B-1}$, where
\begin{equation}
\Delta^{B-1}=\left\{\mathbf{w}\in\mathbb{R}^{B}: w_b\ge0,\ \sum_{b=1}^{B}w_b=1\right\},
\end{equation}
denotes the $(B-1)$-dimensional simplex. The weight vector converts the reward vector into a scalar objective, enabling exploration of different trade-offs among objectives.

Meta-MORL extends MORL by learning across a distribution of tasks rather than a single environment. A task $\mathcal{T}$ corresponds to a specific SC instance, such as a particular network configuration, demand profile, or operating condition. Within a task, different scalarisation weight vectors define preference-specific optimisation subproblems. For example, two subproblems may optimise the same SC network while prioritising profit and sustainability differently. Meta-training aims to learn policy initialisations that can rapidly adapt to new tasks and preferences.

In conventional Meta-MORL, meta-training is organised over a set of subtasks. We use \emph{subtask} as a generic adaptation unit because it may correspond to a task instance, a scalarisation preference, or both. In MERLION, this decomposition is made explicit: $i$ indexes scalarisation subproblems (equivalently, population members), $j$ indexes sampled tasks, and $t$ indexes time steps within an episode of horizon $T$.

To generalise across a distribution of SC tasks, we adopt a gradient-based Meta-RL framework that learns a set of initial policy parameterisations $\theta_i$ for rapid adaptation. During meta-training, each scalarisation subproblem $i$ is associated with a weight vector $\mathbf{w}_i \in \Delta^{B-1}$ and is trained on its own batch of sampled tasks $\{\mathcal{T}_{ij}\}_{j=1}^{J}$, where

\begin{equation}
\mathcal{T}_{ij} = \langle \mathcal{S}, \mathcal{A}, ST_{ij}, \gamma, \mu_{ij}, \mathbf{R}_{ij} \rangle \sim p(\mathcal{T}).
\end{equation}

Here, $ST_{ij}$, $\mu_{ij}$, and $\mathbf{R}_{ij}$ denote the transition dynamics, initial-state distribution, and vector reward function of task instance $\mathcal{T}_{ij}$, respectively.

For task $j$, the realised vector reward at time step $t$ is

\begin{equation}
\mathbf{r}_{ijt} = \mathbf{R}_{ij} (\mathbf{s}_{ijt},\mathbf{a}_{ijt}),
\end{equation}

where $\mathbf{s}_{ijt}\in\mathcal{S}$ and $\mathbf{a}_{ijt}\in\mathcal{A}$ denote the state and action visited under subproblem $i$ on task $j$ at time step $t$. For the SC problem, we define

\begin{equation}
\mathbf{r}_{ijt} = \{\mathit{Profit}_t, -\mathit{Emission}_t, -\mathit{Inequality}_t\} \cdot \rho_t,
\end{equation}

where objectives to be minimised are sign-flipped, and $\rho_t$ is a penalty applied whenever inventory falls below zero (see Eq.~\ref{eq:penalty} in Appendix~\ref{sec:sc_problem}). All objectives are normalised to $[0,1]$ to ensure comparable learning signals across objectives. The agent learns task-adapted policies $\pi_{\theta'_{ij}}(\mathbf{a}_{ijt}\mid\mathbf{s}_{ijt})$ that approximate the PF in terms of expected cumulative vector returns. For clarity, the detailed SC formulations adopted from~\cite{rachmanReinforcementLearningMultiobjective2026} are restated in Appendix~\ref{sec:sc_problem}.

Following Finn et al.~\cite{finnModelAgnosticMetaLearningFast2017}, we define a scalarised RL loss for each task--weight pair $\mathcal{T}_{ij}$ as

\begin{equation}
\label{eq:loss_function}
\mathcal{L}_{\mathcal{T}_{ij}}(\pi_{\theta_i}) = -\mathbb{E}_{\mathcal{D}_{ij}\sim\pi_{\theta_i},\,\mu_{\mathcal{T}_{ij}}} \left[ \sum_{t=0}^{T-1} \gamma^t \mathbf{w}_i^\top \mathbf{r}_{ijt} \right],
\end{equation}

where $\mathcal{D}_{ij}$ denotes a trajectory collected on task $\mathcal{T}_{ij}$ under policy $\pi_{\theta_i}$ and scalarisation weight vector $\mathbf{w}_i$.

\section{Methods} \label{sec:methods}
MOEA-based methods use population-based parallel search to explore the search space efficiently and generate diverse solutions. This helps overcome the limitation of typical Meta-MORL methods, whose search space is often restricted to local regions~\cite{hanneAReviewOfTheEvolution2025}. This section presents the existing Meta-MORL algorithm and its integration with MOEA for parallel search.

\subsection{Background}

MERLION builds on three foundations: MAML's inner--outer optimisation principle, Meta-MORL's preference-conditioned adaptation, and evolutionary search in the scalarisation weight space~\cite{finnModelAgnosticMetaLearningFast2017,chenMetaLearningMultiobjectiveReinforcement2019}. We first recall the standard definitions of dominance, nondominated sets, and hypervolume that underpin both the evolutionary fitness score and the performance metrics used throughout this work. These definitions are also used in the analytical motivation in Section~\ref{sec:analytical_motivation}.

\begin{defn} \label{def:weak_dominate}
    For $\mathbf{u},\mathbf{v}\in\mathbb{R}^{B}$, we say that $\mathbf{u}$ \textbf{weakly dominates} $\mathbf{v}$
    (denoted $\mathbf{u}\succeq \mathbf{v}$) if $u_b \ge v_b, \forall b\in\{1,\dots,B\}.$
\end{defn}

\begin{defn}
    We say that $\mathbf{u}$ \textbf{dominates} $\mathbf{v}$
    (denoted $\mathbf{u}\succ \mathbf{v}$) if $\mathbf{u}\succeq \mathbf{v}$ and
    $\exists b\in\{1,\dots,B\} \text{ such that } u_b > v_b.$
\end{defn}

\begin{defn}
    For a finite solution set $\mathcal{H}\subset\mathbb{R}^{B}$, let $\mathcal{ND}(\cdot)$ denote nondominated subset, so that:
    \begin{equation}
    \mathcal{ND}(\mathcal{H}) := \left\{ \mathbf{x}\in\mathcal{H} \;\middle|\; \nexists\, \mathbf{h}\in\mathcal{H}\text{ such that }\mathbf{h}\succ \mathbf{x} \right\}.
    \end{equation}
\end{defn}

\begin{defn}
    The region weakly dominated by $\mathcal{H}$ is defined as
    \begin{equation}
    \mathrm{Dom}(\mathcal{H}) := \left\{ \mathbf{x}\in\mathbb{R}^{B} \;\middle|\; \exists\, \mathbf{h}\in\mathcal{H}\text{ such that }\mathbf{h}\succeq \mathbf{x} \right\}.
    \end{equation}
\end{defn}

\begin{defn} \label{def:hypervolume}
    Fix a reference point $\mathbf{x}^{\mathrm{ref}} \in \mathbb{R}^{B}$ that is dominated by all attainable objective vectors. 
    The hypervolume ($HV$) of a set $\mathcal{H} \subset \mathbb{R}^{B}$ is defined as the $B$-dimensional volume, i.e.\ the Lebesgue measure $\nu(\cdot)$, of the region weakly dominated by the nondominated points in $\mathcal{H}$ relative to $\mathbf{x}^{\mathrm{ref}}$, which is calculated by:
    \begin{equation} \label{eq:hv}
    HV(\mathcal{H};\mathbf{x}^{\text{ref}}) = \nu\!\left(\bigcup_{\mathbf{x}\in \mathcal{ND}(\mathcal{H})} [\mathbf{x}^{\text{ref}},\mathbf{x}] \right),
    \end{equation}
    where
    \begin{equation}
    [\mathbf{x}^{\text{ref}},\mathbf{x}] := \prod_{b=1}^{B} [x^{\text{ref}}_b,x_b].
    \end{equation}
\end{defn}

\subsubsection{Meta-MORL for SC Optimisation}

Let $\mathcal{T}_y \sim p(\mathcal{T})$ denote a task sampled from a task distribution, $\pi_\theta$ a policy parameterised by $\theta$, and $\mathbf{w}_y \in \Delta^{B-1}$ a scalarisation weight vector for $B$ objectives, where $\gamma \in [0,1]$ is the discount factor and $\mathbf{r}_{yt}$ is the vector reward at time $t$. The expected scalarised discounted return on task $\mathcal{T}_y$ is
\begin{equation}
    \label{eq:bg_scalarised_objective}
    J(\theta;\mathcal{T}_y,\mathbf{w}_y)
    =
    \mathbb{E}_{\pi_{\theta},\mathcal{T}_y}
    \left[\sum_{t=0}^{T-1}\gamma^t \mathbf{w}_y^\top \mathbf{r}_{yt}\right],
\end{equation}
with corresponding RL loss
\begin{equation}
    \label{eq:bg_scalarised_loss}
    \mathcal{L}_{\mathcal{T}_y,\mathbf{w}_y}(\pi_\theta) = -J(\theta;\mathcal{T}_y,\mathbf{w}_y).
\end{equation}
MAML~\cite{finnModelAgnosticMetaLearningFast2017} learns an initialisation $\theta$ that adapts quickly via a small number of inner gradient steps. Chen et al.~\cite{chenMetaLearningMultiobjectiveReinforcement2019} extended this to MORL by treating each task--preference pair $(\mathcal{T}, \mathbf{w})$ as a scalarised subproblem, so that meta-training seeks an initialisation that adapts efficiently across both task and preference variation.

At deployment, each meta-parameter $\theta_i$ is fine-tuned independently on a new, unseen task $\mathcal{T}^{\mathrm{new}}$. For each $\theta_i$, a set of $M$ local scalarisation weight vectors $\{\mathbf{w}_{im}\}_{m=1}^{M}$ is constructed around its associated preference vector. Starting from $\theta_i$, independent fine-tuning under each $\mathbf{w}_{im}$ produces $M$ preference-conditioned adapted policies $\{\theta'_{im}\}_{m=1}^{M}$. Since Pareto optimality is defined in the original multi-objective space, each adapted policy is evaluated by its expected discounted vector return:
\begin{equation}
    \label{eq:bg_vector_return}
    \mathbf{F}\!\left(\theta'_{im};\mathcal{T}^{\mathrm{new}}\right)
    =
    \mathbb{E}_{\pi_{\theta'_{im}},\mathcal{T}^{\mathrm{new}}}
    \left[\sum_{t=0}^{T-1}\gamma^t \mathbf{r}_t\right]
    \in \mathbb{R}^{B}.
\end{equation}
The PF approximation $\widehat{\mathcal{P}}$ is then obtained by extracting the 
nondominated subset of these vector returns. Hence, conventional Meta-MORL learns 
a single good starting point for preference-conditioned adaptation rather than 
directly learning a diverse Pareto set during meta-training.

\subsubsection{MOEA-inspired Techniques}

MOEAs maintain a population of candidate solutions and refine it through three  operators~\cite{fonseca1993genetic,debFastElitistMultiobjective2002, zitzlerSPEA2ImprovingStrength2001}: (i)~\emph{selection}, which favours individuals by fitness or rank; (ii)~\emph{crossover}, which recombines two parents, for weight vectors on the simplex, arithmetic crossover $\tilde{\mathbf{z}} = \lambda \mathbf{z}_i + (1-\lambda)\mathbf{z}_h$ ($\lambda \in [0,1]$) preserves feasibility; and (iii)~\emph{mutation}, which adds random perturbations $\hat{\mathbf{z}} =  \tilde{\mathbf{z}} + \epsilon$, $\epsilon \sim \mathcal{M}_{\mathrm{mut}}$, to maintain exploration. Almost all MOEAs also apply \emph{elitist replacement}, retaining the best individuals from the combined parent--offspring pool. In MERLION, these operators act directly in the scalarisation weight space, driving the preference distribution toward better convergence and broader Pareto coverage.

\subsection{Proposed Algorithm: MERLION} \label{sec:proposed_algorithm}
In this study, we propose MERLION, an evolutionary Meta-MORL framework that maintains a population of meta-policies and updates them through two coupled processes. Policy parameters are learned via gradient-based meta-updates (in the spirit of MAML), while evolutionary operators act in the weight space to adapt the preference vectors that condition each meta-policy. This co-evolution of weights and meta-policies aims to encourage both convergence and diverse coverage of the trade-off surface.

Our proposed MERLION algorithm consists of three primary stages: meta-training, evolutionary procedure, and fine-tuning phase. In meta-training, the Meta-MORL framework is predominantly used with an extension into a population-based strategy. For the evolutionary procedure, an MOEA-based strategy is used, and the fine-tuning phase adopts actor-critic RL with modifications to align with the evolutionary framework.


\subsubsection{Meta-training}
Algorithm~\ref{algo:main_algo} is initialised with $P$ meta-parameters $\{\theta_i\}_{i=1}^{P}$ and a corresponding set of weight vectors $\{\mathbf{w}_i\}_{i=1}^{P}$ sampled uniformly from the simplex $\Delta^{B-1}$, ensuring broad initial coverage of the preference space. An empty \texttt{archive} is also created to store meta-parameters and their associated weight vectors throughout training. Each pair $(\theta_i, \mathbf{w}_i)$ defines a scalarised subproblem, where $\theta_i$ is the meta-parameter specialised for preference region $\mathbf{w}_i$ and serves as the initialisation point for fast adaptation under that preference. For each subproblem $i$, a batch of tasks $\{\mathcal{T}_{ij}\}_{j=1}^{J}$ is sampled from $p(\mathcal{T})$ and trajectories $\mathcal{D}_{ij}$ are collected under $\pi_{\theta_i}$. These trajectories are used to compute the inner adaptation update:
\begin{equation}
\label{eq:adaptation}
\theta'_{ij} = \theta_i - \alpha \nabla_{\theta_i} \mathcal{L}_{\mathcal{T}_{ij}}(\pi_{\theta_i}),
\end{equation}
where $\alpha$ is the inner-loop step size and $\theta'_{ij}$ is the task-adapted parameter for task $\mathcal{T}_{ij}$ under subproblem $i$. In this study, our meta-objective for each $\theta_i$ is as follows:
\begin{equation}
\label{eq:meta_objective}
\min_{\theta_i} \sum_{\mathcal{T}_{ij} \sim p(\mathcal{T})}{\mathcal{L}_{\mathcal{T}_{ij}}(\pi_{\theta'_{ij}})}=\sum_{\mathcal{T}_{ij} \sim p(\mathcal{T})}{\mathcal{L}_{\mathcal{T}_{ij}}(\pi_{\theta_i-\alpha \nabla_{\theta_i}\mathcal{L}_{\mathcal{T}_{ij}}(\pi_{\theta_i})})}.
\end{equation}

\noindent After completion of the adaptation on the batch of tasks, each meta-parameter $\theta_i$ is further updated in the meta-learning phase using the $\mathcal{D}'_{ij}$ that uses $\pi_{\theta'_{ij}}$. The meta-parameters $\theta_i$ are calculated based on the aggregated post-adaptation losses over scalarised tasks $\mathcal{T}_{ij}$ drawn from $p(\mathcal{T})$. Hence, the meta-policy update that considers the meta step size $\beta$ follows:



\begin{equation}
\label{eq:meta_update}
\theta_i \leftarrow \theta_i - \beta \nabla_{\theta_i}\sum_{j=1}^{J}\mathcal{L}_{\mathcal{T}_{ij}}(\pi_{\theta'_{ij}}).
\end{equation}

\noindent Once all meta-policies $\pi_{\theta_i}$ associated with the weight vectors $\mathbf{w}_i$ are updated, the pairs $(\theta_i, \mathbf{w}_i)$ are stored in the \texttt{archive} to perform the evolutionary update in the next stage.

\begin{algorithm}[!th]
\caption{Meta-training phase of MERLION (population-based meta-learning).}
\label{algo:main_algo}
\begin{algorithmic}[1]
\REQUIRE population size $P$, task distribution $p(\mathcal{T})$, weight distribution $p(\mathbf{w})$, step sizes $\alpha,\beta$
\STATE \textbf{Initialise:} meta-parameters $\{\theta_i\}_{i=1}^{P}$, weights $\{\mathbf{w}_i \sim p(\mathbf{w})\}_{i=1}^{P}$, $\texttt{archive} \leftarrow \emptyset$
\WHILE{not done}
    \FOR{$i=1,2,\dots,P$}
        \STATE Sample a batch of tasks $\{\mathcal{T}_{ij}\}_{j=1}^{J} \sim p(\mathcal{T})$
        \FOR{$j=1,2,\dots,J$}
            \STATE Collect trajectories $\mathcal{D}_{ij}$ using $\pi_{\theta_i}$ on task $\mathcal{T}_{ij}$ with weight $\mathbf{w}_i$
            \STATE Compute inner loss $\mathcal{L}_{\mathcal{T}_{ij}}(\pi_{\theta_i})$ from $\mathcal{D}_{ij}$
            \STATE Inner adaptation update, see Eq.~\eqref{eq:adaptation}
            \STATE Collect trajectories $\mathcal{D}'_{ij}$ using $\pi_{\theta'_{ij}}$ on task $\mathcal{T}_{ij}$ with weight $\mathbf{w}_i$
        \ENDFOR
        \STATE Meta-update, see Eq.~\eqref{eq:meta_update}

        \STATE $\texttt{archive} \leftarrow \mathit{Update}(\texttt{archive},(\theta_i,\mathbf{w}_i))$
    \ENDFOR
    \STATE $\texttt{archive} \leftarrow \mathit{EvolveWeights}(\texttt{archive})$
\ENDWHILE
\STATE \RETURN $\texttt{archive}$ containing $\{(\theta_i,\mathbf{w}_i)\}_{i=1}^{P}$
\end{algorithmic}
\end{algorithm}

\begin{figure}[!th]
\centering
\includegraphics[width=\linewidth]{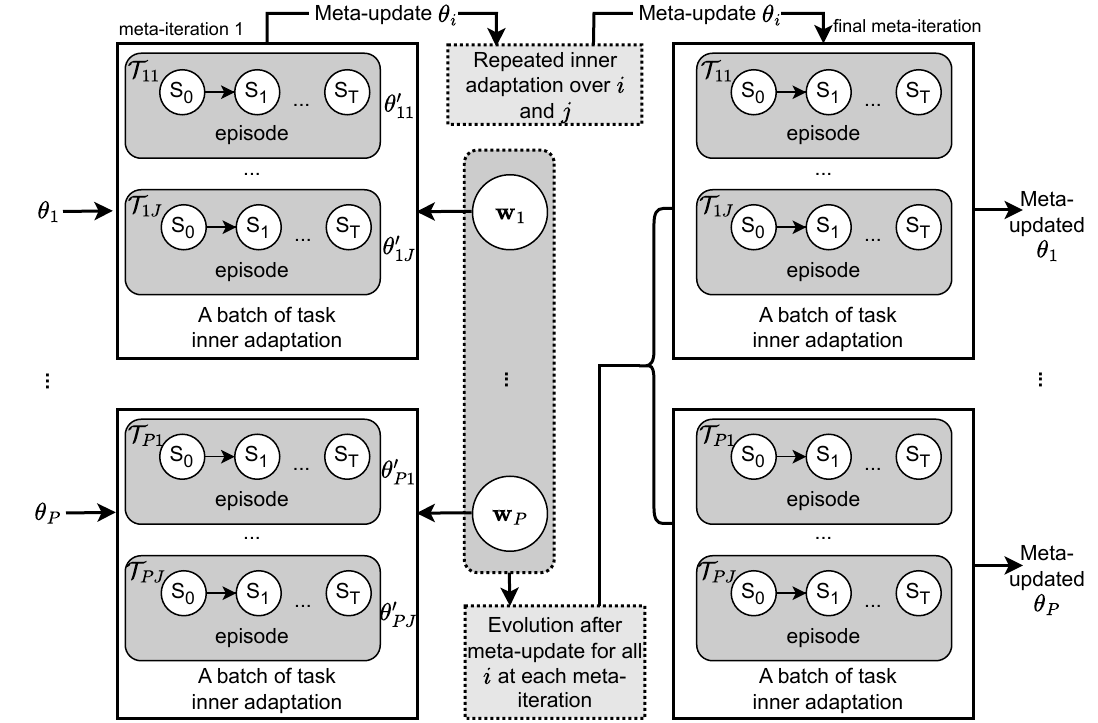}
\caption{Meta-training and adaptation phases in MERLION as given in Algorithm~\ref{algo:main_algo}. Initial meta-policies $\theta_i$ are trained independently in the inner adaptation phase using $\mathbf{w}_i$ on a batch of tasks $\mathcal{T}_{ij}$. Outer meta-update applies to each $\theta_i$ with the evolutionary procedure to all $\mathbf{w}_i$ after each meta-iteration.}
\label{fig:merlion_algo}
\end{figure}

\subsubsection{Evolutionary Procedure}
In Meta-MORL, each meta-policy is associated with a fixed scalarisation weight vector generated during meta-training, with no mechanism to refine the preference distribution over time. In MERLION, we propose to apply MOEA-inspired operators to these weight vectors so that the preference distribution evolves across generations toward better convergence and broader coverage of the trade-off space.

After each meta-update, the \texttt{archive} stores a population of tuples $(\theta_i,\mathbf{w}_i)$, where $\theta_i$ denotes the meta-parameters and $\mathbf{w}_i$ denotes the associated scalarisation weight vector. Following Algorithm~\ref{algo:evolutionary}, MERLION first evaluates the current population using a scalar fitness score, then generates $P$ offspring through selection, crossover or mutation, assigns each offspring the meta-parameters of the fitter parent, and finally recomputes fitness over the combined parent-offspring population before retaining the top $P$ individuals. In this way, the evolutionary routine updates the weight vectors while preserving strong solutions through elitist survival selection.

\paragraph{Selection}
We adopt rank-based parent selection over a scalar fitness score that combines convergence and diversity. Let $\mathbf{r}_i \in \mathbb{R}^B$ denote the objective vector associated with individual $i$, and let $\mathbf{w}_i$ denote its scalarisation weight vector. The fitness of individual $i$ is defined as
\begin{equation}
\label{eq:fitness_score}
Q_i = \kappa \, \Delta HV_i + (1-\kappa)\,\Delta E_i,
\end{equation}
where fitness-mixing coefficient $\kappa \in [0,1]$ balances convergence and diversity. The convergence term is the marginal hypervolume contribution in objective space,
\begin{equation}
\label{eq:marginal_hv}
\Delta HV_i = HV(\mathcal{H}) - HV\!\left(\mathcal{H} \setminus \{\mathbf{r}_i\}\right),
\end{equation}
where $\mathcal{H}=\{\mathbf{r}_1,\dots,\mathbf{r}_P\}$ is the set of population objective vectors. The diversity term is computed from an average log-distance score in the weight space,
\begin{equation}
\label{eq:entropy}
\Delta E_i = \frac{1}{P-1}\sum_{h \neq i} \log\!\left(\|\mathbf{w}_i-\mathbf{w}_h\|_2+\epsilon\right),
\end{equation}
where $\epsilon>0$ is a small constant for numerical stability. While marginal hypervolume contribution is a standard MOEA metric~\cite{debFastElitistMultiobjective2002}, its combination with a weight-space diversity term in Eq.~\eqref{eq:fitness_score} is specific to MERLION, as it jointly accounts for objective-space quality and preference-space coverage during meta-training.

After computing $\{Q_i\}_{i=1}^{P}$, parent-$i$ is selected according to rank-based probabilities. Let $\tau_i \in \{0,\dots,P-1\}$ denote the rank of individual $i$ after sorting in descending order of $Q_i$, where $\tau_i=0$ is the best rank. Then parent-$i$ is sampled with probability
\begin{equation}
\label{eq:rank_selection}
p_{\mathrm{sel}}(i) = \frac{P-\tau_i}{\sum_{h=1}^{P}(P-\tau_h)}.
\end{equation}
This gives higher-fitness individuals a greater chance of reproduction while maintaining stochasticity in the search. After $P$ offspring are generated and added to the \texttt{archive}, MERLION recomputes fitness scores on the combined parent-offspring population and retains the top $P$ individuals via elitist survival selection.

\paragraph{Crossover}
After selecting parent-$i$, MERLION chooses the second parent as the farthest weight vector from $\mathbf{w}_i$ in Euclidean space:
\begin{equation}
    \label{eq:crossover_mating}
    \mathbf{w}_h = \arg\max_{\mathbf{w}_k \in \mathcal{W},\, \mathbf{w}_k \neq \mathbf{w}_i} \|\mathbf{w}_i-\mathbf{w}_k\|_2,
\end{equation}
where $\mathcal{W}$ denotes the current population of weight vectors. This departs from standard nearest-neighbour mating used in decomposition-based MOEAs~\cite{qingfuzhangMOEAMultiobjectiveEvolutionary2007} and instead promotes recombination across well-separated preference regions, thereby encouraging broader coverage of trade-off directions. Crossover is then applied with probability $\eta_{\mathrm{cross}}$ using arithmetic interpolation, which preserves feasibility on the simplex:
\begin{equation}
    \label{eq:crossover}
    \mathbf{w}_{\mathrm{child}} = \lambda \mathbf{w}_i + (1-\lambda)\mathbf{w}_h, \qquad \lambda \sim \mathrm{Uniform}(0,1).
\end{equation}
Arithmetic crossover generates offspring as convex combinations of parent preferences, thereby preserving feasibility without requiring repair operators and naturally interpolating between neighbouring trade-off directions. The offspring inherits the meta-parameters of the fitter parent according to the fitness score $Q$.

\paragraph{Mutation}
If crossover is not applied, the offspring is generated through mutation. The selected parent weight vector is perturbed by additive uniform noise,
\begin{equation}
\label{eq:mutation}
\tilde{\mathbf{w}}_{\mathrm{child}} = \mathbf{w}_i + \boldsymbol{\epsilon}, \qquad \boldsymbol{\epsilon} \sim \mathrm{Uniform}(-\delta,\delta)^{B},
\end{equation}
where $\delta>0$ controls the mutation strength. The perturbed vector is then repaired by clipping negative components and renormalising to satisfy simplex constraints:
\begin{equation}
\label{eq:mutation_projection}
\mathbf{w}_{\mathrm{child}} = \frac{\max(\tilde{\mathbf{w}}_{\mathrm{child}},0)}{\sum_{b=1}^{B}\max(\tilde{w}_{\mathrm{child},b},0)}.
\end{equation}
As in the crossover case, the offspring inherits the meta-parameters of the fitter parent before being inserted into the \texttt{archive}.

\paragraph{Operator design in the weight space}
In MERLION, evolutionary search acts in the scalarisation weight space, while policy parameters are updated by gradient-based meta-learning. Each meta-policy is paired with a weight vector that conditions scalarisation of the vector reward before PPO-style optimisation. Because near-duplicate offspring in weight space provide little additional coverage of the preference simplex, each offspring is generated by exactly one operator: arithmetic crossover with probability $\eta_{\text{cross}}$, and mutation otherwise, i.e., $\eta_{\text{mut}} = 1-\eta_{\text{cross}}$.

This complementary scheme reduces operator control to a single exploration--exploitation parameter: larger $\eta_{\text{cross}}$ favours exploitation by recombining existing trade-offs, whereas smaller $\eta_{\text{cross}}$ favours exploration by allowing mutation to discover new preference regions. Importantly, local improvement in policy space is not lost: after each evolutionary step, gradient-based meta-updates refine each $\theta_i$ under its updated $\mathbf{w}_i$, providing implicit exploitation in policy space. Additional local exploitation is performed explicitly during fine-tuning, where local weight perturbations around each archived $\mathbf{w}_i$ are followed by independent gradient-based policy updates, ensuring that exploration in preference space is always coupled with policy-level exploitation.

\begin{figure*}[t]
\centering
\subfloat[\footnotesize Meta-Learning and adaptation trajectory \label{subfig:meta-learning_trajectory}]{%
  \includegraphics[
    width=.32\linewidth,
    trim=7 7 7 7, 
    clip
  ]{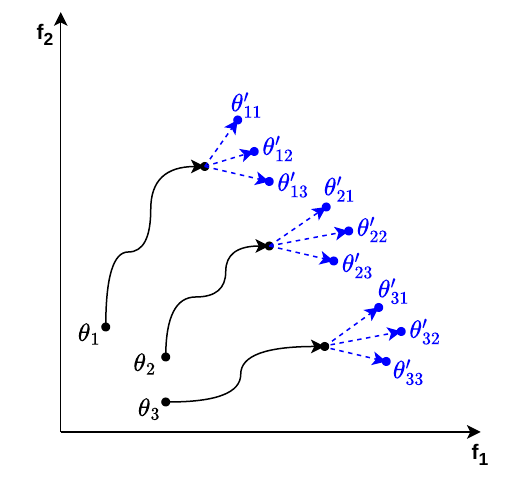}%
}\hfil 
\subfloat[\footnotesize Evolutionary effect on task-adapted policies \label{subfig:evolutionary_effect}]{%
  \includegraphics[
    width=.32\linewidth,
    trim=7 7 7 7,
    clip
  ]{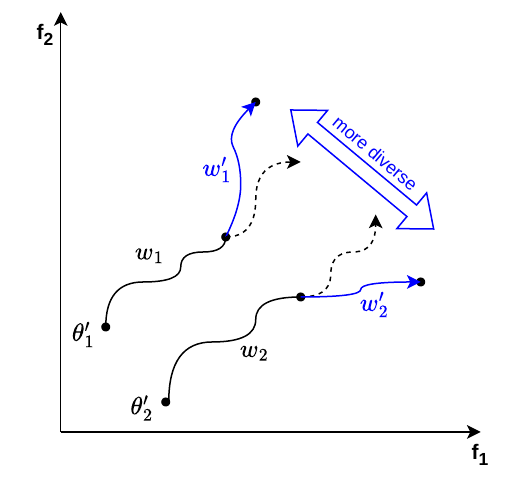}%
}\hfil 
\subfloat[\footnotesize Fine-tuning with local perturbation \label{subfig:fine-tuning}]{%
  \includegraphics[
    width=.32\linewidth,
    trim=7 7 7 7,
    clip
  ]{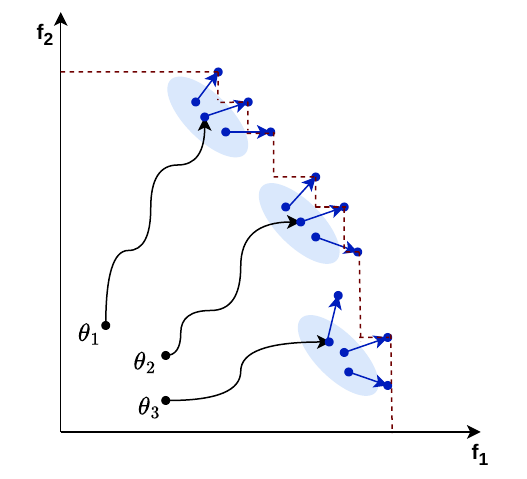}%
}
\caption{Meta-policies and task-adapted policies movement along the training. Subfig.~\ref{subfig:meta-learning_trajectory} shows the population-based meta-policy trajectories (black lines) that are independent to each other, where each meta-policy is adapted to a batch of tasks (dotted blue lines) for next meta-update. Subfig.~\ref{subfig:evolutionary_effect} shows the evolving weights (i.e., $w_1$ and $w_2$ become $w'_1$ and $w'_2$) tend to encourage solutions to move away. Subfig~\ref{subfig:fine-tuning} shows local perturbation in fine-tuning phase, followed with non-dominated sorting to form PF approximation (solutions that are crossed by dotted red line).}
\label{fig:merlion_mechanism}
\end{figure*}


\begin{algorithm}[!ht]
\caption{Evolutionary procedure in MERLION}
\label{algo:evolutionary}
    \begin{algorithmic}[1]
    \REQUIRE $\texttt{archive}$, population size $P$, crossover probability $\eta_{\text{cross}}$, mutation strength $\delta$, fitness-mixing coefficient $\kappa$
    \STATE Extract $\{(\theta_i,\mathbf{w}_i)\}$ and compute fitness scores $\{\mathcal{Q}_i\}_{i=1}^{P}$ from $\texttt{archive}$
    \FOR{offspring $= 1, \dots, P$}
        \STATE Select parent~1: $(\theta_i, \mathbf{w}_i)$ with rank-based selection according to $\{\mathcal{Q}_i\}_{i=1}^{P}$
        \STATE Select parent~2: $(\theta_h, \mathbf{w}_h)$ using distance-aware mating using Eq.~\eqref{eq:crossover_mating}
        
        \IF{ $\texttt{rand()} < \eta_{\text{cross}}$ }
            \STATE Perform crossover to obtain $\mathbf{w}_{child}$ using Eq.~\eqref{eq:crossover}
        \ELSE
            \STATE Apply mutation to $\mathbf{w}_{child}$ using Eq.~\eqref{eq:mutation}
        \ENDIF

        \STATE $\theta_{child} = \theta_i$ if $\mathcal{Q}_i \geq \mathcal{Q}_h$, else $\theta_h$

        \STATE Add $(\theta_{child},\mathbf{w}_{child})$ to $\texttt{archive}$
    \ENDFOR
    \STATE Recompute $\mathcal{Q}$ for all individuals in $\texttt{archive}$
    \STATE Retain the $P$ individuals with the highest $\mathcal{Q}$ in $\texttt{archive}$
    \STATE \RETURN updated $\texttt{archive}$
    \end{algorithmic}
\end{algorithm}

\subsubsection{Fine-tuning}
After meta-training, the \texttt{archive} contains a population of meta-policies and their associated scalarisation weight vectors, $\{(\theta_i,\mathbf{w}_i)\}_{i=1}^{P}$. Fine-tuning is then performed independently for each meta-policy on a target task, as shown in Algorithm~\ref{algo:fine_tuning}. To improve local coverage during fine-tuning, MERLION applies local perturbations around each archived weight vector. For every archived pair $(\theta_i,\mathbf{w}_i)$, we construct a local weight set consisting of the original weight vector $\mathbf{w}_i$ and $M$ perturbed weight vectors sampled from its neighbourhood using perturbation scale $\delta_{pert}$. Let $\theta_{im}$ denote the policy parameters obtained after fine-tuning the $i$-th archived meta-policy under local weight $\mathbf{w}_{im}$, and let its evaluated vector return on the target task $\mathcal{T}^{\mathrm{new}}$ be
\begin{equation}
\mathbf{r}_{im} := F(\theta_{im};\mathcal{T}^{\mathrm{new}}) = \mathbb{E}_{\pi_{\theta_{im}},\,\mathcal{T}^{\mathrm{new}}} \left[ \sum_{t=0}^{T-1}\gamma^t \mathbf{r}_t \right] \in \mathbb{R}^{B}.
\end{equation}

Each perturbed weight is projected back onto the simplex before use. For each local weight vector, the corresponding archived meta-policy is loaded and used to initialise an independent fine-tuning run. Then a fresh policy is trained under the selected scalarisation weight for a fixed fine-tuning budget, with periodic evaluation every fixed number of environment steps, resulting in $K$ fine-tuning intervals. After all local fine-tuning runs have been completed, non-dominated sorting is applied to the collected candidate solutions to obtain the final PF approximation. In this way, fine-tuning preserves the good initialisations learned during meta-training while increasing the density of candidate solutions around promising preference regions.


\begin{algorithm}[!th]
\caption{Fine-tuning phase in MERLION}
\label{algo:fine_tuning}
    \begin{algorithmic}[1]
    \REQUIRE \texttt{archive} containing $\{(\theta_i, \mathbf{w}_i)\}_{i=1}^{P}$, fine-tuning intervals $K$, local perturbation strength $\delta_{pert}$, local perturbation size per meta-policy $M$
    \STATE Initialise $\texttt{candidates} \leftarrow \emptyset$
    \FOR{$i = 1,2,\dots,P$}
        \STATE Construct local weight set $\mathcal{W}_i \gets \{\mathbf{w}_i\} \cup \{\mathbf{w}_{im}\}_{m=1}^{M}$
        \FORALL{$\mathbf{w}_{im} \in \mathcal{W}_i$}
            \IF{$\mathbf{w}_{im} \neq \mathbf{w}_i$}
                \STATE $\mathbf{w}_{im} \gets \mathbf{w}_i + \boldsymbol{\epsilon}_{im} ,\boldsymbol{\epsilon}_{im} \sim \mathrm{Uniform}(-\delta_{pert}, \delta_{pert})^B$
                \STATE Repair $\mathbf{w}_{im}$ using Eq.~\eqref{eq:mutation_projection}
            \ENDIF
            \STATE Clone meta-parameters: $\theta_{im} \gets \theta_i$
            \FOR{$k = 1,2,\dots,K$}
                \STATE Update $\pi_{\theta_{im}}$ under $\mathbf{w}_{im}$
                \STATE Evaluate updated $\pi_{\theta_{im}}$ and record $\mathbf{r}_{im}$
            \ENDFOR
            \STATE Add final $(\theta_{im}, \mathbf{w}_{im}, \mathbf{r}_{im})$ to $\texttt{candidates}$
        \ENDFOR
    \ENDFOR
    \STATE $\mathcal{PF} \gets \mathit{NonDominated}(\texttt{candidates})$
    \STATE \RETURN fine-tuned solution set $\mathcal{PF}$
    \end{algorithmic}
\end{algorithm}

\subsubsection{Analytical Motivation for the Population-based Meta-Policies}
\label{sec:analytical_motivation}

We provide a simple analytical motivation for MERLION's population-based design. Rather than giving a full convergence guarantee, we relate multiple meta-policies to standard hypervolume properties (see Definitions~\ref{def:weak_dominate} to~\ref{def:hypervolume}) and to a plausible optimisation mechanism under conflicting preferences. Under limited fine-tuning steps, conventional meta-learning methods may produce concentrated solution sets, as adaptation tends to explore only a narrow portion of the objective space. Such concentration restricts the coverage of the PF and can reduce the resulting improvement in hypervolume. To formalise this, we ground the argument in two standard properties of the hypervolume indicator: its monotonicity under set expansion and the strict gain it yields when a nondominated point is added. These properties are important because they allow us to reason about when and why adding more diverse candidates to an archive is guaranteed to improve solution quality, providing a theoretical basis for MERLION's population-based design.

We now introduce two standard lemmas that characterise the monotonic behaviour~\cite{ulrichBoundingEffectivenessHypervolumeBased2012} (Lemma~\ref{lemma:hv_monotonicity}) and gain~\cite{guerreiroHypervolumeIndicatorComputational2022} (Lemma~\ref{lemma:hv_gain}) of hypervolume under set expansion. These lemmas do not yet depend on the specifics of MERLION, but they provide the basic foundation for interpreting how archive updates may affect hypervolume. We refer the reader to the original source for the full proof.

\begin{lemma} \label{lemma:hv_monotonicity}
The monotonicity occurs if $\mathcal{K}\subseteq \mathcal{H}$, then
\begin{equation}
HV(\mathcal{K};\mathbf{x}^{\text{ref}}) \le HV(\mathcal{H};\mathbf{x}^{\text{ref}}).
\end{equation}
\end{lemma}

\textit{Proof.} See Appendix~\ref{app:proof-lemma1}.

\begin{lemma} \label{lemma:hv_gain}
    A new nondominated point yields a strict hypervolume gain. Let $\mathcal{H}$ be any finite set, and let $\mathbf{x}\notin \mathrm{Dom}(\mathcal{H})$, i.e., $\mathbf{x}$ is not dominated by any point in $\mathcal{H}$. Assume further that $\mathbf{x}$ is strictly better than the reference point $\mathbf{x}^{\mathrm{ref}}$ in every objective. Then
    \begin{equation}
    HV(\mathcal{H}\cup\{\mathbf{x}\};\mathbf{x}^{\text{ref}}) =
    HV(\mathcal{H};\mathbf{x}^{\text{ref}}) + \Delta HV(\mathbf{x}\mid \mathcal{H}),
    \end{equation}
    with
    \begin{equation}
    \Delta HV(\mathbf{x}\mid \mathcal{H}) > 0.
    \end{equation}
\end{lemma}

\textit{Proof.} See Appendix~\ref{app:proof-lemma2}.

Lemmas~\ref{lemma:hv_monotonicity} and ~\ref{lemma:hv_gain} are generic properties of hypervolume and do not yet depend on the specifics of MERLION. To relate them to our population-based setting, we next formalise the event that a newly generated candidate improves the current archive.

\paragraph{Population advantage via improvement probability}

\noindent \textbf{Assumptions in MERLION.}
To keep the discussion transparent, we used the following assumptions to interpret why a population of meta-policies
$\{\theta_i\}_{i=1}^{P}$ (each paired with a scalarisation weight $\mathbf{w}_i$) is intended to yield higher hypervolume than a single meta-policy.
\begin{enumerate}[label=(A\arabic*), leftmargin=*]
    \item \textit{Fixed reference point and direction.} Maximisation is used for all objectives and the
    reference point $\mathbf{x^{\text{ref}}}$ is fixed.
    
    \item \textit{Nondominated retention (archive/elitist retention).}
    Evaluations are retained in an external nondominated archive or, equivalently, the procedure does not discard
    previously found nondominated objective vectors. In particular, for any current archive $\mathcal{H}$ and newly
    evaluated set $\mathcal{K}$, the retained set includes $\mathcal{ND}(\mathcal{H}\cup\mathcal{K})$.
    
    \item \textit{Preference diversity.} The population weights $\{\mathbf{w}_i\}_{i=1}^{P}$ are sufficiently
    spread over the simplex (i.e., not all concentrated around a single preference), so that the resulting candidates
    are not perfectly correlated.
    
    \item \textit{Non-trivial objective conflict.} There exist weights $\mathbf{w}_{j},\mathbf{w}_{k}$ such that
    their induced update directions can disagree. 
\end{enumerate}

These assumptions aim to interpret the role of population-based search in our setting. On the other hand, since Lemma~\ref{lemma:hv_gain} implies that a strict hypervolume gain occurs whenever a generated candidate is not dominated by the retained archive, we define the corresponding archive-improvement event as follows.

\begin{defn}
Let $\mathcal{H}$ denote the current retained set (archive), and let
\begin{equation}
\mathcal{K}^{(P)}=\{\mathbf{x}_{1},\ldots,\mathbf{x}_{P}\},
\qquad
\mathbf{x}_{i}:=\mathbf{f}(\theta_i),
\end{equation}
be the $P$ candidate objective vectors produced by $P$ meta-policies, where $\theta_i$ is associated with the preference vector $\mathbf{w}_i$. Define the archive-improvement event
\begin{equation}
X_i := \{\mathbf{x}_i \notin \mathrm{Dom}(\mathcal{H})\}.
\end{equation}
Thus, $X_i$ denotes the event that candidate $i$ is not weakly dominated by the current archive $\mathcal{H}$ and therefore improves the retained set. Consequently, $\bigcup_{i=1}^{P} X_i$ denotes the event that at least one candidate improves the archive, while $X_i^c$ denotes the event that candidate $i$ does not improve the archive.
\end{defn}

\begin{props}[Population improvement probability]
\label{prop:population_improvement}
Let $X_1,\dots,X_P$ be the archive-improvement events defined above. Suppose these events are independent and satisfy $\Pr(X_i)\ge \rho$ for all $i$, where $\rho\in(0,1)$. Then
\begin{equation}
\Pr\!\left(\bigcup_{i=1}^{P} X_i\right) = 1-\Pr\!\left(\bigcap_{i=1}^{P} X_i^c\right) \ge 1-(1-\rho)^P.
\end{equation}
Moreover, for $P=1$,
\begin{equation}
\Pr(\text{any improvement})=\Pr(X_1)\ge \rho,
\end{equation}
whereas for $P>1$,
\begin{equation}
\Pr(\text{any improvement})\ge 1-(1-\rho)^P > \rho.
\end{equation}
\end{props}

\textit{Proof.} See Appendix~\ref{app:proof-prop1}.

\noindent\textbf{Interpretation.}
Under independence, the probability that at least one candidate improves the archive increases with the population size. Together with Lemma~\ref{lemma:hv_monotonicity}, Lemma~\ref{lemma:hv_gain}, and Assumption~(A2), this suggests that a population of meta-policies may be more likely than a single meta-policy to produce a strict hypervolume increase. However, this argument does not explain why different meta-policies should generate complementary rather than redundant candidates. We therefore next consider how preference specialisation may reduce cross-preference gradient interference.

\paragraph{Why specialisation helps}
When a single shared meta-policy is trained across heterogeneous preference vectors, the resulting gradients may interfere. By contrast, assigning different preferences to different meta-policies allows each member of the population to specialise, which may reduce cross-preference gradient cancellation.

Given a task $\mathcal{T}$, let $\mathbf{F}(\theta,\mathcal{T})\in\mathbb{R}^{B}$ denote the expected discounted vector return induced by policy $\pi_{\theta}$.
For a scalarisation weight vector $\mathbf{w}\in\Delta^{B-1}$, define the scalarised objective
\begin{equation}
J(\theta;\mathcal{T},\mathbf{w}) := \mathbf{w}^{\top}\mathbf{F}(\theta,\mathcal{T}), \qquad g(\theta;\mathcal{T},\mathbf{w}) := \nabla_{\theta}J(\theta;\mathcal{T},\mathbf{w}).
\end{equation}
A single meta-policy trained to serve multiple preferences drawn from a distribution $p(\mathbf{w})$
therefore follows
\begin{equation}
\nabla_{\theta}\,\mathbb{E}_{\mathbf{w}\sim p(\mathbf{w})}\!\left[J(\theta;\mathcal{T},\mathbf{w})\right] = \mathbb{E}_{\mathbf{w}\sim p(\mathbf{w})}\!\left[g(\theta;\mathcal{T},\mathbf{w})\right].
\end{equation}
Let $g_j := g(\theta;\mathcal{T},\mathbf{w}_{(j)})$ and $g_k := g(\theta;\mathcal{T},\mathbf{w}_{(k)})$.

By using the above information, we can formalise the specialisation using Proposition~\ref{prop:grad_interf}.

\begin{props}[gradient interference]
\label{prop:grad_interf}
If two preferences $\mathbf{w}_{(j)},\mathbf{w}_{(k)}\in\Delta^{B-1}$ induce conflicting gradients:
    \begin{equation}
    g_j^\top g_k < 0,
    \end{equation}
    
    \noindent then gradient conflict implies cancellation in the combined direction.
    \begin{equation}
    \begin{aligned}
    \|g_j + g_k\|^2
    &= \|g_j\|^2 + \|g_k\|^2 + 2 g_j^\top g_k \\
    &< \|g_j\|^2 + \|g_k\|^2.
    \end{aligned}
    \end{equation}
\end{props}


\noindent In contrast, a population $\{\theta_i\}_{i=1}^{P}$ can specialise updates by assigning one preference
vector $\mathbf{w}_i$ to each meta-policy:
\begin{equation}
\theta_i \leftarrow \theta_i + \beta\, g(\theta_i;\mathcal{T},\mathbf{w}_i),
\qquad i=1,\dots,P.
\end{equation}
Equivalently, since $\mathcal{L}_{\mathcal{T},\mathbf{w}}(\pi_{\theta})=-J(\theta;\mathcal{T},\mathbf{w})$, this can be written in the loss-minimisation form used earlier as
\begin{equation}
\theta_i \leftarrow \theta_i - \beta\, \nabla_{\theta_i} \mathcal{L}_{\mathcal{T},\mathbf{w}_i}(\pi_{\theta_i}).
\end{equation}

\textit{Proof.} See Appendix~\ref{app:proof-prop2}.

This suggests that specialisation across preferences can reduce cross-preference gradient cancellation and, together with preference diversity, encourage the discovery of complementary nondominated solutions that expand the PF set and increase hypervolume. Although both Meta-MORL and MERLION sample task variation during meta-training, the key distinction is that Meta-MORL aggregates all preference-conditioned gradients into a single shared meta-policy, whereas MERLION assigns each preference vector $\mathbf{w}_i$ to its own meta-parameter $\theta_i$. Hence, MERLION decentralises optimisation across preferences, which can mitigate cross-preference gradient interference even though each meta-policy still learns over a batch of tasks.

\subsection{Performance Measurements}\label{sec:performance_measurement}

Let $\mathcal{H}=\{\mathbf{v}^{(n)}\}_{n=1}^{N_{\mathcal{H}}}\subset\mathbb{R}^{B}$ denote the finite nondominated approximation set obtained on the target task, where each $\mathbf{v}^{(n)}$ is an expected discounted vector return as in Eq.~\eqref{eq:bg_vector_return}, $B$ is the number of objectives, and all objectives are treated in maximisation form. Let $\mathbf{x}^{\mathrm{ref}}\in\mathbb{R}^{B}$ denote a fixed reference point dominated by all attainable objective vectors. For distance-based evaluation, let $\mathcal{Y}=\{\mathbf{y}^{(m)}\}_{m=1}^{N_{\mathcal{Y}}}\subset\mathbb{R}^{B}$ denote the empirical reference PF, estimated by pooling all solutions across runs and extracting the nondominated subset.

We assess each method using complementary indicators, as no single metric fully captures both convergence and diversity~\cite{knowles2002local,zitzlerPerformanceAssessmentMultiobjective2003,brankeMultiobjectiveOptimizationInteractive2008}.

\paragraph{Hypervolume.}
Hypervolume measures the Lebesgue measure of the region dominated by $\mathcal{H}$ with respect to $\mathbf{x}^{\mathrm{ref}}$ (see Eq.~\eqref{eq:hv}). A larger hypervolume indicates a better approximation in terms of both convergence and coverage~\cite{zitzlerMultiobjectiveOptimizationUsing1998,obayashiEvolutionaryMulticriterionOptimization2007}. In this study, we use a standardised hypervolume with reference point $\mathbf{x}^{\mathrm{ref}} = (0,0,0)$, which is valid since all objectives are normalised to $[0,1]$ prior to evaluation, ensuring the reference point is dominated by all attainable objective vectors.

\paragraph{Expected Utility Metric (EUM).}
EUM evaluates preference coverage under linear scalarisation:
\begin{equation}
\label{eq:EUM}
\operatorname{EUM}(\mathcal{H}) = \mathbb{E}_{\mathbf{w}\sim p(\mathbf{w})} \left[ \max_{\mathbf{v}\in\mathcal{H}} \mathbf{w}^{\top}\mathbf{v} \right],
\end{equation}
where $\mathbf{w}\in\Delta^{B-1}$ is a preference vector sampled from an evaluation distribution $p(\mathbf{w})$ over the $(B-1)$-simplex. Larger EUM indicates that the approximation set better supports a wide range of user preferences~\cite{zintgrafQualityAssessmentMORL2015}.

\paragraph{Sparsity.}
Sparsity characterises the spacing between neighbouring solutions along each objective:
\begin{equation}
\label{eq:sparsity}
\operatorname{S}(\mathcal{H}) = \frac{1}{N_{\mathcal{H}}-1} \sum_{b=1}^{B} \sum_{i=1}^{N_{\mathcal{H}}-1} \left( \tilde{v}_{b,(i+1)}-\tilde{v}_{b,(i)} \right)^{2},
\end{equation}
where $\tilde{v}_{b,(1)}\leq \cdots \leq \tilde{v}_{b,(N_{\mathcal{H}})}$ are the objective-$b$ values of the points in $\mathcal{H}$ after sorting in ascending order along objective $b$. Smaller sparsity indicates a denser approximation set, whereas excessively small values may also suggest concentration in a narrow region of the front~\cite{pmlr-v119-xu20h}.


\paragraph{Average Hausdorff Distance (AHD).}
AHD measures the discrepancy between the approximation set $\mathcal{H}$ and the reference front $\mathcal{Y}$:
\begin{equation}
\label{eq:AHD}
\begin{aligned}
\Delta_{p}(\mathcal{H},\mathcal{Y})
= \max \Biggl\{&
\left(\frac{1}{N_{\mathcal{H}}} \sum_{\mathbf{v}\in\mathcal{H}} d(\mathbf{v},\mathcal{Y})^{p} \right)^{1/p}, \\
&
\left(\frac{1}{N_{\mathcal{Y}}} \sum_{\mathbf{y}\in\mathcal{Y}} d(\mathbf{y},\mathcal{H})^{p} \right)^{1/p}
\Biggr\}.
\end{aligned}
\end{equation}
where $d(\mathbf{z},\mathcal{A}) = \min_{\mathbf{a}\in\mathcal{A}} \lVert \mathbf{z}-\mathbf{a}\rVert_{2}$ and we use $p=2$. Lower AHD indicates a closer approximation to the reference front~\cite{schutzeUsingAveragedHausdorff2012}. Here, $\mathcal{Y}$ is constructed by pooling all solutions from all methods and extracting the nondominated subset.

\section{Experiments and Results}
This section evaluates the proposed algorithm against several benchmark methods on three SC problems varying in complexity~\cite{rachmanReinforcementLearningMultiobjective2026} and various popular RL-based problems~\cite{felten_toolkit_2023}. We conduct our experiments using Python 3.11.2 on JupyterLab with 128 GB RAM and a 16-core CPU. The SC experiments are representative rather than exhaustive: the selected simple, moderate, and complex instances are non-trivial and capture increasing operational coupling, dimensionality, and constraint complexity. We train our main experiment on SC problems using three configurations that mimic the problems our industrial partner typically encounters in real life, as shown in Table~\ref{tab:sc_environment}, using a customisable SC simulator~\cite{rachmanMultiobjectiveSequentialDecision2025}.

\begin{table}[!ht]
    \centering
    \caption{SC environments used in our experiment, where complexity levels are defined by the sizes of the action and observation spaces and by the counts of nodes and edges.}
    \label{tab:sc_environment}
    \begin{tabular}{lllll}
        \toprule
        Complexities & Action & Observation & Nodes & Edges\\
        \midrule
        Simple & 8 & 20 & 7 & 8\\
        Moderate & 21 & 49 & 13 & 21\\
        Complex & 59 & 131 & 24 & 59\\
        \bottomrule
    \end{tabular}
\end{table}

\subsection{Algorithm Settings}
This section presents the experiment settings of our proposed algorithm along with its benchmark methods. Each instantiation runs across 10 seeds and obtains 10 PF approximation sets. Further detailed parameters and hyperparameters used in these experiments are given in Appendix~\ref{app:sc_setting}.

\subsubsection{Meta-training settings}
We implement MAML~\cite{finnModelAgnosticMetaLearningFast2017} with PPO~\cite{schulmanProximalPolicyOptimization2017} in Ray RLlib 2.3.1~\cite{liangRLlibAbstractionsDistributed2017}, extending it to handle both Meta-MORL and MERLION. For Meta-MORL, meta-training runs for $10^6$ environment steps on randomly generated SC tasks. We use four inner steps for simple tasks and eight for moderate or complex ones, trading off rapid adaptation against task-specific learning. A meta-environment wrapper supports dynamic task sampling to better mirror realistic task variation.

MERLION uses a comparable total step budget to Meta-MORL, distributed across the $P$ subproblems rather than multiplied by $P$. Preference weights $\mathbf{w}$ are drawn uniformly from the 3D simplex via Dirichlet$(1,1,1)$, and linear scalarisation is adopted for its simplicity and smooth optimisation landscape when training many subproblems in high-dimensional combinatorial domains. Training stability is enhanced through PPO clipping, Kullback--Leibler regularisation, and generalised advantage estimation. Although maintaining $P$ meta-policies increases meta-training cost, the subproblems are independent and can in principle be trained in parallel. Moreover, this is a one-time offline investment: at deployment, MERLION requires only 5,000--25,000 fine-tuning steps per subproblem ($\sim$3.3\% of conventional training), amortising the upfront overhead across repeated deployments on new SC configurations.

\subsubsection{Fine-tuning settings}
Fine-tuning uses the Messiah SC simulator~\cite{rachmanReinforcementLearningMultiobjective2026} for 5,000, 15,000, and 25,000 steps on simple, moderate, and complex tasks, respectively, which is substantially fewer than training-from-scratch RL baselines (typically $\sim$500,000 steps in our SC setup). Similar fine-tuning steps are applied in the MERLION setup for each individual within the population. The number of shots matches meta-training. The observation and action spaces are normalised. We fine-tune using PPO and apply non-dominated sorting~\cite{rachmanMultiobjectiveSequentialDecision2025}. To reduce variance, we employ generalised advantage estimation with advantage normalisation and adopt a standard multilayer perceptron policy. We set local perturbation sizes $M=[5,2,2]$ for simple, moderate, and complex problems and a fixed perturbation strength $\delta_{pert}=0.05$ across all settings.

\subsubsection{MORL-based method settings}
We adopt MORL/D~\cite{feltenMultiObjectiveReinforcementLearning2023} as the representative conventional MORL approach because its underlying mechanism closely matches that of our proposed method, namely a population-based framework with problem decomposition. The MORL/D implementation is equipped with a reward normalisation wrapper so that all components of the reward vector lie on a comparable scale. It is trained under the scalarised expected returns objective~\cite{hayesPracticalGuideMultiobjective2022}, where expected values are first estimated from multiple runs and subsequently used to compute the utility. In this study, MORL/D is instantiated with a multi-policy soft actor--critic~\cite{chenCombiningGradientbasedMethod2020}, which proceeds in two phases: it first employs MORL/D to learn a collection of policies and then applies an evolutionary strategy to further explore the solution space using these previously learned policies as a starting point.



\subsubsection{Population-based metaheuristic method settings}

Comparing RL with optimisation-based methods requires careful consideration because RL learns policies through sequential environment interactions, whereas population-based metaheuristics search directly in the decision space by improving a set of candidate solutions. Nevertheless, this comparison is useful for positioning RL methods against established multi-objective SC optimisation approaches.

We use the NSGA-II implementation in Pymoo~\cite{blankPymooMultiObjectiveOptimization2020}, and MOEA/D and MOPSO from Pygmo~\cite{Biscani2020}. Parameter settings are selected to promote convergence and diversity within our computational budget. For NSGA-II, mutation and crossover distribution indices control the exploration--exploitation trade-off, with smaller values encouraging broader search and larger values favouring local refinement~\cite{debFastElitistMultiobjective2002}. MOEA/D guides the search through scalarised subproblems, whereas MOPSO updates particles using individual and global best positions.

Empirical trials showed that population sizes greater than 300 repeatedly caused system-level memory failures. We therefore set the population size to 300 for all population-based metaheuristics, which is the largest feasible value under our computational constraints.

\subsubsection{Standard MORL benchmarks}
To extend MERLION performance outside of SC (see Section~\ref{sec:domain_agnosticism}), we adopt several additional traditional RL methods as baselines. To ensure a fair and rigorous comparison of sample efficiency, we carefully control the environmental interaction budgets across all algorithms. The conventional MORL baselines were trained from scratch for a total of $3 \times 10^6$ timesteps per environment to allow sufficient convergence. In contrast, the meta-learning approaches (Meta-MORL and the proposed MERLION) undergo a meta-training phase consisting of $3 \times 10^6$ cumulative timesteps distributed across the diverse task distribution. However, during the crucial adaptation phase on unseen target environments, both Meta-MORL and MERLION were restricted to a fine-tuning budget of only $100,000$ timesteps per scalarisation weight. This experimental design explicitly highlights the core advantage of our approach: amortising the high $3 \times 10^6$ computational cost during offline meta-training, thereby enabling rapid, sample-efficient policy adaptation (requiring only $\sim 3.3\%$ of the conventional training budget) when deployed in new scenarios.

\subsection{Experiment Results} \label{sec:experiment_results}
In this section, we compare MERLION against several baseline algorithms, including Meta-MORL~\cite{chenMetaLearningMultiobjectiveReinforcement2019}, which extends MAML-style~\cite{finnModelAgnosticMetaLearningFast2017} Meta-RL via a scalarisation approach; MORL/D~\cite{feltenMultiObjectiveReinforcementLearning2023}, which exemplifies MORL methods based on decomposition; and NSGA-II~\cite{debFastElitistMultiobjective2002}, a widely used evolutionary algorithm for combinatorial optimisation~\cite{wangEfficiencySortingMultiobjective2020}. For additional comparison, we also include MOEA/D~\cite{qingfuzhangMOEAMultiobjectiveEvolutionary2007} and MOPSO~\cite{coelloHandlingMultipleObjectives2004} in our experiments.

We assess significance using a non-parametric pipeline: the Kruskal--Wallis test identifies overall differences, followed by Dunn's test with Bonferroni correction for pairwise comparisons against MERLION. This controls the family-wise error rate more conservatively than uncorrected tests. In all result tables, markers denote a significant difference from MERLION ($p < 0.05$).

\subsubsection{Convergence analysis}
Figure~\ref{fig:convergence} illustrates the hypervolume achieved by MERLION and the baseline algorithms over the course of fine-tuning. For simple and moderate SC settings, our method attains the largest hypervolume within 5,000 and 15,000 fine-tuning steps, respectively, whereas MORL/D requires 500,000 steps to reach its performance. Furthermore, for the complex SC case with 25,000 steps, the hypervolume obtained by our approach remains on par with MORL/D, with no substantial difference. Across all levels of problem complexity, our algorithm consistently outperforms Meta-MORL and the metaheuristic approaches in terms of hypervolume (see Table~\ref{tab:overall_performance}).

\begin{figure*}
\centering
\subfloat[\footnotesize Simple SC]{\includegraphics[width=0.33\linewidth]{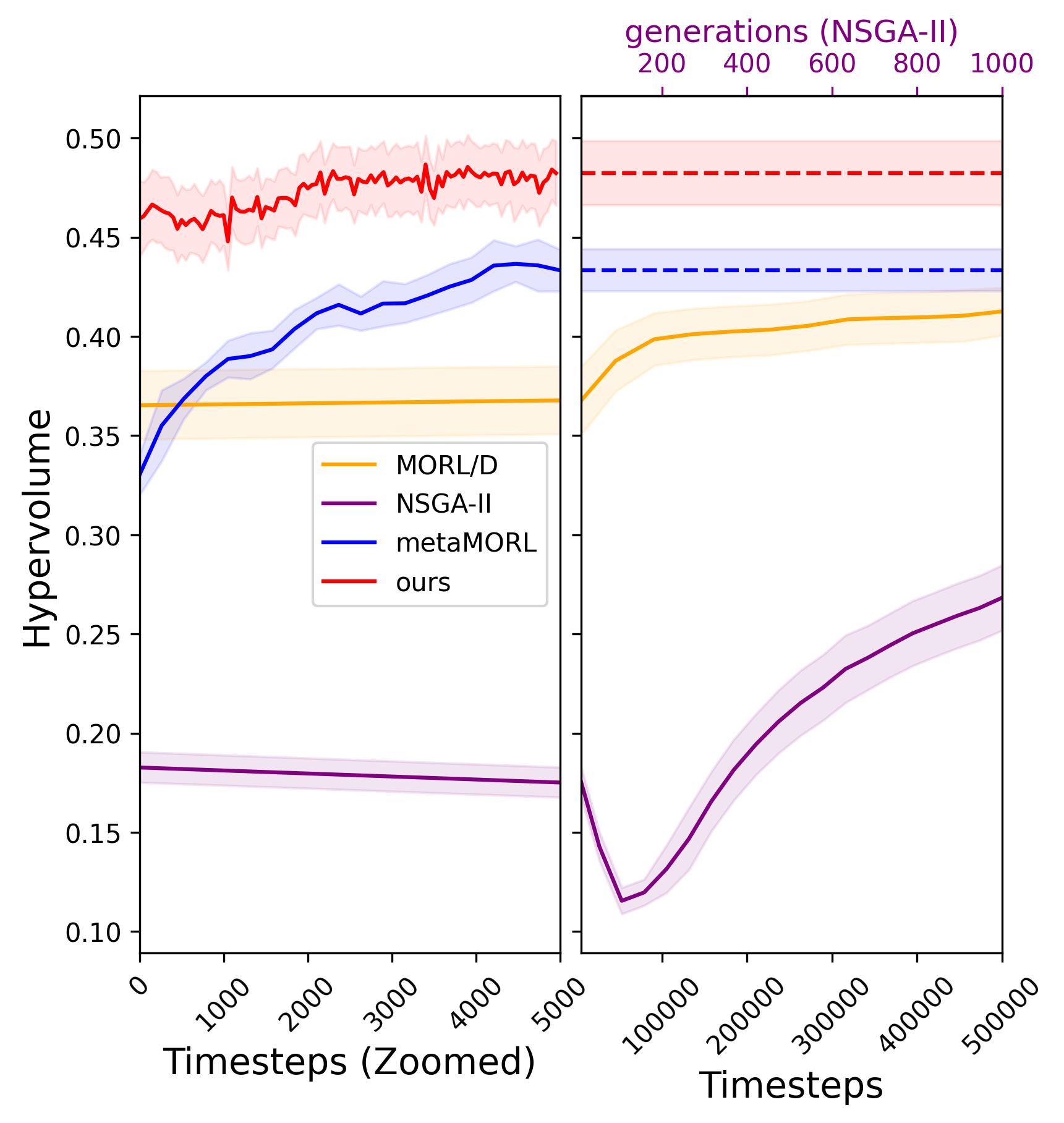}}\hfil
\subfloat[\footnotesize Moderate SC]{\includegraphics[width=0.33\linewidth]{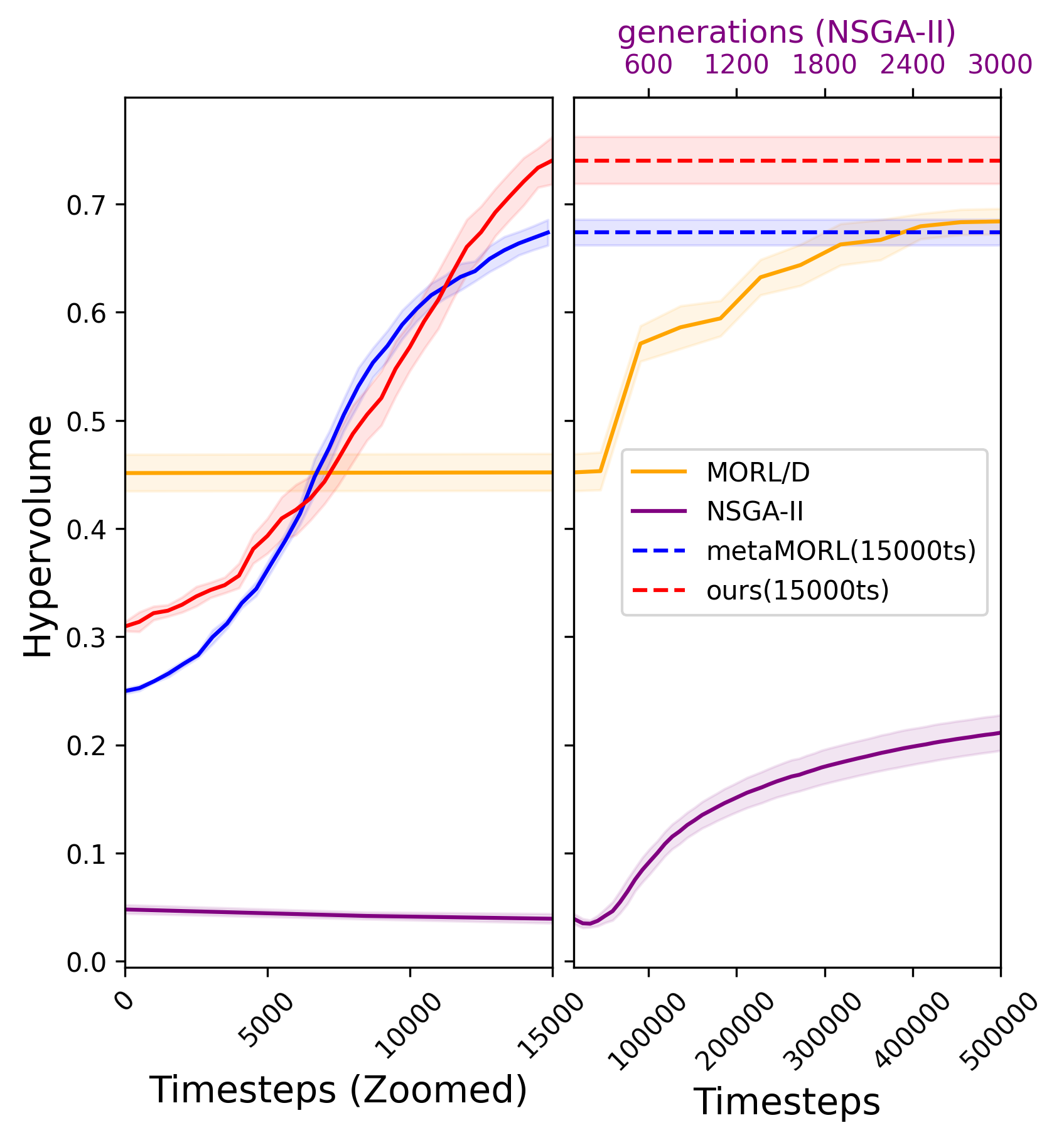}}\hfil
\subfloat[\footnotesize Complex SC]{\includegraphics[width=0.33\linewidth]{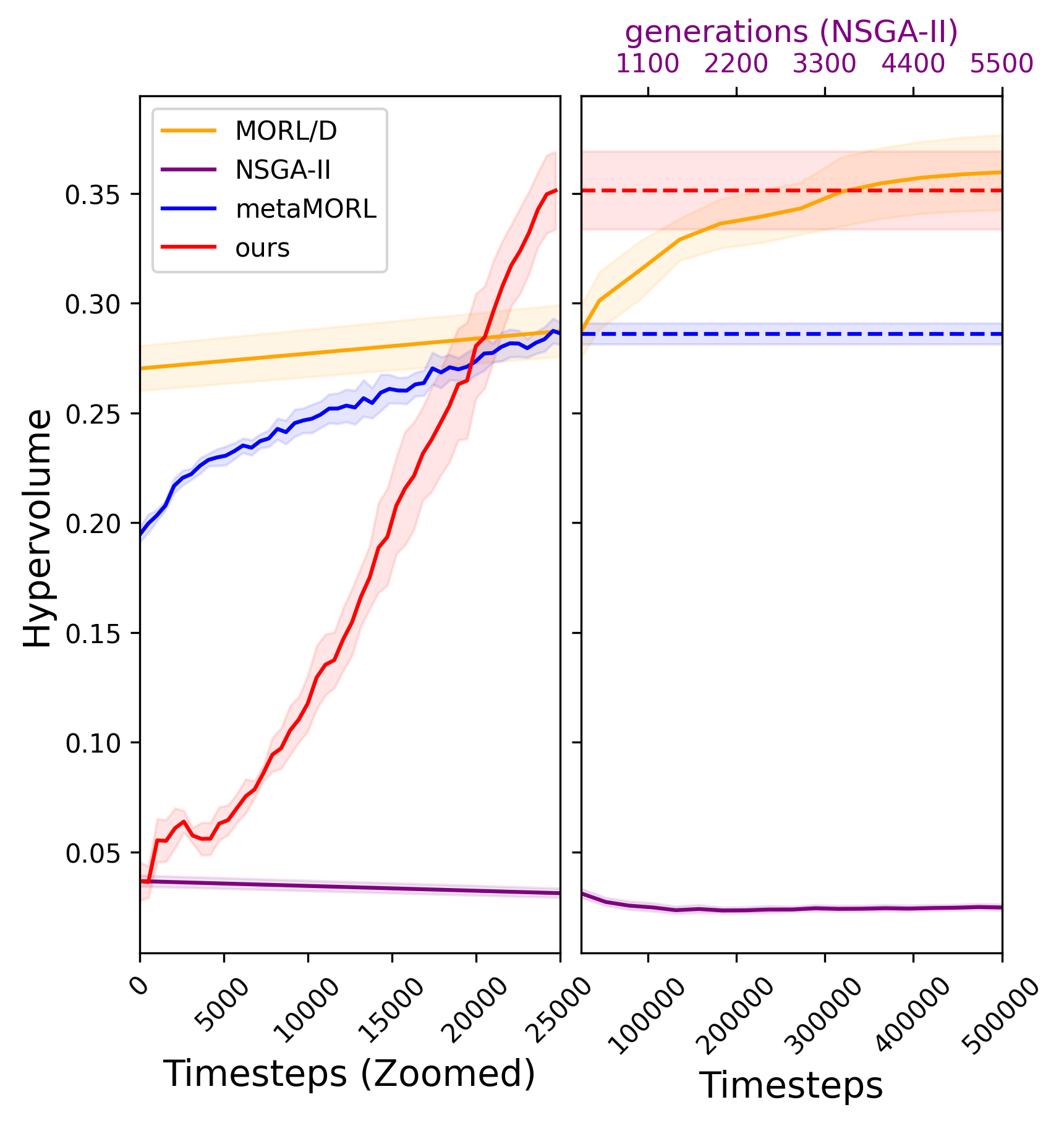}} \\
\caption{Hypervolume convergence plots across SC environments. Our proposed algorithm consistently shows higher hypervolume than Meta-MORL and NSGA-II at convergence in all problem complexities, while showing higher hypervolume than MORL/D in simple and moderate SC.}
\label{fig:convergence}
\end{figure*}

\begin{table*}[!ht]
\centering
\caption{Hypervolume (Hv), sparsity (Spar), Expected Utility Metric (EUM), and AHD summary across problem complexities. MERLION consistently presents the best AHD values in all complexities, and the best hypervolume in simple and moderate SC.}
\label{tab:overall_performance}
    \begin{tabular}{lllllllllll}
    \toprule
    \textbf{Problem} & \textbf{Algorithm} & \textbf{Hv Mean $\uparrow$} & \textbf{Hv Std} & \textbf{Spar Mean $\approx$} & \textbf{Spar Std} & \textbf{EUM Mean $\uparrow$} & \textbf{EUM Std} & \textbf{AHD Mean $\downarrow$} & \textbf{AHD Std} \\
    \midrule
    \multirow{5}{*}{Simple}
      & MOEA/D   & 0.2740 $\blacktriangle$ & 0.0115 & 0.0008 $\blacktriangle$ & 0.0000 & - & - & 0.4673 $\blacktriangledown$ & 0.0174\\
      & MOPSO\textsuperscript{*} & 0.2071 $\blacktriangle$ & 0.0228 & 0.0000 $\blacktriangle$ & 0.0000 & - & - & 0.4501 $\blacktriangledown$ & 0.0190 \\
      & NSGA-II   & 0.2682 $\blacktriangle$ & 0.0376 & 0.0022 $\circ$ & 0.0010 & - & - & 0.5976 $\blacktriangledown$ & 0.0145 \\
      & MORL/D   & 0.4126 $\circ$ & 0.0185 & 0.0245 $\circ$ & 0.0142 & - & - & 0.5813 $\blacktriangledown$ & 0.0177 \\
      & Meta-MORL & 0.4333 $\circ$ & 0.0170 & 0.0056 $\circ$ & 0.0011 & 0.7075 $\blacktriangle$ & 0.0070 & 0.2237 $\circ$ & 0.0418 \\
    \cmidrule(lr){2-10}
      & MERLION (\textbf{ours}) & \textbf{0.5356} & 0.0179 & 0.0607 & 0.0504 & \textbf{0.7229} & 0.0808 & \textbf{0.1008} & 0.0043 \\
    \midrule
    \multirow{5}{*}{Moderate}
      & MOEA/D   & 0.3046 $\blacktriangle$ & 0.0172 & 0.0012 $\blacktriangle$ & 0.0000 & - & - & 0.3648 $\blacktriangledown$ & 0.0209 \\
      & MOPSO\textsuperscript{*} & 0.0592 $\blacktriangle$ & 0.0067 & 0.0000 $\blacktriangle$ & 0.0000 & - & - & 0.7033 $\blacktriangledown$ & 0.0172 \\
      & NSGA-II   & 0.2115 $\blacktriangle$ & 0.0374 & 0.0000 $\blacktriangle$ & 0.0000 & - & - & 0.6500 $\blacktriangledown$ & 0.0372\\
      & MORL/D   & 0.6849 $\circ$ & 0.0193 & 0.0082 $\circ$ & 0.0047 & - & - & 0.8527 $\blacktriangledown$ & 0.0157 \\
      & Meta-MORL & 0.6737 $\circ$ & 0.0190 & 0.0015 $\blacktriangle$ & 0.0006 & \textbf{0.8308} $\blacktriangledown$ & 0.0073 & 0.3092 $\circ$ & 0.0378 \\
    \cmidrule(lr){2-10}
      & MERLION (\textbf{ours}) & \textbf{0.7402} & 0.0353 & 0.0178 & 0.0075 & 0.7778 & 0.0309 & \textbf{0.0779} & 0.0115 \\
    \midrule
    \multirow{5}{*}{Complex}
      & MOEA/D   & 0.0516 $\blacktriangle$ & 0.0202 & 0.0025 $\circ$ & 0.0004 & - & - & 0.7243 $\blacktriangledown$ & 0.0250 \\
      & MOPSO\textsuperscript{*} & 0.0000 $\blacktriangle$ & 0.0000 & 0.0000 $\blacktriangle$ & 0.0000 & - & - & 0.7991 $\blacktriangledown$ & 0.0192 \\
      & NSGA-II   & 0.0384 $\blacktriangle$ & 0.0096 & 0.0000 $\blacktriangle$ & 0.0000 & - & - & 0.4623 $\circ$ & 0.0147\\
      & MORL/D   & \textbf{0.3619} $\circ$ & 0.0274 & 0.0322 $\circ$ & 0.0121 & - & - & 0.5879 $\blacktriangledown$ & 0.0614 \\
      & Meta-MORL & 0.2658 $\circ$ & 0.0099 & 0.0022 $\circ$ & 0.0006 & \textbf{0.6127} $\blacktriangledown$ & 0.0030 & 0.2894 $\circ$ & 0.0541\\
    \cmidrule(lr){2-10}
      & MERLION (\textbf{ours}) & 0.3513 & 0.0334 & 0.0100 & 0.0065 & 0.6091 & 0.0552 & \textbf{0.1779} & 0.0061\\
    \bottomrule
    \end{tabular}
\footnotesize{\textit{Note:} $\uparrow$ and $\downarrow$ indicate metrics to be maximised or minimised, respectively. For sparsity ($\approx$), balanced values are preferred; too high implies poor coverage, while too low indicates excessive concentration of solutions. $\blacktriangle$, $\blacktriangledown$, or $\circ$ signifies that ours is significantly higher, significantly lower ($p < 0.05$), or not significantly different from the appointed value, respectively.\\
\textsuperscript{*}Constraint violations occur in all MOPSO runs, but only the complex case results in zero feasible solutions.
}
\end{table*}

\begin{figure*}[!ht]
\centering
\subfloat[\footnotesize Simple SC]{\includegraphics[width=0.33\linewidth]{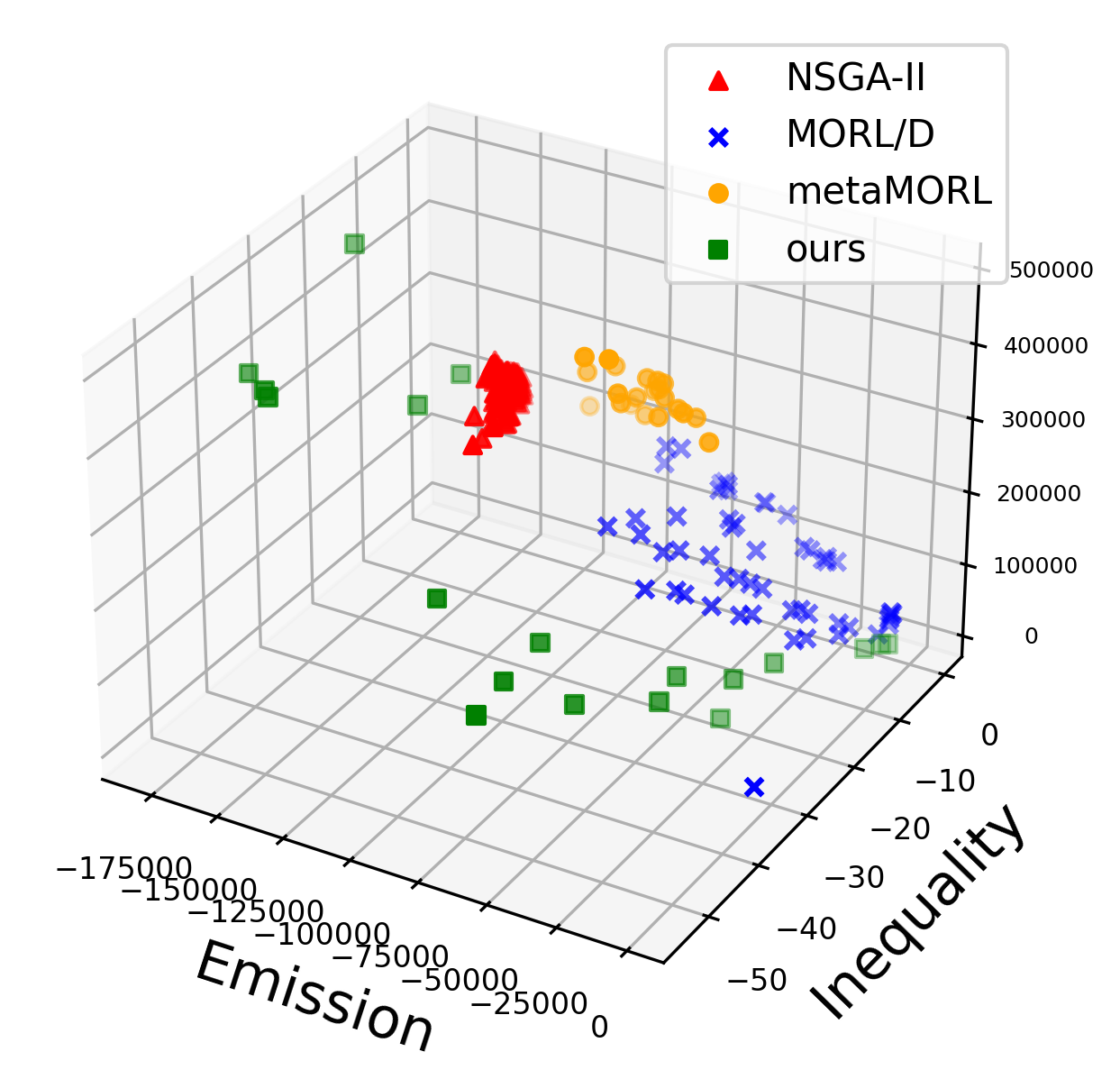}}\hfil
\subfloat[\footnotesize Moderate SC]{\includegraphics[width=0.33\linewidth]{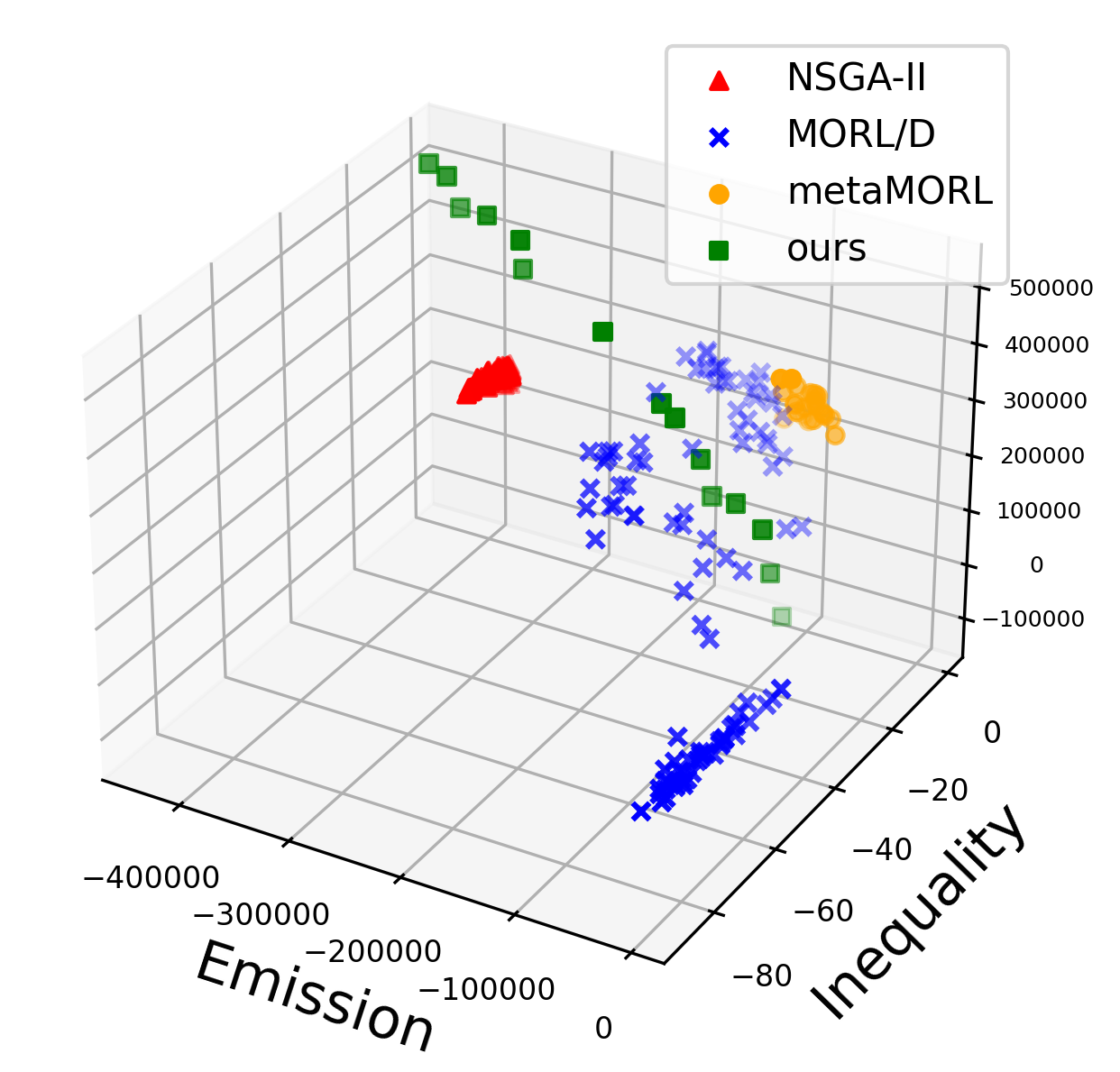}}\hfil
\subfloat[\footnotesize Complex SC]{\includegraphics[width=0.33\linewidth]{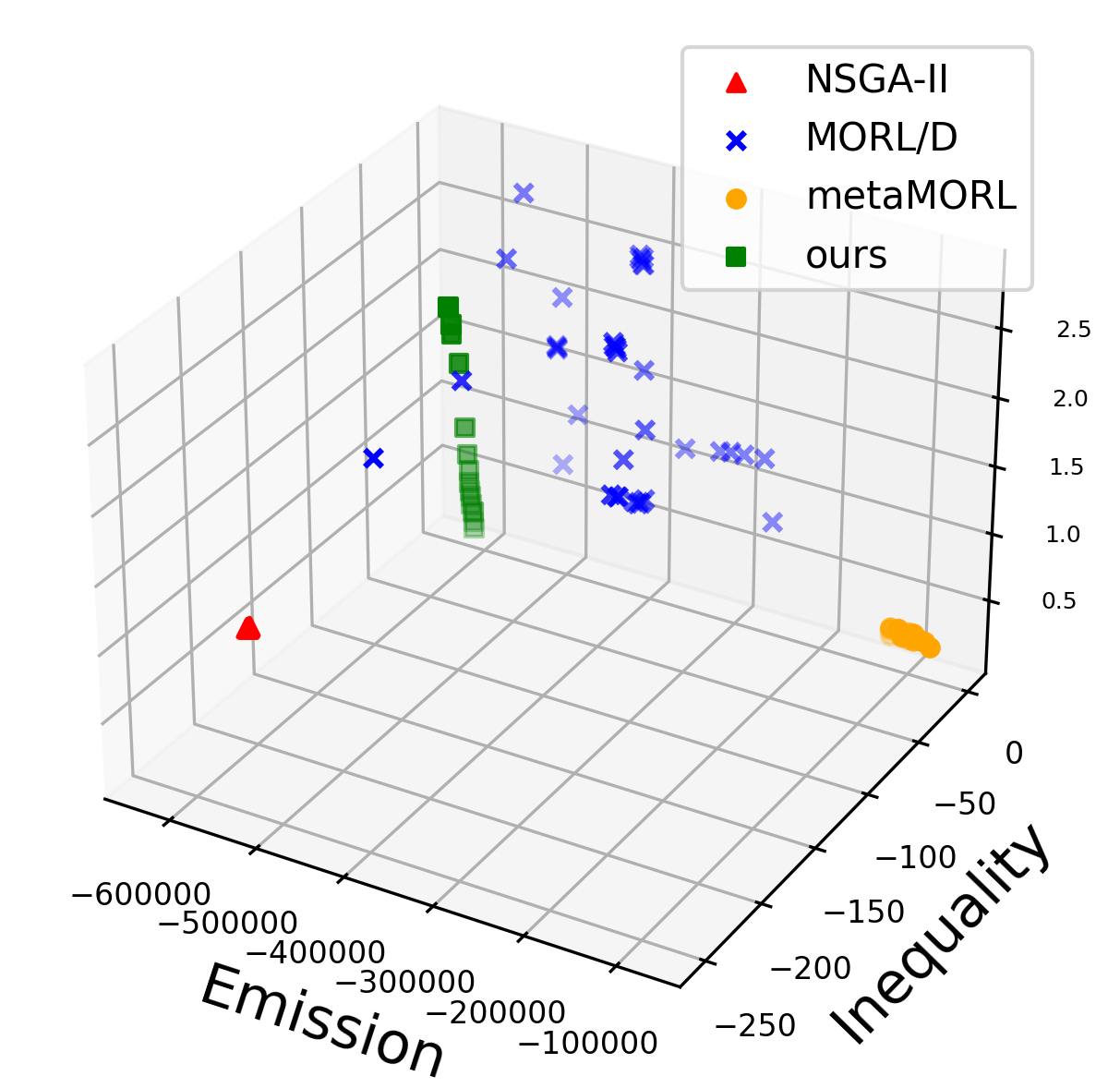}}\\
\caption{PF plots for the SC problems. Overall, the solution set generated by our proposed algorithm exhibits the highest diversity in simple and moderate SC instances,while comparable with MORL/D in complex SC.}
\label{fig:pf}
\end{figure*}



\subsubsection{Overall performance}
Table~\ref{tab:overall_performance} summarises the overall performance across the three SC complexities. MERLION achieves the best hypervolume in the simple and moderate SC cases, while being competitive against MORL/D in complex SC, and the best AHD in all cases. This indicates that MERLION generally provides strong PF approximation quality, balancing dominated objective-space coverage with closeness to the estimated reference PF. The difference between hypervolume and AHD also explains why some methods obtain relatively high hypervolume but poor AHD such as shown in MORL/D. Hypervolume rewards the volume dominated by the obtained solutions with respect to a reference point, so a method can score highly if it finds a few extreme or widely spread solutions that dominate a large region. In contrast, AHD measures the distance between the obtained approximation set and the estimated reference PF, so it penalises solutions that are far from the true trade-off surface even when their dominated volume is large.

Sparsity further characterises the distribution of solutions. MERLION produces relatively less concentrated fronts in the simple and moderate complex cases, while in the complex case, its sparsity becomes comparable to that of other methods, reflecting the increased difficulty of maintaining a broad spread under tighter operational constraints. EUM is reported only for preference-conditioned meta-learning methods. MERLION achieves the best EUM in simple SC, but Meta-MORL is higher in moderate and complex SC, suggesting that MERLION prioritises PF approximation quality and diversity rather than explicitly optimising uniform preference coverage.

For MOPSO, constraint violations occur across all SC settings, but their effect differs by complexity. In the simple and moderate cases, MOPSO can still identify feasible candidates, whereas in the complex case, no feasible solutions are found, leading to zero hypervolume and indicating that the swarm search fails to locate operationally valid solutions under tightly coupled SC constraints.

\subsubsection{PF approximation}
We compare the solution sets of the PF approximations produced by representative algorithms using the run that achieves the highest hypervolume. Figure~\ref{fig:pf} illustrates the solution spread obtained by each algorithm, consistent with Table~\ref{tab:overall_performance}, where our proposed algorithm shows the highest hypervolume in simple and moderate complex SC settings, indicating more diverse and closer-to-optimal solutions. Moreover, the AHD values of our proposed algorithm consistently show the least, indicating that the resulting PF approximation sets are the closest to the true PF values.

We highlight that the PF approximation set produced by MERLION exhibits diversity comparable to that of MORL/D, a traditional MORL method, while requiring substantially fewer time steps (5,000–25,000 for MERLION versus 500,000 for MORL/D). Furthermore, the integration of evolutionary mechanisms is proven to be effective in significantly spreading out solution sets. This is signified when comparing the sets with Meta-MORL, which uses similar fine-tuning time steps. Meanwhile, the traditional evolutionary method alone (i.e., represented by NSGA-II) is rather ineffective in exploring the objective space, especially when the solution dimension grows significantly, as in complex SC.

Our experiment signifies that the evolutionary mechanism applied in the population of meta-policies can address the typical problem arising from few-shot and limited fine-tuning time steps when applying meta-learning-based methods, while still benefiting from amortised meta-training cost since we do not need to repeat meta-training for new, unseen tasks. Instead, we only require fine-tuning for a substantially smaller number of time steps than traditional methods.

Meta-MORL and metaheuristic methods generally yield more clustered solutions in higher-dimensional search spaces, largely due to insufficient learning when trading off exploration and exploitation. In MERLION, this issue can be alleviated by increasing the population size of meta-policies and/or extending the fine-tuning horizon, at the expense of additional computational resources.

\begin{table*}[!ht]
\centering
\caption{Hypervolume (Hv), sparsity (Spar), Expected Utility Metric (EUM) across parameters. Defaut setting follows simple SC configuration, i.e., $P$=10, $\eta_{\text{cross}}=0.6$, $M$=4, except for time steps.}
\label{tab:ablation}
    \begin{tabular}{llllllllll}
    \toprule
    \textbf{Parameters} & \textbf{Settings} & \textbf{Hv Mean $\uparrow$} & \textbf{Hv Std} & \textbf{Spar Mean $\approx$} & \textbf{Spar Std} & \textbf{EUM Mean $\uparrow$} & \textbf{EUM Std} & \textbf{AHD mean $\downarrow$} & \textbf{AHD Std} \\
    \midrule
    \textbf{All} & \textbf{default} & 0.5356 & 0.0179 & 0.0607 & 0.0504 & 0.7229 & 0.0808 & 0.1008 & 0.0043 \\
    \midrule
    \multirow{2}{*}{$P$}
      & 5 & 0.5166 $\circ$ & 0.0225 & 0.0810 $\circ$ & 0.0880 & 0.6528 $\circ$ & 0.0460 & 0.0972 $\circ$ & 0.0284 \\
      & 15 & 0.5446 $\circ$ & 0.0155 & 0.0484 $\circ$ & 0.0454 & 0.6470 $\circ$ & 0.2518 & 0.0789 $\blacktriangle$ & 0.0082 \\
    \midrule
    \multirow{2}{*}{$\eta_{\text{cross}}$}
      & 0.3 & 0.5290 $\circ$ & 0.0156 & 0.0459 $\circ$ & 0.0648 & 0.7445 $\circ$ & 0.0982 & 0.0937 $\circ$ & 0.0284 \\
      & 0.9 & 0.5426 $\circ$ & 0.0146 & 0.0256 $\circ$ & 0.0272 & 0.8348 $\circ$ & 0.0827 & 0.0911 $\circ$ & 0.0119 \\
    \midrule
    \multirow{2}{*}{$M$}
      & 0 & 0.4625 $\circ$ & 0.0122 & 0.0543 $\circ$ & 0.0547 & 0.7164 $\circ$ & 0.0801 & 0.1238 $\circ$ & 0.0189\\
      & 2 & 0.5221 $\circ$ & 0.0211 & 0.0756 $\circ$ & 0.0635 & 0.7172 $\circ$ & 0.0852 & 0.0981 $\circ$ & 0.0104 \\
    \midrule
    \multirow{2}{*}{Time steps\textsuperscript{*}}
      & 20,000 & 0.2907 $\blacktriangle$ & 0.0334 & 0.0112 $\circ$ & 0.0102 & 0.6096 $\circ$ & 0.1722 & 0.2144 $\circ$ & 0.0109\\
      & \textbf{25,000} & 0.3513 & 0.0101 & 0.0101 & 0.0066 & 0.6091 & 0.0552 & 0.1779 & 0.0061\\
    \bottomrule
    \end{tabular}

\vspace{2mm}
\footnotesize{\textit{Note:} Symbols follow Table~\ref{tab:overall_performance}.\\
\textsuperscript{*} Time-step ablation is evaluated using the complex SC setup with 25,000 being the default time steps.
}
\end{table*}

\subsubsection{Operational analysis}
We examine representative solutions with the most balanced final weights across multiple runs for each SC setting. As problem complexity increases, the learned policies shift from more reactive production and inventory dynamics toward steadier production and more concentrated stockholding at selected facilities. Demand loss remains broadly controlled, but becomes less uniformly distributed in the more complex settings. More details and operational plots are provided in the Appendix~\ref{app:operational_performance}.

\begin{table*}[!ht]
\centering
\caption{Normalised Hypervolume (HV) and Sparsity (Spar) of Merlion vs Traditional and Meta RL}
\begin{tabular}{p{1.1 cm}llllllllllp{0.3 cm}p{0.3 cm}}
\toprule
\multirow{2}{*}{Algorithm} & \multicolumn{2}{c}{mo-halfcheetah-v4} & \multicolumn{2}{c}{deep-sea-treasure-v0} & \multicolumn{2}{c}{resource-gathering-v0} & \multicolumn{2}{c}{mo-hopper-v4} & \multicolumn{2}{c}{mo-reacher-v4} & \multicolumn{1}{c}{Average} & \multirow{2}{*}{Rank} \\  
\cmidrule(lr){2-11}
 & \multicolumn{1}{c}{HV} & \multicolumn{1}{c}{Spar} & \multicolumn{1}{c}{HV} & \multicolumn{1}{c}{Spar} & \multicolumn{1}{c}{HV} & \multicolumn{1}{c}{Spar} & \multicolumn{1}{c}{HV} & \multicolumn{1}{c}{Spar} & \multicolumn{1}{c}{HV} & \multicolumn{1}{c}{Spar} & \multicolumn{1}{c}{HV} &  \\
\midrule
PCN & 0.0185 $\blacktriangle$ & 0.0000 $\blacktriangle$ & - & - & - & - & 0.0024 $\blacktriangle$ & 0.0020 $\blacktriangle$ & 0.0244 $\blacktriangle$ & 0.0168 $\blacktriangle$ & 0.0151 & 7\\
CAPQL & 0.1981 $\circ$ & 0.0001 $\blacktriangle$ & - & - & - & - & 0.0069 $\blacktriangle$ & 0.0003 $\blacktriangle$ & 0.0004 & - & 0.1025 & 5 \\
PGMORL & 0.0375 $\blacktriangle$ & 0.0002 $\blacktriangle$ & - & - & - & - & 0.0058 $\blacktriangle$ & 0.0037 $\blacktriangle$ & - & - & 0.0217 & 6\\
MPMOQL & - & - & 0.4507 $\circ$ & 0.0702 $\circ$ & 0.1421 $\blacktriangle$ & 0.3272 $\circ$ & - & - & - & - & 0.2964 & 3 \\
MORL/D & 0.2240 $\circ$ & 0.0004 $\circ$ & 0.3033 $\blacktriangle$ & 0.0086 $\circ$ & 0.0984 $\blacktriangle$ & 0.0233 $\blacktriangle$ & 0.0138 $\circ$ & 0.0010 $\blacktriangle$ & 0.0685 $\circ$ & 0.0391 $\blacktriangle$ & 0.1416 & 4\\
Meta-MORL & \textbf{0.7432} $\circ$ & 0.0201 $\circ$ & 0.2900 $\blacktriangle$ & 0.0000 $\blacktriangle$ & 0.4000 $\blacktriangle$ &  0.3275 $\blacktriangle$ & 0.0520 $\circ$ & 0.0073 $\circ$ & \textbf{0.1513} $\circ$ & 0.1883 $\circ$ & 0.3273 & 2\\
\midrule
\textbf{MERLION (ours)} & 0.4181 & 0.0405 & \textbf{0.5416} & 0.0185 & \textbf{1.0000} & 0.0000 & \textbf{0.5888} & 0.1158 & 0.1295 & 0.1713 & \textbf{0.5356} & \textbf{1}\\
\bottomrule
\end{tabular}

\label{tab:res-overall}
\vspace{2mm}
\footnotesize{\textit{Note:} Symbols follow Table~\ref{tab:overall_performance}.}
\end{table*}

\subsection{Ablation Study} \label{sec:ablation_study}
We run an ablation study on the \textit{simple SC} benchmark to assess MERLION's sensitivity to three hyperparameters: population size $P$, crossover probability $\eta_{\text{cross}}$, and local perturbation size $M$ (defaults: $P{=}10$, $\eta_{\text{cross}}{=}0.6$, $M{=}4$). Table~\ref{tab:ablation} summarises the resulting changes in optimality and diversity.

Varying $P$ has the clearest effect. A smaller population ($P{=}5$) lowers hypervolume and EUM, and yields a sparser approximation set, indicating weaker coverage. Increasing to $P{=}15$ gives the strongest PF approximation, with the highest hypervolume, much lower AHD (closer to the true PF), and denser solutions (lower sparsity). However, EUM drops for both $P{=}5$ and $P{=}15$ relative to the default, implying that larger $P$ favours broader PF coverage over maximising expected utility under the sampled preferences. The only clearly significant change across settings is the AHD reduction at $P{=}15$.

Varying $\eta_{\text{cross}}$ shows consistent but modest effects. Higher crossover probability generally increases EUM, slightly improves hypervolume and AHD, and yields denser solution sets, reflecting more effective recombination of behavioural traits across meta-policies. Lower $\eta_{\text{cross}}$ produces weaker gains, with slightly reduced hypervolume and higher AHD.

Local perturbations are crucial for exploration. Disabling them ($M{=}0$) yields the lowest hypervolume and worst AHD, indicating premature convergence and poor Pareto coverage. A moderate rate ($M{=}2$) partly restores performance and produces a sparser set than $M{=}0$, but still underperforms the default, suggesting stronger perturbations better preserve diversity. Except for the significant AHD gain at $P{=}15$, most configuration differences are small (often within one standard deviation and not significant), showing MERLION is fairly robust to these hyperparameters on the simple SC benchmark.

Moreover, we investigate how adding more time steps during fine-tuning affects performance when using complex SC. The results show that including an additional 5,000 time steps can increase the hypervolume by up to approximately 20\%, while having little impact on the other metrics.

\subsection{Cross-Domain Experiments} \label{sec:domain_agnosticism}

Table~\ref{tab:res-overall} summarises MERLION's performance beyond the supply-chain domain. Across the toy MORL benchmarks, MERLION achieves the highest average normalised hypervolume and the best overall rank, with statistically significant gains on deep-sea-treasure-v0, resource-gathering-v0, and mo-hopper-v4. Although Meta-MORL attains the highest mean hypervolume on mo-halfcheetah-v4, the difference is not significant due to its higher variance, and the two meta-learning methods are also statistically comparable on mo-reacher-v4. In terms of sparsity, MERLION generally maintains a moderate and well-dispersed approximation set, suggesting that its strong hypervolume does not arise from collapsing solutions into a narrow region of the PF. Overall, these results suggest that the benefits of MERLION's are not limited to SC and can extend to a wider range of MORL domains.

\section{Conclusions and Future Works} \label{sec:conclusions}

MERLION is a population-based Meta-MORL framework that combines gradient-based meta-learning with evolutionary search in the scalarisation weight space. Across SC problems of varying complexity, it consistently achieves the lowest AHD and the highest hypervolume in Simple and Moderate SC. In Complex SC, MERLION remains competitive with MORL/D under a 25$\times$ smaller fine-tuning budget, while still outperforming Meta-MORL under a comparable meta-learning adaptation budget. MERLION improves hypervolume by up to 32\% over Meta-MORL, indicating that maintaining multiple meta-policies alleviates the limited exploration of few-shot adaptation from a single shared meta-policy.

The PF approximations further show well-distributed solution sets, particularly in Simple and Moderate SC, in line with the hypervolume and sparsity results. Overall, MERLION is comparable to MORL/D and outperforms Meta-MORL and metaheuristic baselines. Operational analysis showed more stable manufacturing and inventory trajectories as SC complexity increased, although stockpiling became more concentrated at several facilities. Nevertheless, demand loss remains low across markets.

MERLION also performed well on popular RL problems, attaining the highest average normalised hypervolume and the best overall ranking against Meta-MORL and several conventional MORL baselines, despite using substantially fewer adaptation timesteps. This highlights its ability to convert offline meta-training into rapid and sample-efficient adaptation.

Future work includes improving coverage in more complex decision spaces, incorporating preference-aware evolutionary objectives to enhance EUM, and extending MERLION to larger and more realistic SC settings with richer operational constraints.


 
\bibliography{references_nourl}
\bibliographystyle{IEEEtran}


 




\clearpage
\appendices
\input{_appendix}

\end{document}

%% file: _appendix.tex
\section{Problem Definition: MOMDP-Based Formulation} \label{sec:sc_problem}
Table~\ref{tab:parameter_definition} introduces the specific notation employed in the SC modelling, whereas Table~\ref{tab:sets} lists the nodes contained in each set as well as the sets that belong to each set family.
\begin{table}[ht!]
    \centering
    \caption{\footnotesize Sets, parameters, and variables in the supply chain model. Nodes $u$, $v$, and $w$ ($u<v<w$) denote facilities, while $t$, $\varsigma$, $\psi$, and $\tau$ denote the time period, inventory, production, and transportation, respectively.}
    \label{tab:parameter_definition}
    \begin{threeparttable}
    \begin{tabularx}{\linewidth}{p{1 cm}X}
        \toprule
        \textbf{Notations} & \textbf{Definition}\\
        \midrule
        \multicolumn{2}{@{}l}{\textbf{Sets}} \\
        $N^s$ & suppliers \\
        $N^m$ & manufacturers \\
        $N^{w\omega}$ & mid-level warehouses that deliver to lower-level warehouses\\
        $N^d$ & distribution centres that deliver to retailers\tnote{*} \\
        $N^r$ & retailers \\
        $N^z$ & markets \\
        \midrule
        \multicolumn{2}{@{}l}{\textbf{Family of sets}} \\
        $N^{From}$ & origin nodes ($|N^{From}| = n_1$) \\
        $N^{To}$ & destination nodes ($|N^{To}| = n_2$)\\
        $N^{Inv}$ & inventory-holding non-production nodes ($|N^{Inv}| = n_3$) \\
        $N^{Invf}$ & upstream inventory-shipping nodes ($|N^{Invf}| = n_4$) \\
        $N^{Invt}$ & downstream inventory-receiving nodes ($|N^{Invt}| = n_5$) \\
        \midrule
        \multicolumn{2}{@{}l}{\textbf{Parameters}} \\
        $T$ & Total period horizon \\
        $\text{\it Profit}$ & Total profit of all periods \\
        $E$ & Total emission of all periods \\
        $F$ & Total SL inequality measure of all periods \\ 
        $SL_y$ & SL at node $y$ of all periods \\
        $Rev_t$ & Revenue at period $t$ \\
        $PC_t$ & Production cost at period $t$ \\
        $TC_t$ & Transport cost at period $t$ \\
        $IC_t$ & Inventory cost at period $t$ \\
        $L$ & Transport lead time \\
        $I_{tv}$ & Inventory unit at period $t$, node $v$ \\
        $c_v^\varsigma$ & Inventory cost per unit at node $v$ \\
        $e^{\varsigma}_v$ & Emission per unit resulted from inventory at node $v$ \\
        $c_v^\psi$ & Manufacturing cost per unit at node $v$ \\
        $e^{\psi}_w$ & Emission per unit resulted from the manufacturing process at node $w$ \\
        $c_{uv}^\tau$ & Transport cost per unit per day from node $u$ to $v$ \\
        $e^{\tau}_{uv}$ & Emission per unit per day resulted from product transport from node $u$ to $v$ \\
        $\mathbf{Q_{tv}}$ & The outstanding order at node $v$ \\
        $CE_t$ & The accumulated emission at period $t$ \\
        $AF_t$ & The average service level inequality in period $t$ \\
        $d_{tv}$ & Demand at period $t$, retailer $v$ \\
        $Cap$ & Transport capacity \\
        $price_v$ & Product selling price at retailer $v$ \\
        $\mathcal{M}$ & Very large number (big-M)\\
        \midrule
        \textbf{Variables}\\
        $q^\tau_{tuv}$ & Transport quantity at period $t$, from node $u$ to $v$ \\
        $q^\psi_{tw}$ & Manufacturing quantity at period $t$, at manufacturer $w$ \\
        \bottomrule     
    \end{tabularx}

    \begin{tablenotes}
        \footnotesize
        \item[*] Optional nodes in the SC. Mid-level warehouses, which may span multiple echelons, exist only if distribution centres (i.e., the lowest-level warehouses) are present. Without warehouses, manufacturers act as distribution centres. 
    \end{tablenotes}
    \end{threeparttable}
\end{table}

\begin{table}[H]
\centering
\caption{ \footnotesize Sets and family of sets used in the formulation to distinguish between SC network problems. The following table lists the nodes covered in the sets and the sets covered in the family of sets in the three scenarios.}
\label{tab:sets}
    \begin{tabularx}{9cm}{llll}
        \toprule
        Notation & Simple SC & Moderate SC & Complex SC\\
        \midrule
        \textbf{Sets}\\
        $N^s$ & 1 & 1, 2&   1, 2, 3\\
        $N^m$ & 2, 3 & 3, 4, 5 & 4, 5, 6, 7, 8\\
        $N^{w1}$ & NA & 6, 7 & 9, 10, 11\\
        $N^d$ & NA & NA & 12, 13, 14\\
        $N^r$ & 4, 5 & 8, 9, 10 & 15, 16, 17, 18, 19\\
        $N^z$ & 6, 7 & 11, 12, 13 & 20, 21, 22, 23, 24\\
        \midrule
        \textbf{Family of sets}\\
        $N^{From}$ & $N^s, N^m$ & $N^s, N^m, N^w$ & $N^s, N^m, N^w, N^d$\\
        $N^{To}$ & $N^m, N^r$ & $N^m, N^w, N^r$ & $N^m, N^w, N^d N^r$\\
        $N^{Inv}$ & $N^r$ & $N^w, N^r$ & $N^w, N^d, N^r$\\
        $N^{Invf}$ & $N^m$ & $N^m, N^w$ & $N^m, N^w, N^d$\\
        $N^{Invt}$ & $N^r,N^z$ & $N^w,N^r,N^z$ & $N^w,N^d,N^r,N^z$\\
        \bottomrule
    \end{tabularx}
\end{table}

For the completeness of this study, we restate the SC formulations that we adopt. Multi-echelon SC problems that involve several facilities and markets, along with non-stationary demand, are modelled using the MOMDP framework.

\subsection{State Transition Functions}
These expressions describe the evolution of the system from its state in period $(t-1)$ to its subsequent state in period $t$. The corresponding state transition relations are specified in Equations~\eqref{eq:rl_mfg_qty} through~\eqref{eq:rl_avg_ineq}. In particular, Equation~\eqref{eq:rl_mfg_qty} imposes that the entire quantity of raw materials supplied is fully converted into finished products, as detailed below:

\begin{equation}
    \label{eq:rl_mfg_qty}
    q_{tv}^{\psi} = \sum_{u \in N^s} q_{(t-L)uv}^{\tau} \quad \forall t, \forall v \in N^m.
\end{equation}

\noindent Equation~\eqref{eq:rl_inv_all} calculates the inventory levels considering the remainder of previous periods, additional goods (i.e., inflow), and goods shipped (i.e., outflow):
\begin{equation}
\begin{split}
\label{eq:rl_inv_all}
    I_{tv} &= I_{(t-1)v} + \text{inflow}_{tv} - \text{outflow}_{tv}, \\
    &\quad \forall t, \forall v \in N^{Inv} \cup N^m\\
    \text{where} \quad
    \text{inflow}_{tv} &=
    \begin{cases}
        q_{tv}^{\psi} & \text{if } v \in N^m \\
        \sum\limits_{u \in N^{Invf}} q_{(t-L)uv}^{\tau} & \text{if } v \in N^{Inv}
    \end{cases} \\
    \text{and} \quad
    \text{outflow}_{tv} &=
    \begin{cases}
        \sum_{w \in N^{Inv}_1} q_{t}^vw{\tau} & \text{if } v \in N^m\\
        \sum_{w \in N^{Invt}_{\{2,\dots,n_5\}}}q_{tvw}^\tau
&\text{if } v \in N^{Inv}.
    \end{cases}
\end{split}
\end{equation}


\noindent Equation~\eqref{eq:rl_order_vector} defines, for each node, the overall ordering quantity, represented as the aggregate of all incoming deliveries:
\begin{equation}
    \label{eq:rl_order_vector}
    \mathbf{Q}_{tv} = \big\{ q^\tau_{(t-L+1)uv}, \ldots, q^\tau_{tuv}, \quad \forall u \in N^{From},\forall v \in N^{To} \big\}.
\end{equation}

\noindent Equation~\eqref{eq:rl_cum_emission} specifies the computation of the daily cumulative emissions, given by:
\begin{equation}
    \label{eq:rl_cum_emission}
    CE_t=CE_{(t-1)}+E_t.
\end{equation}

\noindent Equation~\eqref{eq:rl_avg_ineq} computes the daily average SL inequality:
\begin{equation}
    \label{eq:rl_avg_ineq}
    AF_t=\frac{\left(AF_{t-1} \cdot(t-1)+F_t\right)}{t}.
\end{equation}

\noindent In our SC problem, we assume that state transitions are deterministic, meaning that no intrinsic randomness is involved (for instance, from lead time variability or defective product disposal), which would otherwise introduce uncertainty into the state transition, particularly in computing inventory for the next period.

\subsection{State Space (S)}
The system state captures how agents perceive the learned environment. Our SC model’s state variables include inventory levels for raw materials and finished products, outstanding orders across all routes within a given lead time, cumulative emissions, and the average service level inequality, subject to $0 \leq I_{(t-1)j}$. The state in period $t$ is defined as follows:
\begin{equation}
    \label{eq:rl_state_vector}
    \mathbf{S_t}=\{I_{(t-1)v}, \mathbf{Q_{(t-1)v}}, CE_{t-1}, AF_{t-1},\quad \forall{v} \in N^{To}\}.
\end{equation}

\subsection{Action Space (A)}
In our SC setting, the action encompasses the production quantities of all manufacturers as well as the shipment quantities across all routes. The actions in period $t$ are defined as:
\begin{equation}
\begin{split}
    \label{eq:rl_action_vector}
    \mathbf{A_t}=&\{q_{tuv}^{\tau}, q_{tw}^{\psi}, \quad 0\leq q_{tuv}^{\tau} \leq Cap, \quad 0\leq q_{tw}^{\psi},\\ & \forall{u} \in N^{From},
    j \in N^{To}, k \in N^{m}\}.
\end{split}
\end{equation}

\subsection{Rewards (R)}
This problem incorporates three objectives: maximising profit, minimising GHG emissions, and minimising SL inequality, which together form a three-dimensional reward vector. The MORL agents are designed to maximise the expected cumulative rewards. Since the second and third objectives should be penalised, they are assigned negative reward values. Accordingly, the reward vector is defined as $\mathbf{R_t}=\{\mathit{Prof_t},-\mathit{E_t},-\mathit{F_t}\}\cdot\rho_t$, where $\rho$ denotes the penalty term that prevents the inventory level from dropping below zero, as specified in Equation~\eqref{eq:penalty}:
\begin{equation}
\label{eq:penalty}
    \rho_t = \sum_{v \in N^{To}} \min(I_{tv},0) \cdot \mathcal{M}.
\end{equation}
\noindent In RL, the reward values are derived from objective functions at each period as given in Equations~\eqref{eq:rl_profit}, ~\eqref{eq:rl_emission}, and ~\eqref{eq:rl_sl_ineq} below:

\begin{equation}
\label{eq:rl_profit}
\begin{split}
\text{\it Profit}_t =\;& 
\underbrace{\sum_{u \in N^d} \sum_{v \in N^r} q_{(t-L)uv}^\tau \cdot price_v}_{\text{Revenue from delivered products}} 
\; - \;
\underbrace{\sum_{w \in N^m} \left( \frac{q_{tw}^\psi \cdot c_w^\psi}{v_w^\psi} \right)}_{\text{Manufacturing cost}}\\
 &- \underbrace{\sum_{u \in N^{From}} \sum_{v \in N^{To}} q_{tuv}^\tau \cdot c_{uv}^\tau \cdot L}_{\text{Transport cost over lead time}} 
\; - \;
\underbrace{\sum_{v \in N^{To}} I_{tv} \cdot c_v^\varsigma}_{\text{Inventory holding cost}}
\end{split}
\end{equation}

\begin{equation}
\label{eq:rl_emission}
\begin{split}
E_t =\;& 
\underbrace{\sum_{v \in N^{To}} I_{tv} \cdot e^\varsigma_v}_{\text{Inventory-related emissions}} 
\; + \;
\underbrace{\sum_{w \in N^m} q_{tw}^\psi \cdot e_w^\psi}_{\text{Manufacturing emissions}} \\
&\; + \;
\underbrace{\sum_{u \in N^{From}} \sum_{v \in N^{To}} q_{uv}^\tau \cdot e_{uv}^\tau \cdot L}_{\text{Transport emissions over lead time}}
\end{split}
\end{equation}

\begin{equation}
\begin{split}
\label{eq:rl_sl_ineq}
F_t =&\;
\frac{1}{2} \sum_{v \in N^r} \sum_{\substack{v' \in N^r \\ v' \neq v}} 
\left| SL_{tv} - SL_{tv'} \right|,\\
\text{where: }\\
SL_{tj} =& \min \left(\frac{\sum_{i \in N^d} q^\tau_{(t-L)ij}}{d_{tj}}, 1 \right),\quad \forall j \in N^r,\\
\end{split}
\end{equation}

\noindent where $d_{tv}$ denotes the corresponding demand for retailer $v$. Formulating a model within the MDP framework can be more involved than constructing an optimisation model, because each paradigm introduces its own specific components and notation. While most optimisation-based and MDP-based formulations are theoretically equivalent, they do not necessarily coincide in their mathematical expression. In particular, MDP formulations frequently express the governing relations in a sequential manner.

\clearpage
\section{Proofs}
\label{app:proof}

\subsection{Proof of Lemma 1}
\label{app:proof-lemma1}

\begin{proof}[Proof of Lemma 1: hypervolume monotonicity]
By definition, $HV(\mathcal{S};\mathbf{x^{\text{ref}}})=\lambda(\mathcal{D}(\mathcal{S}))$, where $\lambda(\cdot)$ denotes the Lebesgue measure and $\mathcal{D}(\mathcal{S})$ is the region dominated by $\mathcal{S}$ with respect to $\mathbf{x^{\text{ref}}}$. Since $\mathcal{K}\subseteq\mathcal{H}$, every point dominated by $\mathcal{K}$ is also dominated by $\mathcal{H}$. Hence,
\begin{equation}
\mathcal{D}(\mathcal{K}) \subseteq \mathcal{D}(\mathcal{H}).
\end{equation}
By monotonicity of Lebesgue measure, this implies
\begin{equation}
\lambda(\mathcal{D}(\mathcal{K})) \le \lambda(\mathcal{D}(\mathcal{H})).
\end{equation}
Therefore,
\begin{equation}
HV(\mathcal{K};\mathbf{x^{\text{ref}}}) \le HV(\mathcal{H};\mathbf{x^{\text{ref}}}).
\end{equation}
\end{proof}

\subsection{Proof of Lemma 2}
\label{app:proof-lemma2}
\begin{proof}[Proof of Lemma 2: Strict hypervolume gain]
The equality
\begin{equation}
HV(\mathcal{H}\cup\{\mathbf{x}\};\mathbf{x^{\text{ref}}}) = HV(\mathcal{H};\mathbf{x^{\text{ref}}}) + \Delta HV(\mathbf{x}\mid \mathcal{H})
\end{equation}
follows directly from the definition of hypervolume contribution. It remains to show that $\Delta HV(\mathbf{x}\mid \mathcal{H})>0$.

Since $\mathbf{x}\notin \mathrm{Dom}(\mathcal{H})$, no point in $\mathcal{H}$ dominates $\mathbf{x}$. Since $\mathbf{x}$ is also strictly better than the reference point in every objective, the box bounded by $\mathbf{x^{\text{ref}}}$ and $\mathbf{x}$ has strictly positive volume. Moreover, because $\mathbf{x}$ is not dominated by any point in $\mathcal{H}$, there exists a nonzero subregion of this box that is dominated by $\mathbf{x}$ but not by any point in $\mathcal{H}$. Hence, the exclusive hypervolume contribution of $\mathbf{x}$ is strictly positive, i.e.,
\begin{equation}
\Delta HV(\mathbf{x}\mid \mathcal{H}) > 0.
\end{equation}
Therefore, adding $\mathbf{x}$ to $\mathcal{H}$ strictly increases the hypervolume.
\end{proof}

\subsection{Proof of Proposition 1}
\label{app:proof-prop1}
\begin{proof}[Proof of Proposition 1: Population improvement probability]
    By the complement rule,
    \begin{equation}
    \Pr\!\left(\bigcup_{i=1}^{P} X_i\right) = 1-\Pr\!\left(\bigcap_{i=1}^{P} X_i^c\right).
    \end{equation}
    Under independence,
    
    \begin{equation}
    \Pr\!\left(\bigcap_{i=1}^{P} X_i^c\right) = \prod_{i=1}^{P}\Pr(X_i^c) = \prod_{i=1}^{P}(1-\Pr(X_i)).
    \end{equation}
    Since $\Pr(X_i)\ge \rho$ for all $i$, we have $1-\Pr(X_i)\le 1-\rho$, hence
    
    \begin{equation}
    \Pr\!\left(\bigcap_{i=1}^{P} X_i^c\right)\le (1-\rho)^P.
    \end{equation}
    Therefore,
    \begin{equation}
    \Pr\!\left(\bigcup_{i=1}^{P} X_i\right)\ge 1-(1-\rho)^P.
    \end{equation}
    For $P=1$, this reduces to $\Pr(X_1)\ge \rho$. For $P>1$ and $\rho\in(0,1)$, we have $(1-\rho)^P<1-\rho$, so
    \begin{equation}
    1-(1-\rho)^P>\rho.
    \end{equation}
    This completes the proof.
\end{proof}

\subsection{Proof of Proposition 2}
\label{app:proof-prop2}
\begin{proof}[Proof of Proposition 2: Gradient inference]
    The result follows directly from the identity
    \begin{equation}
    \|g_j + g_k\|^2 = \|g_j\|^2 + \|g_k\|^2 + 2 g_j^\top g_k,
    \end{equation}
    together with the assumption $g_j^\top g_k < 0$.
\end{proof}

\clearpage
\section{Simulated Environment} \label{app:sc_setting}
Our experiment employs three SC environments: simple, moderate, and complex networks. SC networks are node collections linked by edges, with parameters set at nodes and edges. Each network is 'fully connected', linking all nodes across consecutive layers. Material supply from suppliers is treated as unlimited. Market demand is unstable: markets one and two fluctuate as normal distributions ($\mu=150, \sigma=60$ and $\mu=100, \sigma=40$) and markets three, four, and five as Poisson distributions ($\lambda=200, 100, 150$). Seasonal demand shifts through sinusoidal-modulated distributions. Simple SC serves markets 1 and 2, moderate SC covers markets 1-3, and complex SC involves all markets.

For simple SC markets, a fixed mean price of $20$ is applied, while moderate and complex SC markets encounter variable mean prices: $[20, 21, 20.5]$ and $[100, 101, 105, 103, 104]$. Unfulfilled demand is regarded as demand loss. This problem is structured within the MOMDP framework with a finite time horizon $T=100$ and a lead time range of $L=\{1,3\}$ days. The transportation capacity ($Cap$) is set at 200, while the manufacturing capacity is not restricted.

To simulate uncertainties originating mainly from demand, cost, price structure, and lead time, we randomise the parameters during environment generations. For example, in addition to the fluctuation of demand along the period horizon, we also introduce randomised mean and standard deviation values within a range of $\{90\%, 110\%\}$ of set values. A similar range is also applied to the costs and prices, subject to the task arrangement scenarios applied. During fine-tuning, new unseen tasks are given in the Table~\ref{tab:para_value_simple} to~\ref{tab:trans_value_complex}.

\begin{table}[H]
    \centering
    \caption{\textbf{Simple SC.} Node parameter values that correspond to the facilities in Simple SC environment.}
    \label{tab:para_value_simple}
    \begin{tabularx}{\linewidth}{ccccccc}
        \toprule
        Node & $I_{0j}$ & $c_j^\varsigma$ & $e^\varsigma_j$ & $c_k^\psi$ & $v_k^\psi$ & $e^\psi_k$ \\
        \midrule
        3 & 380 & 0.1100 & 0.0002 & 2.0000 & 1.0000 & 5.0126\\
        4 & 350 & 0.1300 & 0.0002 & 2.2000 & 1.0000 & 4.5754\\
        5 & 400 & 0.1200 & 0.0002 & NA & NA & NA\\
        6 & 80  & 0.1500 & 0.0002 & NA & NA & NA\\
        \bottomrule      
    \end{tabularx}
\end{table}

\begin{table}[H]
    \centering
    \caption{\textbf{Simple SC.} Transport parameter values that correspond to product delivery.}
    \label{tab:trans_value_simple}
    \begin{tabularx}{6 cm}{cccc}
        \toprule
        From Node & To Node & $c^\tau_{ij}$ & $e^\tau_{ij}$ \\
        \midrule
        \multirow{2}{*}{1} & 2 & 0.2200 & 0.1258\\
        & 3 & 0.6900 & 0.3947\\
        \hline
        \multirow{2}{*}{2} & 4 & 1.0550 & 0.6035\\
        & 5 & 0.4300 & 0.2460\\
        \hline
        \multirow{2}{*}{3} & 4 & 0.4850 & 0.2774\\
        & 5 & 0.7500 & 0.4290\\
        \bottomrule
    \end{tabularx}
\end{table}

\begin{table}[H]
    \centering
    \caption{\textbf{Moderate SC.} Node parameter values that correspond to the facilities in Moderate SC environment.}
    \label{tab:para_value_moderate}
    \begin{tabularx}{8.5cm}{ccccccc}
        \toprule
        Node & $I_{0j}$ & $c_j^\varsigma$ & $e^\varsigma_j$ & $c_k^\psi$ & $v_k^\psi$ & $e^\psi_k$ \\
        \midrule
        3 & 380 & 0.1100 & 0.0002 & 2.0000 & 1.0000 & 5.0126\\
        4 & 350 & 0.1300 & 0.0002 & 2.2000 & 1.0000 & 4.5754\\
        5 & 400 & 0.1200 & 0.0002 & 2.3000 & 1.0000 & 5.4491\\
        6 & 80 & 0.1500 & 0.0002 & NA & NA & NA\\
        7 & 110 & 0.2000 & 0.0002 & NA & NA & NA\\
        8 & 100 & 0.2500 & 0.0002 & NA & NA & NA\\
        9 & 80 & 0.3000 & 0.0002 & NA & NA & NA\\
        10 & 120 & 0.2000 & 0.0002 & NA & NA & NA\\
        \bottomrule      
    \end{tabularx}
\end{table}

\begin{table}[H]
    \centering
    \caption{\textbf{Moderate SC.} Transport parameter values that correspond to product delivery.}
    \label{tab:trans_value_moderate}
    \begin{tabularx}{6cm}{cccc}
        \toprule
        From Node & To Node & $c^\tau_{ij}$ & $e^\tau_{ij}$ \\
        \midrule
        \multirow{3}{*}{1} & 3 & 0.2200 & 0.1258\\
        & 4 & 0.6900 & 0.3947\\
        & 5 & 0.5650 & 0.3232\\
        \hline
        \multirow{3}{*}{2} & 3 & 1.0550 & 0.6035\\
        & 4 & 0.6500 & 0.3718\\
        & 5 & 0.6300 & 0.3604\\
        \hline
        \multirow{2}{*}{3} & 6 & 0.0750 & 0.0429\\
        & 7 & 0.4300 & 0.2460\\
        \hline
        \multirow{2}{*}{4} & 6 & 0.6300 & 0.3604\\
        & 7 & 0.2300 & 0.1316\\
        \hline
        \multirow{2}{*}{5} & 6 & 0.4950 & 0.2831\\
        & 7 & 0.0750 & 0.0429\\
        \hline
        \multirow{3}{*}{6} & 8 & 1.0950 & 0.6263 \\
        & 9 & 0.6250 & 0.3575\\
        & 10 & 0.9500 & 0.5434\\
        \hline
        \multirow{3}{*}{7} & 8 & 1.6400 & 0.9381\\
        & 9 & 1.1600 & 0.6635\\
        & 10 & 0.5800 & 0.3318\\
        \bottomrule
    \end{tabularx}
\end{table}

\begin{table}[H]
    \centering  
    \caption{\textbf{Complex SC.} Node parameter values that correspond to the facilities in Complex SC environment.}
    \label{tab:para_value_complex}
    \begin{tabularx}{8.5cm}{ccccccc}
        \toprule
        Node & $I_{0j}$ & $c_j^\varsigma$ & $e^\varsigma_j$ & $c_k^\psi$ & $v_k^\psi$ & $e^\psi_k$ \\
        \midrule
        4 & 155 & 0.2300 & 0.0002 & 2.0000 & 1.0000 & 5.0126\\
        5 & 267 & 0.3500 & 0.0002 & 2.2000 & 1.0000 & 4.5754\\
        6 & 342 & 0.2200 & 0.0002 & 2.1000 & 1.0000 & 5.4491\\
        7 & 211 & 0.1100 & 0.0002 & 2.0000 & 1.0000 & 6.1232\\
        8 & 162 & 0.2900 & 0.0002 & 2.3000 & 1.0000 & 5.5157\\
        9 & 195 & 0.3700 & 0.0002 & NA & NA & NA\\
        10 & 333 & 0.1100 & 0.0002 & NA & NA & NA\\
        11 & 96 & 0.3600 & 0.0002 & NA & NA & NA\\
        12 & 285 & 0.3300 & 0.0002 & NA & NA & NA\\
        13 & 68 & 0.2600 & 0.0002 & NA & NA & NA\\
        14 & 379 & 0.3000 & 0.0002 & NA & NA & NA\\
        15 & 344 & 0.1700 & 0.0002 & NA & NA & NA\\
        16 & 66 & 0.2900 & 0.0002 & NA & NA & NA\\
        17 & 356 & 0.2700 & 0.0002 & NA & NA & NA\\
        18 & 382 & 0.2300 & 0.0002 & NA & NA & NA\\
        19 & 362 & 0.3700 & 0.0002 & NA & NA & NA\\
        \bottomrule      
    \end{tabularx}
\end{table}

\begin{table}[H]
    \centering
    \caption{\textbf{Complex SC.} Transport parameter values that correspond to product delivery.}
    \label{tab:trans_value_complex}
    \begin{tabularx}{6cm}{cccc}
        \toprule
        From Node & To Node & $c^\tau_{ij}$ & $e^\tau_{ij}$ \\
        \midrule
        \multirow{5}{*}{1} & 4 & 0.5350 & 0.3060\\
        & 5 & 0.2650 & 0.1516\\
        & 6 & 1.8450 & 1.0553\\
        & 7 & 1.6000 & 0.9152\\
        & 8 & 1.4400 & 0.8237\\
        \midrule
        \multirow{5}{*}{2} & 4 & 0.3600 & 0.2059\\
        & 5 & 0.2950 & 0.1687\\
        & 6 & 1.2350 & 0.7064\\
        & 7 & 0.6250 & 0.3575\\
        & 8 & 1.8550 & 1.0611\\
        \midrule
        \multirow{5}{*}{3} & 4 & 0.6000 & 0.3432\\
        & 5 & 0.1750 & 0.1001\\
        & 6 & 0.7450 & 0.4261\\
        & 7 & 1.3300 & 0.7608\\
        & 8 & 0.1700 & 0.0972\\
        \midrule
        \multirow{3}{*}{4} & 9 & 1.9900 & 1.1383\\
        & 10 & 0.3400 & 0.1945\\
        & 11 & 0.8100 & 0.4633\\
        \midrule
        \multirow{3}{*}{5} & 9 & 1.5150 & 0.8666\\
        & 10 & 0.6600 & 0.3775\\
        & 11 & 0.6450 & 0.3689\\
        \midrule
        \multirow{3}{*}{6} & 9 & 1.6950 & 0.9695\\
        & 10 & 1.5800 & 0.9038\\
        & 11 & 0.8150 & 0.4662\\
        \midrule
        \multirow{3}{*}{7} & 9 & 1.6150 & 0.9238\\
        & 10 & 1.2600 & 0.7207\\
        & 11 & 0.6750 & 0.3861\\
        \midrule
        \multirow{3}{*}{8} & 9 & 1.0300 & 0.5892\\
        & 10 & 1.0900 & 0.6235\\
        & 11 & 1.6300 & 0.9324\\
        \midrule
        \multirow{3}{*}{9} & 12 & 1.9650 & 1.1240 \\
        & 13 & 1.9250 & 1.1011\\
        & 14 & 1.6200 & 0.9266\\
        \midrule
        \multirow{3}{*}{10} & 8 & 1.4900 & 0.8523\\
        & 9 & 1.9600 & 1.1211\\
        & 10 & 0.6350 & 0.3632\\
        \bottomrule
    \end{tabularx}
\end{table}

\begin{table}[H]\ContinuedFloat
    \centering
    \caption{Table~\ref{tab:trans_value_complex} continued..}
    \begin{tabularx}{6cm}{cccc}
        \toprule
        \multirow{3}{*}{11} & 8 & 1.8700 & 1.0696\\
        & 9 & 0.2000 & 0.1144\\
        & 10 & 1.8550 & 1.0611\\
        \midrule
        \multirow{5}{*}{12} & 15 & 1.9450 & 1.1125\\
        & 16 & 0.9650 & 0.5520\\
        & 17 & 1.9050 & 1.0897\\
        & 18 & 0.9000 & 0.5148\\
        & 19 & 0.6900 & 0.3947\\
        \midrule
        \multirow{5}{*}{13} & 15 & 0.8050 & 0.4605\\
        & 16 & 1.0650 & 0.6092\\
        & 17 & 1.8400 & 1.0525\\
        & 18 & 0.8300 & 0.4748\\
        & 19 & 1.8850 & 1.0782\\
        \midrule
        \multirow{5}{*}{14} & 15 & 1.6600 & 0.9495\\
        & 16 & 1.5100 & 0.8637\\
        & 17 & 0.5900 & 0.3375\\
        & 18 & 0.4000 & 0.2288\\
        & 19 & 1.3950 & 0.7979\\
        \bottomrule
    \end{tabularx}
\end{table}

\clearpage
\section{Detailed Algorithm Settings For SC Experiments} \label{sec:detailed_algo_setting}
This section provides the detailed hyperparameters used in our experiment.

\subsection{Meta-training settings}

Table~\ref{tab:hyperparameters_ml} and~\ref{tab:hyperparameters_ppo} show the algorithm settings we used for Meta-MORL and MERLION during our experiments.
\begin{table}[H]
    \centering
    \caption{Hyperparameters for MAML-based meta-training.}
    \label{tab:hyperparameters_ml}
    \begin{tabularx}{5cm}{lr}
        \toprule
        Hyperparameter & Value \\
        \midrule
        \textbf{Shared parameters:}\\
        Inner-loop learning rate $\alpha$ & 0.03 \\
        Outer-loop learning rate $\beta$ & 0.001 \\
        PPO steps per meta-update & 10 \\
        Adaptation steps & 4--8 \\
        Network architecture & [64, 64] \\
        Discount factor $\gamma$ & 0.99 \\
        GAE weight & 1.0 \\
        VF loss coefficient & 0.5 \\
        Clipping parameter & 0.3 \\
        KL target & 0.01 \\
        KL coefficient & 0.001 \\
        Batch size & 32 \\
        Fragment length & 100 \\
        Rollout workers & 4 \\
        Evaluation interval & 100 \\
        \midrule
        \textbf{MERLION-specific:}\\
        Population size & 10\\
        Fitting score coefficient $\kappa$ & 0.5\\
        Cross-over rate $\eta_{cross}$ & 0.6\\
        Mutation strength $\delta$ & 0.1\\
        \bottomrule
    \end{tabularx}
\end{table}

\subsection{Fine-tuning settings}

\begin{table}[H]
    \centering
    \caption{PPO hyperparameters used in fine-tuning.}
    \label{tab:hyperparameters_ppo}
    \begin{tabularx}{5cm}{lr}
    \toprule
    Hyperparameter & Value \\
    \midrule
    Learning rate & $3 \times 10^{-4}$ \\
    Minibatch size & 64 \\
    Epochs per update & 10 \\
    Discount factor $\gamma$ & 0.99 \\
    GAE weight $\lambda$ & 0.95 \\
    Clipping range & 0.2 \\
    Entropy coefficient & 0 \\
    VF loss coefficient & 0.5 \\
    Max gradient norm & 0.5 \\
    Episode length & 100 \\
    Total timesteps & $5 \times 10^5$ \\
    Activation & Tanh \\
    \bottomrule
    \end{tabularx}
\end{table}

\subsection{MORL with decomposition (MORL/D) setup}

Table~\ref{tab:hyperparameter_MORLD} shows the setup used for the MORL/D algorithm as one of the baselines.
\begin{table}[H]
    \centering
    \caption{Hyperparameters of MORL/D Algorithm with MOSAC Policy}
    \label{tab:hyperparameter_MORLD}
    \begin{tabularx}{6cm}{lr}
        \toprule
        Hyperparameters & Values \\
        \midrule
        \textbf{MORL/D} & \\
        Discount factor & 0.995 \\
        Population size & 6 \\
        Exchange every & $5 \times 10^4$ \\
        Total time steps & $5 \times 10^5$ \\
        Neighbourhood size & 1 \\
        Update passes & 10 \\
        Initial weight dist. & uniform \\
        \midrule
        \textbf{MOSAC} & \\
        Buffer size & $10^6$ \\
        Discount factor & 0.99 \\
        Target smoothing coef. & 0.005 \\
        Batch size & 128 \\
        Steps before learning & $10^3$ \\
        Hidden neurons & [256; 256] \\
        Actor LR & $3 \times 10^{-4}$ \\
        Critic LR & $10^{-3}$ \\
        Actor training freq. & 2 \\
        Target training freq. & 1 \\
        Activation function & ReLU \\
        \bottomrule
    \end{tabularx}
\end{table}

\subsection{Non-dominated sorting genetic algorithm II (NSGA-II) setup}
Table~\ref{tab:nsga_parameters} shows the setup used in NSGA-II as one of the baselines.
\begin{table}[H]
    \centering
    \caption{Hyperparameters of NSGA-II to solve the SC problem across three network complexities: simple, moderate, and complex.}
    \label{tab:nsga_parameters}
    \begin{tabularx}{6cm}{lr}
        \toprule
        Hyperparameters & Values\\
        \midrule
        Population size & 300\\
        Offspring number & 30\\
        Cross-over operator & binary (SBX)\\
        Cross-over probability & 90\%\\
        Cross-over $\eta_{cross}$ & 15\\
        Mutation method & polynomial\\
        Mutation $\eta_{mut}$ & 20\\
        \bottomrule
    \end{tabularx}
\end{table}

\clearpage
\section{Detailed Experimental Setting for Conventional RL in standard Problems}
\subsection{Algorithm settings}

\label{app:benchmark-settings}

This appendix reports the hyperparameter settings of the benchmark algorithms used in the standard RL problem experiments. Unless otherwise stated, all methods were implemented using the \texttt{morl-baselines} framework and evaluated on the same \texttt{mo-gymnasium} environments as MetaMORL and MERLION. Table~\ref{tab:hyperparameter_morld} to~\ref{tab:hyperparameter_pcn} displayed the used hyperparameter for standard problem benchmarking.

\begin{table}[H]
    \centering
    \caption{Hyperparameters of MORL/D with MOSAC or MOSAC-Discrete policy}
    \label{tab:hyperparameter_morld}
    \begin{tabularx}{7cm}{l r}
        \toprule
        Hyperparameter & Value \\
        \midrule
        \textbf{MORL/D} & \\
        Scalarisation method & Weighted sum \\
        Utility type & SER \\
        Population size & 10 \\
        Exchange every & $4 \times 10^4$ \\
        Neighbourhood size & 1 \\
        Update passes & 10 \\
        Initial weight distribution & Uniform \\
        Shared buffer & False \\
        \midrule
        \textbf{Policy backbone} & \\
        Policy type & MOSAC / MOSAC-Discrete \\
        Network architecture & [64, 64] \\
        Policy learning rate & $10^{-4}$ \\
        Critic learning rate & $10^{-2}$ \\
        Discount factor & 0.99 \\
        \bottomrule
    \end{tabularx}
\end{table}

\begin{table}[H]
    \centering
    \caption{Hyperparameters of the PGMORL algorithm}
    \label{tab:hyperparameter_pgmorl}
    \begin{tabularx}{8.1cm}{l r}
        \toprule
        Hyperparameter & Value \\
        \midrule
        Number of environments & 4 \\
        Population size & 10 \\
        Warm-up iterations & $0.05 \times$ episode horizon \\
        Steps per iteration & Episode horizon \\
        Evolutionary iterations & 20 \\
        Number of weights candidates & 21 (2 objectives), 32 (otherwise) \\
        Performance buffer size & 100 \\
        Number of performance buffers & 2 \\
        Weight range & [0, 1] \\
        $\Delta w$ & $1/9$ \\
        Sparsity coefficient & $-1.0$ \\
        Network architecture & [64, 64] \\
        Number of minibatches & 32 \\
        Update epochs & 10 \\
        Learning rate & $3 \times 10^{-4}$ \\
        \bottomrule
    \end{tabularx}
\end{table}

\begin{table}[H]
    \centering
    \caption{Hyperparameters of the CAPQL algorithm}
    \label{tab:hyperparameter_capql}
    \begin{tabularx}{6cm}{lr}
        \toprule
        Hyperparameter & Value \\
        \midrule
        Learning rate & $3 \times 10^{-4}$ \\
        Target smoothing coefficient ($\tau$) & 0.005 \\
        Replay buffer size & $10^6$ \\
        Network architecture & [64, 64] \\
        Batch size & Episode horizon \\
        Number of Q-networks & 2 \\
        Entropy coefficient ($\alpha$) & 0.2 \\
        Learning starts after & $10^3$ steps \\
        Gradient updates per step & 1 \\
        Discount factor & 0.99 \\
        \bottomrule
    \end{tabularx}
\end{table}

\begin{table}[H]
    \centering
    \caption{Hyperparameters of the PQL algorithm}
    \label{tab:hyperparameter_pql}
    \begin{tabularx}{6cm}{l r}
        \toprule
        Hyperparameter & Value \\
        \midrule
        Initial $\epsilon$ & 1.0 \\
        Final $\epsilon$ & 0.1 \\
        Epsilon decay steps & $10^5$ \\
        Discount factor & 0.99 \\
        \bottomrule
    \end{tabularx}
\end{table}

\begin{table}[H]
    \centering
    \caption{Hyperparameters of the MPMOQL algorithm}
    \label{tab:hyperparameter_mpmoql}
    \begin{tabularx}{4cm}{lr}
        \toprule
        Hyperparameter & Value \\
        \midrule
        Learning rate & 0.1 \\
        Initial $\epsilon$ & 0.1 \\
        Final $\epsilon$ & 0.1 \\
        Epsilon decay & None \\
        Weight selection & Random \\
        GPI policy & False \\
        Transfer Q-table & True \\
        Dyna updates & False \\
        Discount factor & 0.99 \\
        \bottomrule
    \end{tabularx}
\end{table}

\begin{table}[H]
    \centering
    \caption{Hyperparameters of the PCN algorithm}
    \label{tab:hyperparameter_pcn}
    \begin{tabularx}{6.5cm}{lr}
        \toprule
        Hyperparameter & Value \\
        \midrule
        Scaling factor & [1.0] \\
        Learning rate & $10^{-3}$ \\
        Batch size & Episode horizon \\
        Hidden dimension & 64 \\
        Exploration noise & 0.1 \\
        Continuous action setting & Environment-dependent \\
        \bottomrule
    \end{tabularx}
\end{table}

\subsection{Benchmark problem descriptions}
\label{subsec:toy-problem-descriptions}

We evaluated the benchmark Merlion and other algorithms on five standard problems from \texttt{mo-gymnasium}: \texttt{mo-halfcheetah-v4}, \texttt{deep-sea-treasure-v0}, \texttt{resource-gathering-v0}, \texttt{mo-hopper-v4}, and \texttt{mo-reacher-v4}. These tasks were selected to cover both discrete and continuous domains, as well as different numbers of objectives, ranging from two to four. This allows us to assess performance under diverse Pareto-front structures and control settings.

\subsubsection{MO-HalfCheetah.}
\texttt{mo-halfcheetah-v4} is a continuous-control MuJoCo locomotion task with a continuous observation space, a continuous action space, and two objectives. In our implementation, it is treated as a long-horizon control problem with a maximum episode length of 1000 steps. This environment is useful for testing whether a method can recover meaningful trade-offs between competing locomotion objectives in a standard continuous benchmark.

\subsubsection{Deep Sea Treasure.}
\texttt{deep-sea-treasure-v0} is a discrete grid-world problem with discrete observations, discrete actions, and two objectives. It is one of the canonical MORL benchmarks because it provides a structured Pareto front with a clear trade-off between treasure value and time penalty. We used a maximum episode length of 200 steps.

\subsubsection{Resource Gathering.}
\texttt{resource-gathering-v0} is a discrete grid-world environment with discrete observations, discrete actions, and three objectives. The task requires balancing multiple resource-related returns while avoiding unfavourable outcomes, making it a useful benchmark for testing whether a method can identify non-dominated policies beyond the two-objective setting. We used a maximum episode length of 200 steps.

\subsubsection{MO-Hopper.}
\texttt{mo-hopper-v4} is a continuous-control MuJoCo locomotion task with continuous observations, continuous actions, and three objectives. Compared with \texttt{mo-halfcheetah-v4}, this environment introduces a more challenging multi-objective control setting with an additional objective and longer-horizon dynamics. We used a maximum episode length of 1000 steps.

\subsubsection{MO-Reacher.}
\texttt{mo-reacher-v4} is a multi-objective reaching task with a continuous observation space, a discrete action space in our benchmark interface, and four objectives. This task provides the highest reward dimensionality among the standard problems considered here, making it suitable for evaluating how well each method scales to more structured and higher-dimensional Pareto fronts. We used a maximum episode length of 200 steps.

\subsubsection{Common environment configuration.}
All environments were instantiated through \texttt{mo\_gym.make(..., render\_mode="rgb\_array")} and wrapped with \texttt{TimeLimit} using an environment-specific maximum episode length. No reward normalisation or scalarisation was applied at the environment level for benchmark evaluation, and comparisons were based on the cumulative raw vector rewards returned by the environments. The observation space type, action space type, number of objectives, episode horizon, and reference point used for each environment are summarised in Table~\ref{tab:toy-problem-envs}.

\begin{table}[t]
    \centering
    \caption{RL standard problem environments details.}
    \label{tab:toy-problem-envs}
    \begin{tabular}{lcccc}
        \toprule
        Env. & Obs. & Act. & Obj. & Timesteps/eps. \\
        \midrule
        \texttt{mo-halfcheetah-v4}     & Cont. & Cont. & 2 & 1000 \\
        \texttt{deep-sea-treasure-v0}  & Disc. & Disc. & 2 & 200 \\
        \texttt{resource-gathering-v0} & Disc. & Disc. & 3 & 200 \\
        \texttt{mo-hopper-v4}          & Cont. & Cont. & 3 & 1000 \\
        \texttt{mo-reacher-v4}         & Cont. & Disc. & 4 & 200 \\
        \bottomrule
    \end{tabular}
\end{table}

\subsubsection{Common benchmark configuration.}
Across all benchmark algorithms, we used the same environment definitions, episode horizons, and random-seed interface. The discount factor was set to $\gamma = 0.99$ for all methods. Algorithm applicability depended on the observation and action spaces of each task. In particular, PQL and MPMOQL were only used on discrete-observation, discrete-action environments, while PGMORL and CAPQL were only used on continuous-observation, continuous-action environments. MORL/D and PCN were applied whenever their architectural assumptions matched the corresponding environment.

\clearpage
\section{Pareto Front Plots}
To compliment the PF plots in the main text, we provide additional views of PF plots for better representation as shown in Figure~\ref{fig:pf_additional}.

\begin{figure*}[!ht]
\centering

\subfloat[\footnotesize Simple SC - view 2]{\includegraphics[width=0.33\linewidth]{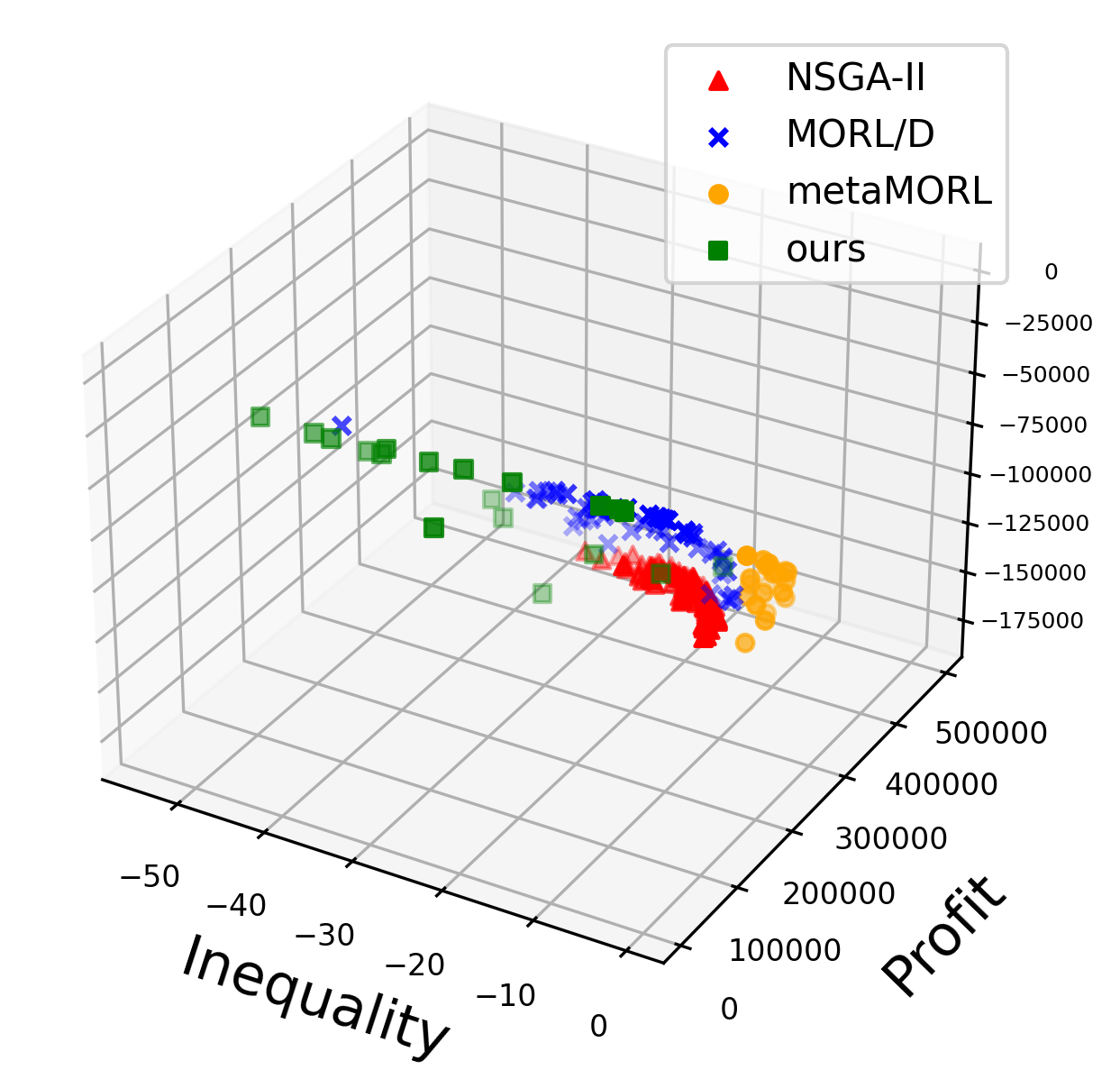}}\hfil
\subfloat[\footnotesize Moderate SC - view 2]{\includegraphics[width=0.33\linewidth]{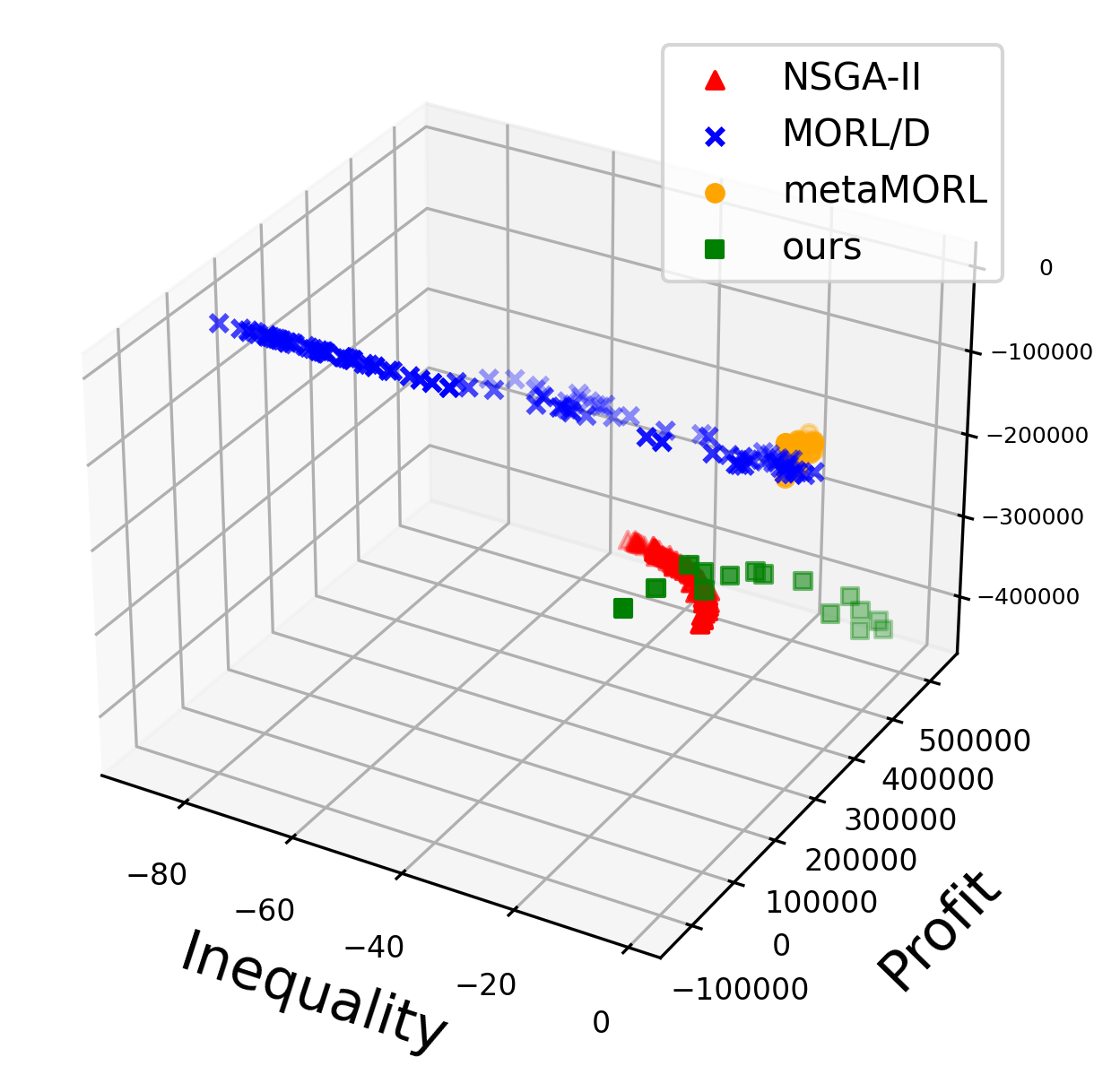}}\hfil
\subfloat[\footnotesize Complex SC - view 2]{\includegraphics[width=0.33\linewidth]{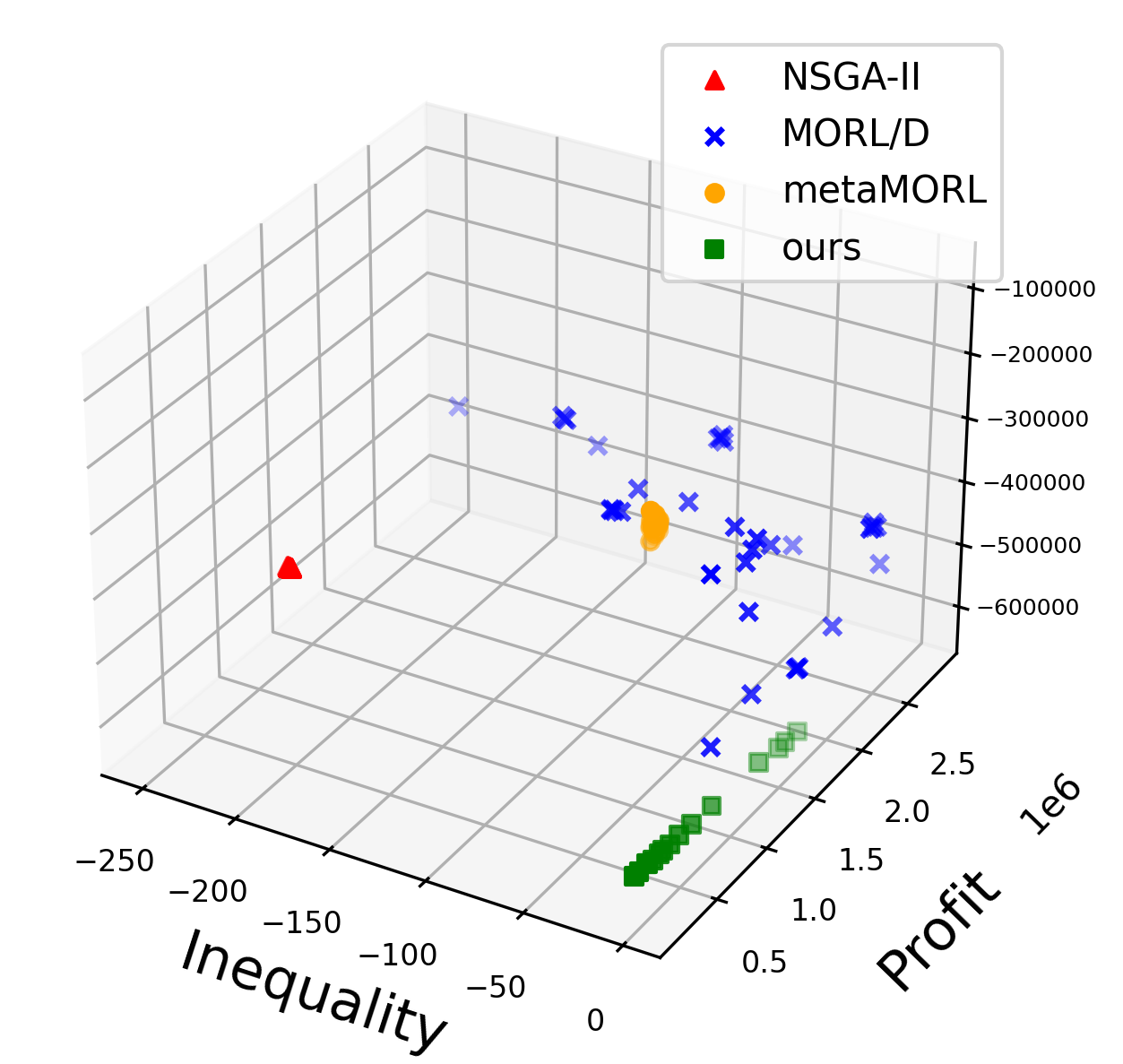}}\\

\subfloat[\footnotesize Simple SC - view 3]{\includegraphics[width=0.33\linewidth]{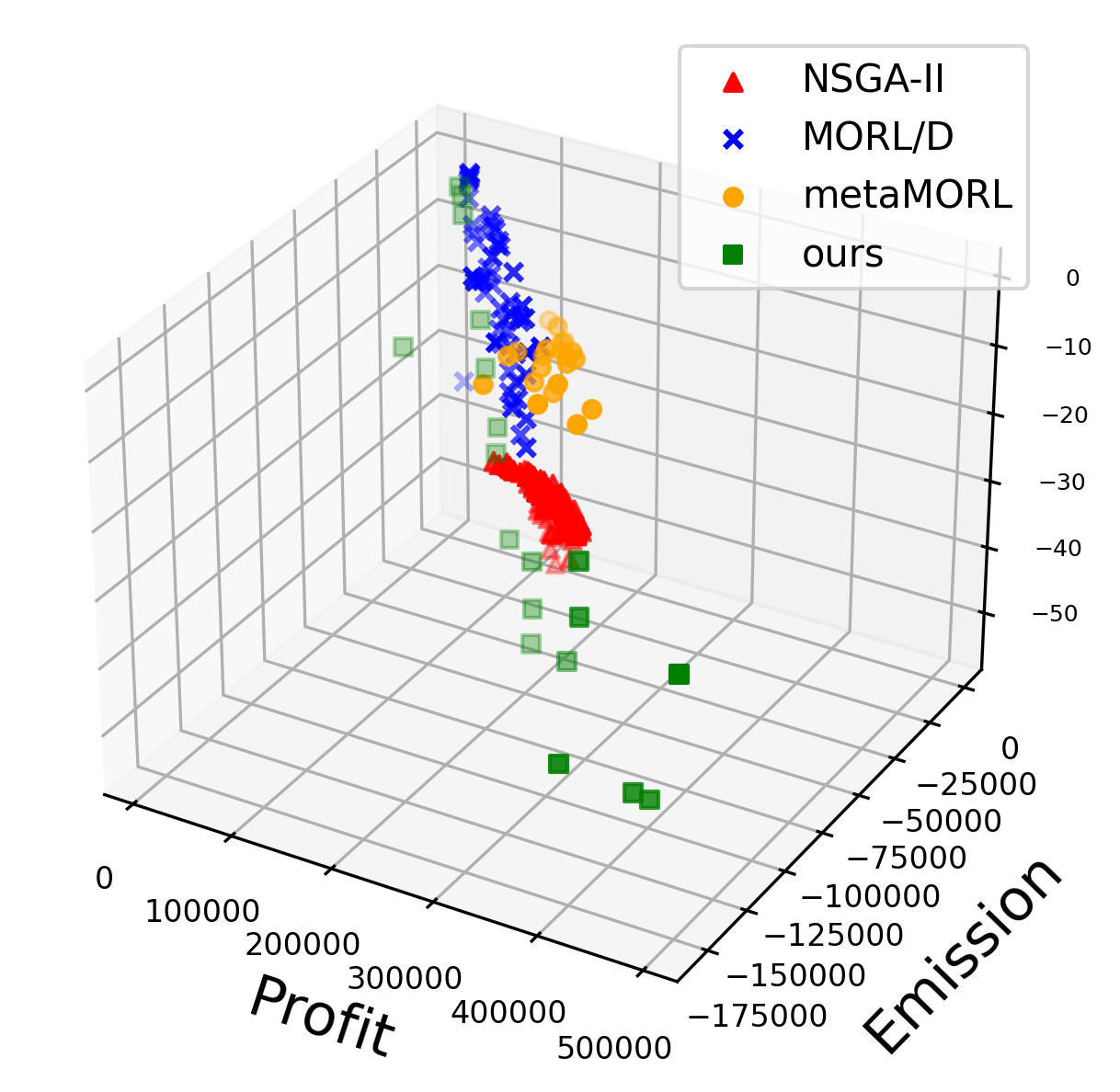}}\hfil
\subfloat[\footnotesize Moderate SC - view 3]{\includegraphics[width=0.33\linewidth]{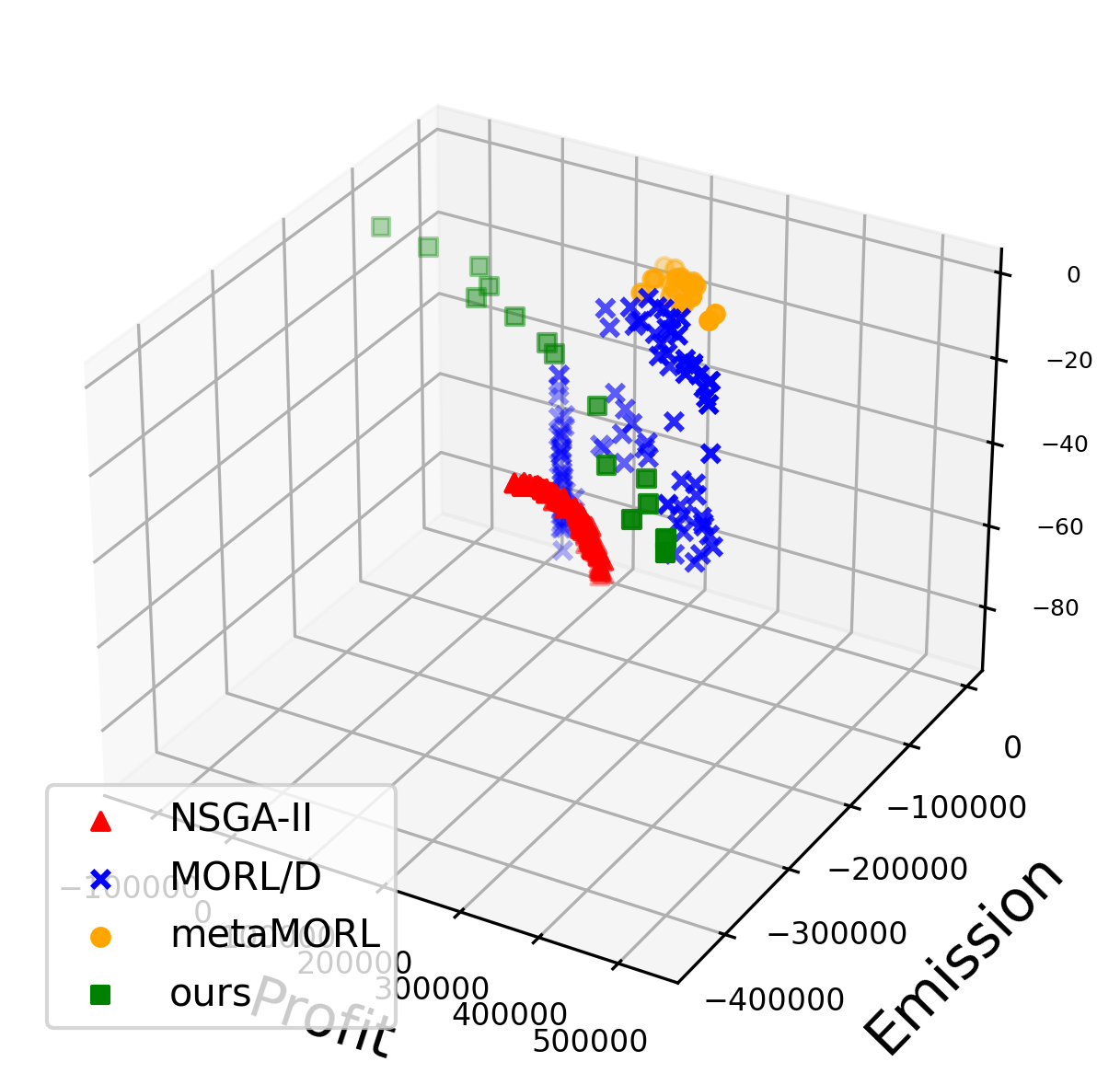}}\hfil
\subfloat[\footnotesize Complex SC - view 3)]{\includegraphics[width=0.33\linewidth]{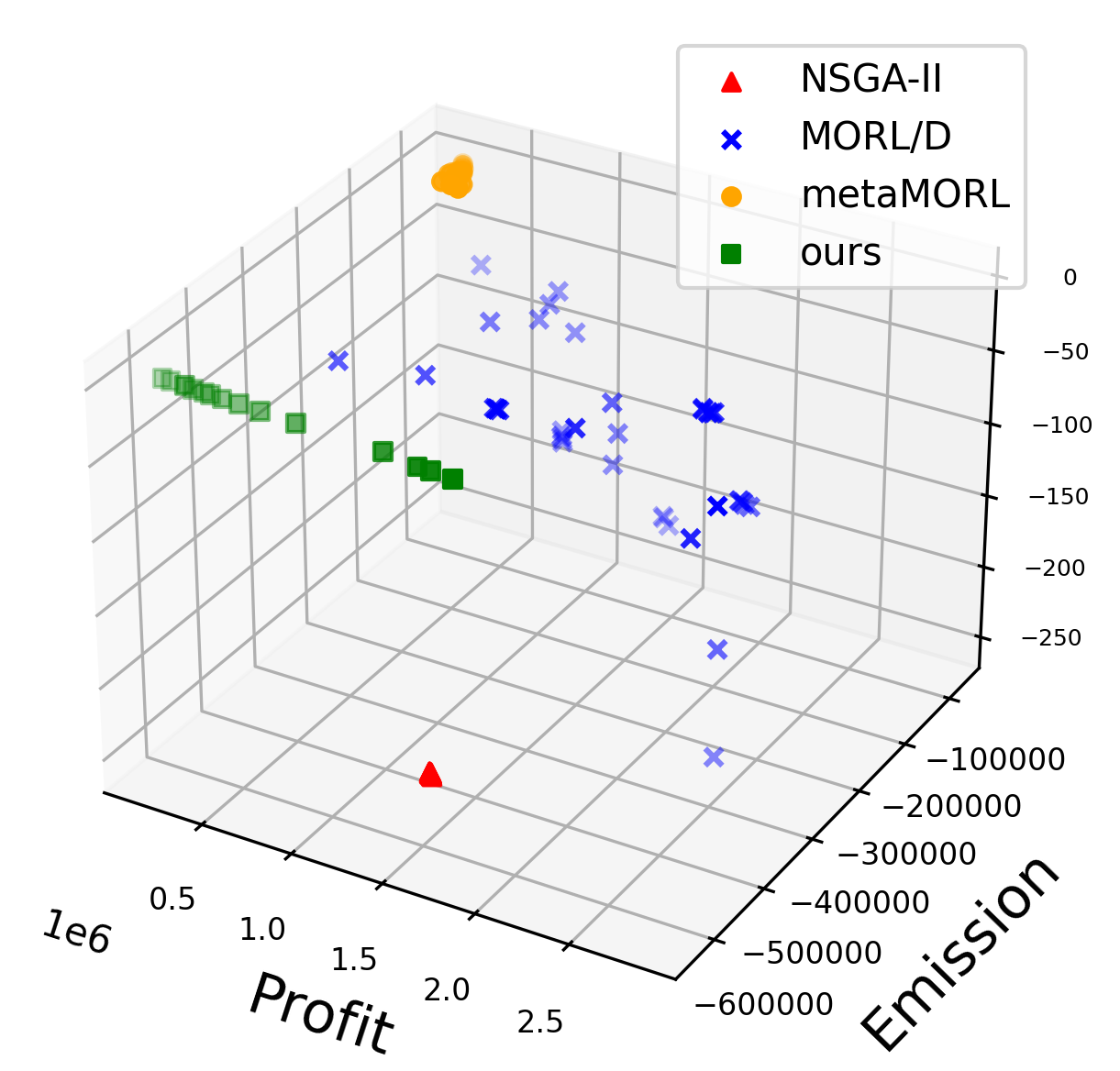}}\\
\caption{Additional views from the other perspectives to complement PF Plots in the main text.}
\label{fig:pf_additional}
\end{figure*}

\clearpage
\section{Operational Performance} \label{app:operational_performance}
We further examine the operational trajectories underlying the representative solutions selected from the PF approximation sets. Figure~\ref{fig:operational_performance} presents the corresponding manufacturing quantities, inventory levels, and demand loss over time.

Manufacturing fluctuations are more pronounced in the simple SC setting, while the moderate and complex settings display steadier production schedules. One reason is that, in deeper and more interconnected networks, decisions propagate through additional stages, transit delays, and downstream buffers, so abrupt short-term changes may create delayed and amplified side effects. As a result, effective policies tend to resemble more stable base-stock or level-production strategies, using inventories and shipments to absorb variability rather than relying on frequent production corrections. In addition, the feasible action space becomes tighter as complexity increases because more facilities, routes, and interdependencies impose stronger operational constraints. These factors also make policy learning more difficult, and PPO-style updates may become effectively more conservative under noisier gradient signals, encouraging reliance on steadier control heuristics.

The inventory plots further indicate that buffering behaviour changes with complexity. In the simple SC setting, inventories stay close to zero for much of the horizon, interrupted by occasional spikes, which is consistent with reactive replenishment and rapid stock clearing. In the moderate and complex settings, inventories become more persistent and are increasingly concentrated at a smaller subset of facilities. This suggests that the learned policies exploit positional buffering, where selected nodes absorb variability locally rather than maintaining evenly distributed stock across the network.

Demand loss remains bounded overall, but its spatial distribution becomes more heterogeneous in the more complex settings. Some markets exhibit consistently higher unmet demand, whereas others maintain relatively low and stable losses. This indicates that, under stronger coupling and delayed flow adjustments, the learned policy does not equalise service levels uniformly across all markets, but instead allocates service unevenly across locations as part of its broader multi-objective trade-off.

\begin{figure*}[!ht]
\centering
\subfloat[\footnotesize Mfg - simple SC]{\includegraphics[width=0.33\linewidth]{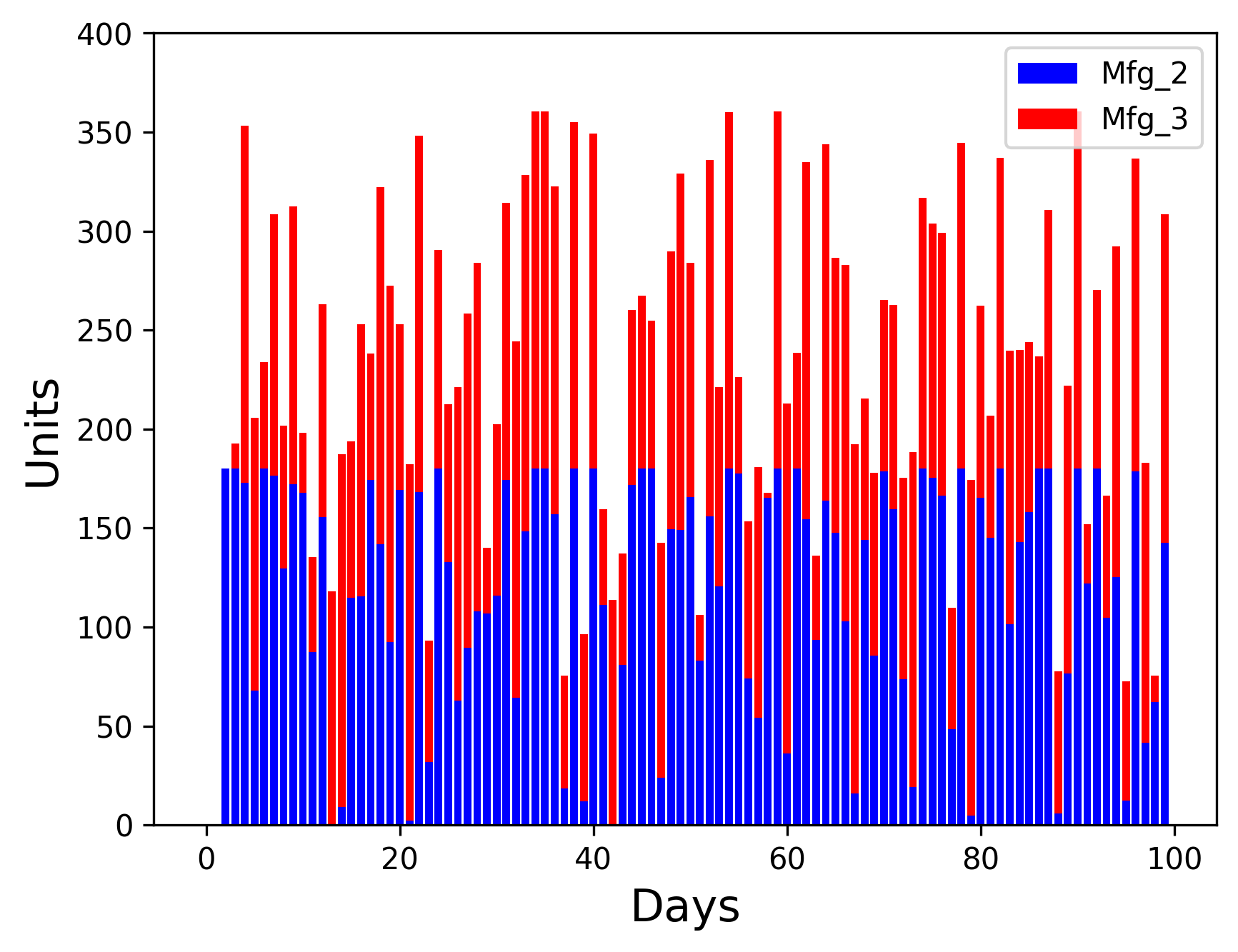}}\hfil
\subfloat[\footnotesize Mfg - moderate SC]{\includegraphics[width=0.33\linewidth]{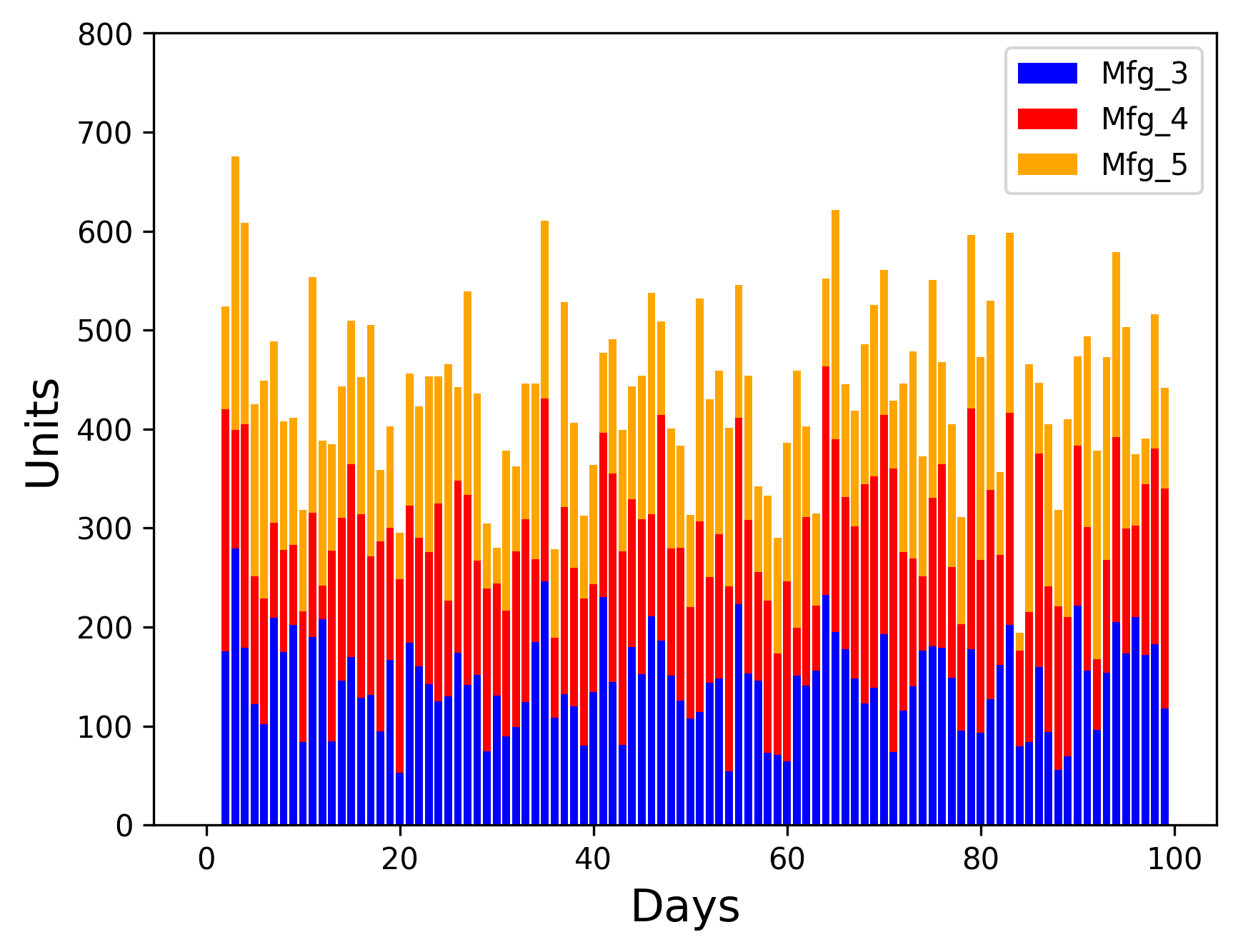}}\hfil
\subfloat[\footnotesize Mfg - complex SC]{\includegraphics[width=0.33\linewidth]{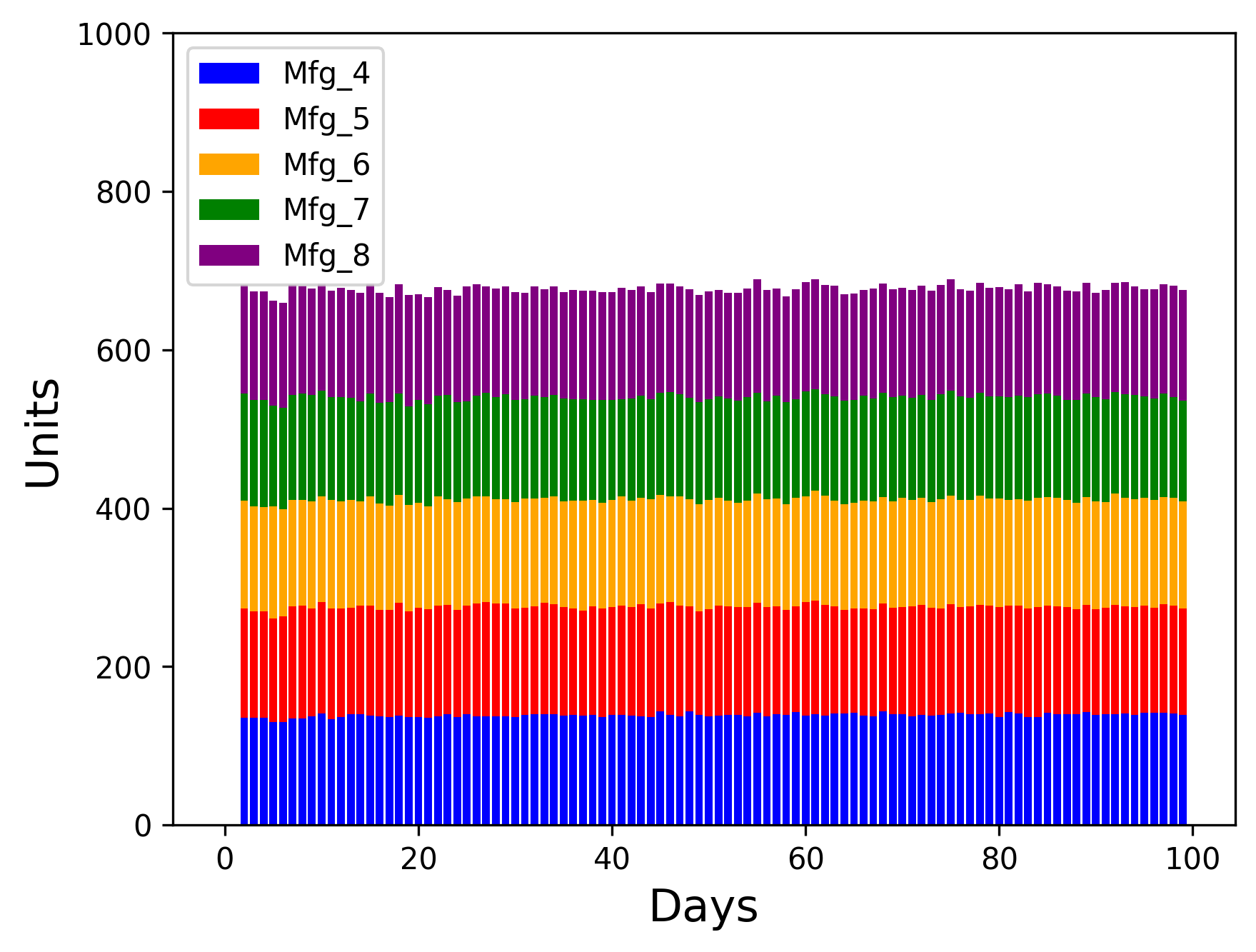}}\\

\subfloat[\footnotesize Inv - simple SC]{\includegraphics[width=0.33\linewidth]{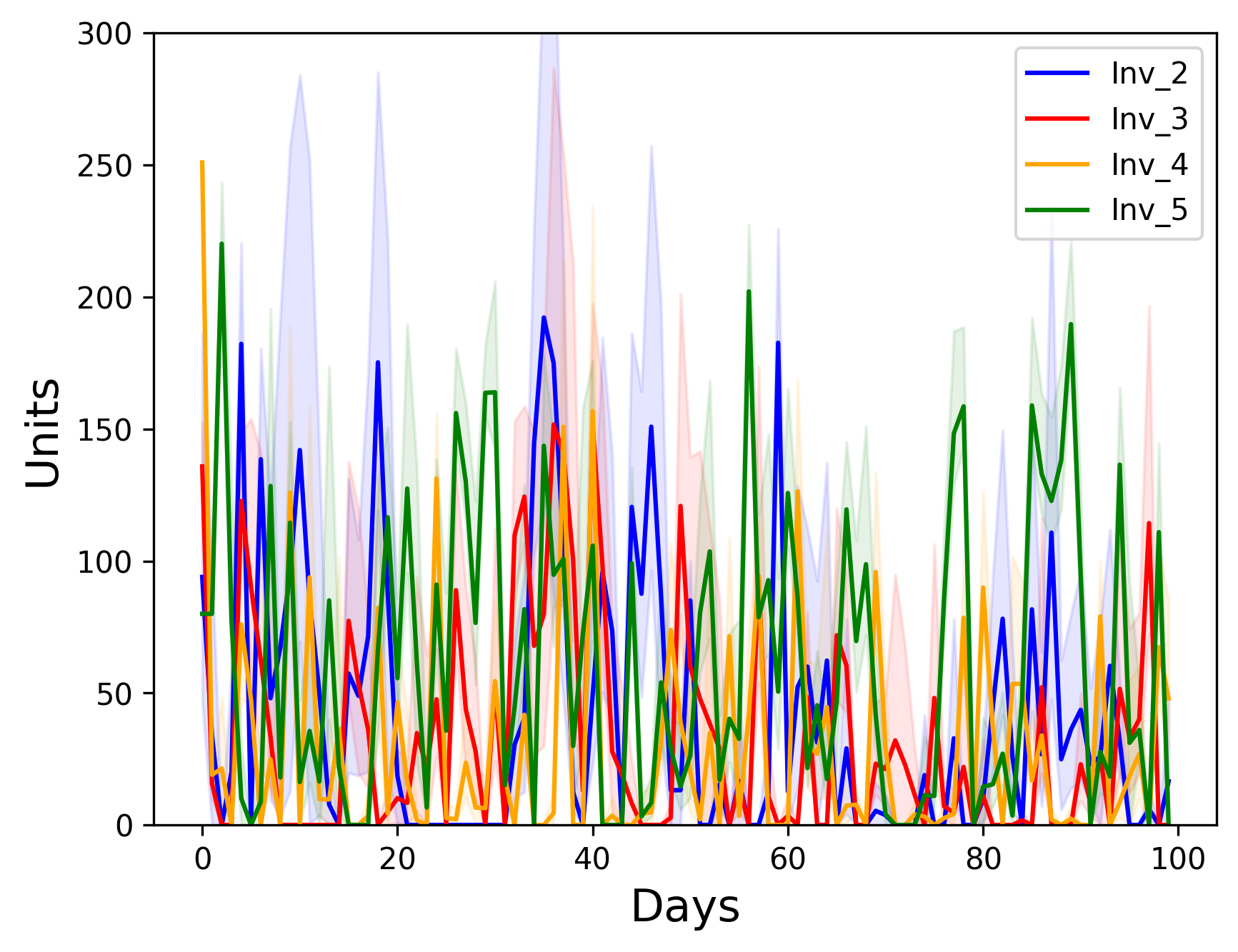}}\hfil
\subfloat[\footnotesize Inv - moderate SC]{\includegraphics[width=0.33\linewidth]{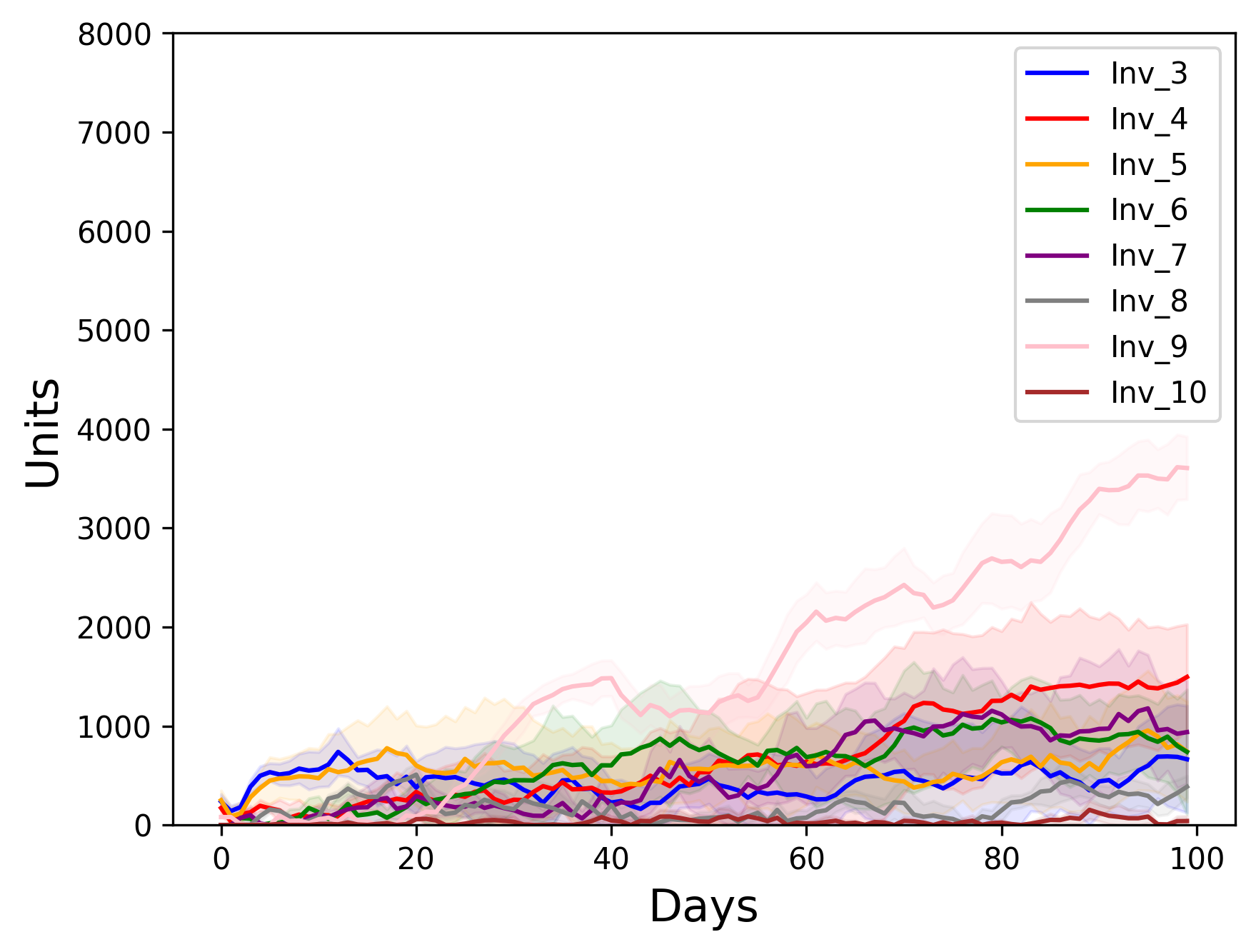}}\hfil
\subfloat[\footnotesize Inv - complex SC]{\includegraphics[width=0.33\linewidth]{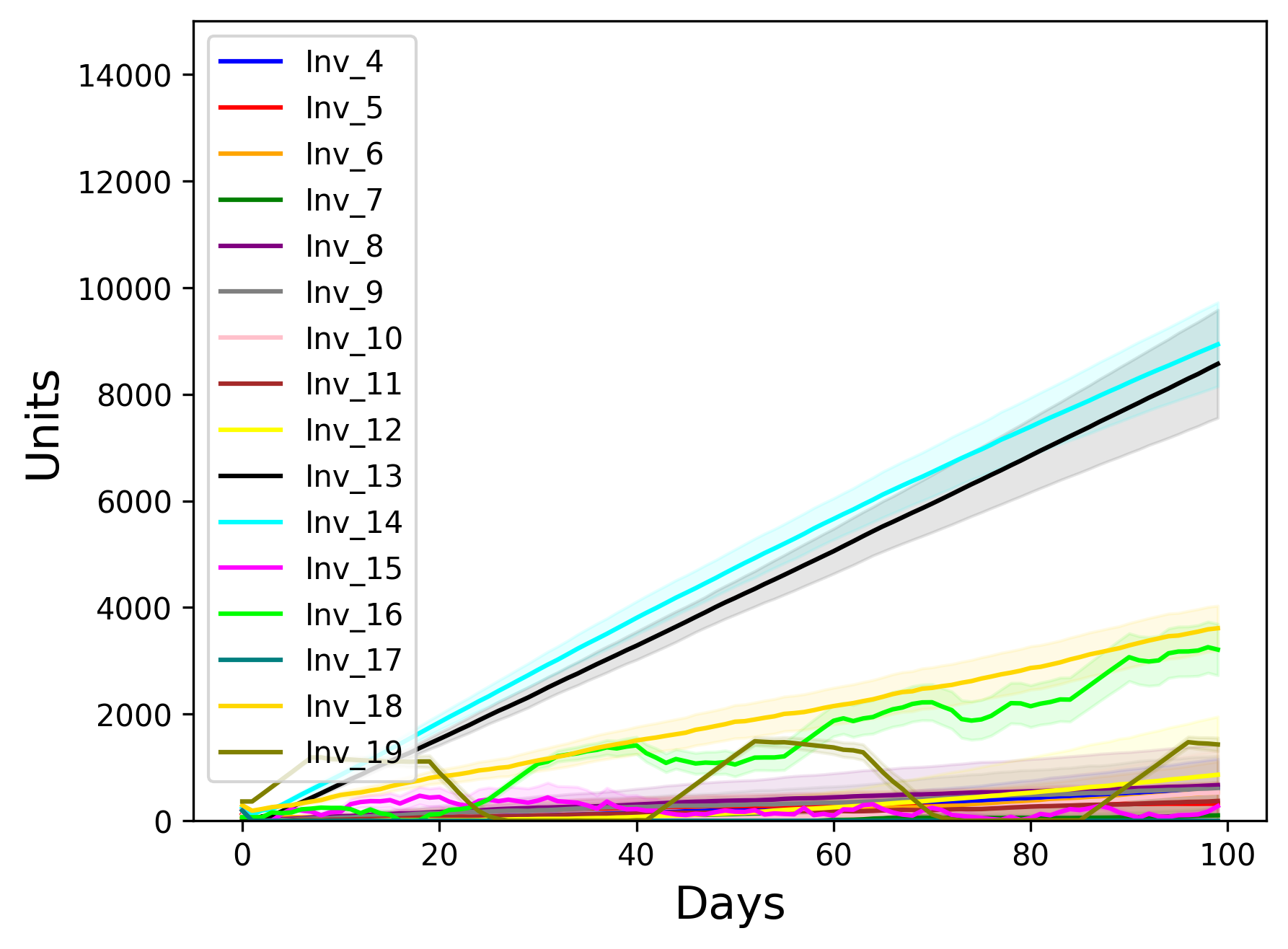}}\\

\subfloat[\footnotesize Loss - simple SC]{\includegraphics[height=0.25\linewidth]{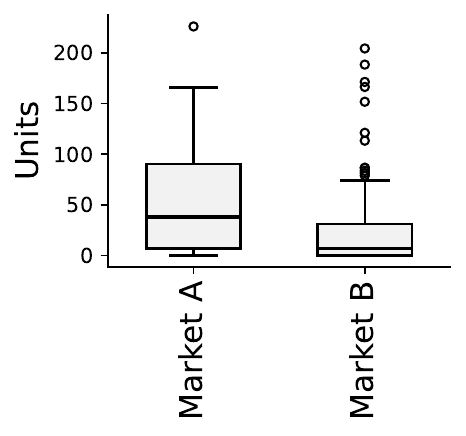}}\hfil
\subfloat[\footnotesize Loss - moderate SC]{\includegraphics[height=0.25\linewidth]{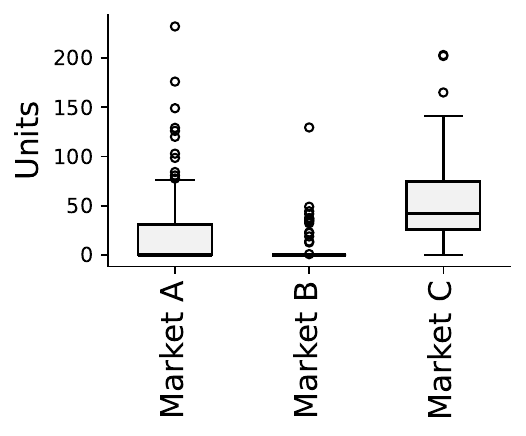}}\hfil
\subfloat[\footnotesize Loss - complex SC]{\includegraphics[height=0.25\linewidth]{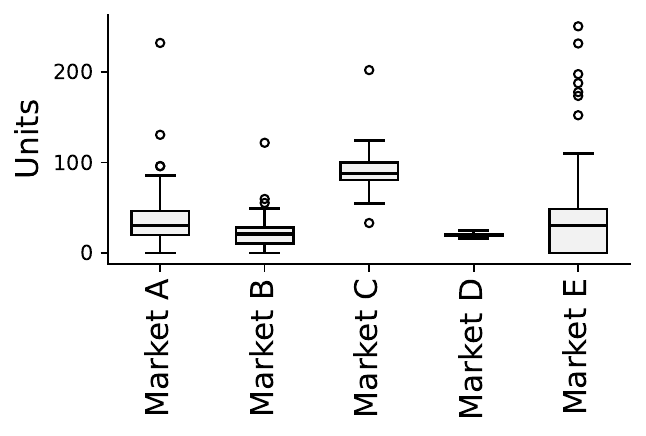}}\\

\caption{MERLION’s operational performance across the SC problems shows that manufacturing and inventory levels fluctuate less in more complex problems, although some stockpiling becomes more noticeable. Demand losses in the various markets are relatively low.}
\label{fig:operational_performance}
\end{figure*}